%% file: NeurIPS_2026/main.tex
\documentclass{article}

\usepackage[final,main]{neurips_2026}

\usepackage{amsmath}  
\usepackage{amsthm}
\usepackage{amssymb}  
\usepackage[utf8]{inputenc} 
\usepackage[T1]{fontenc}    
\usepackage{hyperref}       
\usepackage{url}            
\usepackage{booktabs}       
\usepackage{amsfonts}       
\usepackage{nicefrac}       
\usepackage{microtype}      
\usepackage{xcolor}         
\usepackage{graphicx}
\usepackage{multirow}
\usepackage{pifont}
\usepackage{wrapfig}
\usepackage[dvipsnames]{xcolor}
\usepackage{placeins}      
\usepackage{enumitem}

\newcommand{\cmark}{\textcolor{ForestGreen}{\ding{51}}}
\newcommand{\xmark}{\textcolor{red}{\ding{55}}}

\newtheorem{remark}{Remark}
\newtheorem{proposition}{Proposition}
\newtheorem{theorem}{Theorem}
\newtheorem{lemma}{Lemma}
\newtheorem{assumption}{Assumption}

\title{Towards Reliable LLM Evaluation: Correcting the Winner’s Curse in Adaptive Benchmarking}

\author{%
  Jiefu Zhang\thanks{Equal contribution.} \\
  Purdue University\\
  \texttt{zhan4018@purdue.edu} \\
  \And
  Yang Xu\footnotemark[1] \\
  Purdue University\\
  \texttt{xu1720@purdue.edu} \\
  \And
  Haixiang Sun\\
  Purdue University\\
  \texttt{sun1321@purdue.edu} \\
  \And
  Zihan Zhou\\
  Johns Hopkins University\\
  \texttt{zzhou150@jh.edu} \\
  \And
  Tianyu Cao\\
  Purdue University\\
  \texttt{cao357@purdue.edu} \\
  \And
  Vaneet Aggarwal \\
  Purdue University \\
  \texttt{vaneet@purdue.edu}
}

\begin{document}

\maketitle

\begin{abstract}
Adaptive prompt and program search makes LLM evaluation selection-sensitive. Once benchmark items are reused inside tuning, the observed winner's score need not estimate the fresh-data performance of the full tune-then-deploy procedure. We study conditional post-search inference under explicit tuning budgets. We propose SIREN, a selection-aware repeated-split reporting protocol that freezes the post-search shortlist, separates splitwise selection from held-out evaluation, and uses an item-level Gaussian multiplier bootstrap for uncertainty quantification. In a fixed-shortlist regime with smooth stabilized selection, the estimator admits a first-order item-level representation, and the bootstrap yields valid simultaneous inference on a finite budget grid. This supports confidence intervals for procedure-performance curves and pre-specified equal-budget and cross-budget comparisons. Controlled simulations and MMLU-Pro tuning experiments show that winner-based reporting can be optimistic and can change deployment conclusions, while SIREN remains close to the finite-sample reporting target. Codes are available at \url{https://github.com/jznmsl/siren}.
\end{abstract}

\section{Introduction}

Large language models (LLMs) have become a central tool for reasoning, coding, question answering, and related language-based tasks \cite{openai2023gpt4,anil2023palm2,grattafiori2024llama3}. As their capabilities have improved, the question of how to evaluate them has become correspondingly more important. More concretely, for a given task and a collection of models, we would like an evaluation scheme whose reported scores remain informative about how those systems actually perform on that task. Even the simplest setting, where each model is evaluated using one fixed configuration chosen in advance, can be fragile and prone to bias. Recent work shows that changing this template can change both absolute performance and relative rankings across models \cite{mizrahi2024state,polo2024efficient}. A growing line of LLM systems goes beyond this fixed system setting. Instead of fixing one prompt or one program configuration in advance, these methods use benchmark examples to search for a stronger way of using the same base model. Concretely, they may generate several candidate instructions or program variants, evaluate them on development examples, and then refine or select among them under an iterative optimization loop or an explicit budget constraint \cite{yang2024large,khattab2023dspy,shi2024triple,opsahlong2024optimizing,ashizawa2025opts}. Once evaluation is used in this way, the concrete benchmark example, which is a task instance together with the scoring criteria, no longer serves only as a measuring device. It also becomes part of the search process that determines which configuration is ultimately reported. Throughout the paper, a system $j$ is a fixed base model together with fixed pipeline components, and an artifact $a\in\mathcal A_j$ is a deployable configuration. We use tuning for the full procedure that searches over $\mathcal A_j$ and returns an artifact, and prompt/program search for the generation, evaluation, refinement, and selection of candidates within tuning. We call this setting adaptive benchmarking because benchmark examples influence which artifact is returned and reported. Thus, tuning searches over artifacts for each fixed system, while $j\in\{1,\dots,N\}$ indexes the systems being compared.

This reuse of benchmark examples creates an immediate ambiguity about what the final reported number is meant to represent. One possibility is to report the artifact that happened to score highest on the observed development examples, where an artifact refers to a concrete configuration of the system, such as an instruction, a set of demonstrations, or a larger language model program configuration. The alternative is to study the full tune-then-deploy procedure: start from a base system, run a tuner with budget $B$ to search over artifacts, and deploy the returned artifact on future task instances. After searching ends, the chosen artifact is the one actually used at test time. Evaluation must therefore distinguish two questions. A retrospective question asks which candidate happened to look best on the finite sample used during searching. An operational question asks something different: if we rerun the same tuning rule with budget $B$ and then use its returned artifact on fresh data, what performance should we expect on average. These questions coincide when no search is performed, because then the reported system is fixed before any benchmark examples are observed. They diverge once benchmark examples are reused inside the tuning loop, because the final reported artifact is itself data dependent. We refer to this gap between the ex-post winning artifact and the fresh-data performance of the budget-$B$ tune-then-deploy procedure as the winner's curse in adaptive benchmarking. This distinction leads us to treat the full tune-then-deploy procedure as the object of evaluation. The corresponding target is a budget-indexed procedure-performance curve: for each budget $B$, the expected fresh-data performance of the artifact returned by the budget-$B$ tuner. The target is therefore not the artifact that happened to win on one observed sample, but the fresh-data performance of the rule that produced it.

Recent work makes this problem concrete, but does not yet fully solve it. Adaptive prompt and program optimization methods study how to search for stronger prompts or program configurations under iterative optimization loops or limited evaluation budgets \cite{yang2024large,khattab2023dspy,shi2024triple,opsahlong2024optimizing,ashizawa2025opts}. A separate line studies prompt variation and statistically principled evaluation \cite{mizrahi2024state,polo2024efficient,lior2025reliableeval,zollo2024prompt}. These papers are closely related, but they target different statistical objects: robustness to prompt choice, efficient evaluation over a pre-specified prompt pool, stochastic evaluation under meaning-preserving perturbations, or prompt choice under controlled risk guarantees. We ask a different question. Specifically, after a budgeted tuner has completed its search and frozen a finite shortlist, what can be reported about fresh-data performance under the specified post-search reporting rule?

Motivated by the above questions, we propose SIREN: \textbf{S}election-aware \textbf{I}nference for \textbf{R}epeated-split \textbf{E}valuatio\textbf{N}. SIREN is a post-search reporting protocol for a tune-then-deploy procedure. It begins after the upstream tuning stage has ended, treats the retained candidate artifacts as a frozen shortlist, and uses a separate evaluation pool for reporting. The formal guarantees are therefore conditional on the completed search and the fixed shortlist, rather than on resampling an unrestricted search process. SIREN repeatedly partitions the evaluation pool into a scoring part and a held-out part. In each split, the scoring part determines stabilized weights over retained artifacts, while the held-out part evaluates the selected output on fresh items. Aggregating across splits preserves the selection step while separating selection evidence from final performance evaluation. This is the reporting layer analyzed in Section~\ref{sec:siren}.

We validate SIREN both theoretically and empirically. In the fixed-shortlist, smooth stabilized-selection regime, we prove that the selected repeated-split estimator has a first-order item-level expansion and that a Gaussian multiplier bootstrap yields confidence intervals, simultaneous bands over tuning budgets, and pre-specified pairwise comparisons, without rerunning the tuner inside every resample. Controlled studies verify calibration and expose same-data winner optimism, while MMLU-Pro experiments across random search, DSPy, and eleven open-weight models show that selection-aware reporting can change tuned-versus-default conclusions and the interpretation of budget-aware prompt evaluation.

\textbf{Key contributions.}
\begin{itemize}
    \item We formulate adaptive LLM benchmarking as inference for budget-indexed tune-then-deploy procedures, and distinguish this broader objective from the conditional post-search target studied by SIREN.
    \item We propose SIREN, a selection-aware repeated-split evaluation protocol based on a frozen shortlist, splitwise scoring and held-out evaluation, and stabilized selection.
    \item We prove that the selected repeated-split estimator admits a per-item asymptotic representation and that a Gaussian multiplier bootstrap yields valid simultaneous inference on a finite grid of tuning budgets, including equal-budget and cross-budget comparisons fixed in advance.
    \item We validate SIREN in controlled studies and real LLM tuning experiments, showing that selection-aware procedure-level evaluation can change conclusions about system choice and the reliability of budget-aware reporting after adaptive tuning.
\end{itemize}

\section{Problem formulation}
\label{sec:problem}

We first define the target of interest. The setup covers prompt search, demonstration selection, and language-model program optimization. Let $P$ denote the distribution of a task instance together with the reference information for scoring. A system $j\in\{1,\dots,N\}$ is a base model plus fixed pipeline components, with $\mathcal A_j$ denoting its artifact space; an artifact $a\in\mathcal A_j$ is a deployable configuration such as instructions, demonstration sets, decoding choices, or larger language-model programs.

For a generic item $W\sim P$, artifact $a\in\mathcal A_j$, and execution randomness $\xi$, let $Z_j(W,a,\xi)\in[0,1]$ denote the per-item score, and define
\small
\begin{equation}
\mu_j(w,a)=\mathbb E_\xi\big[Z_j(W,a,\xi)\mid W=w\big]  
\label{eq:pf_mean}
\end{equation}
\normalsize
as the item-conditional mean score. Any deterministic evaluation is the special case with no $\xi$. For system $j$ and tuning budget $B$, let 
$D=(\widetilde W_1,\dots,\widetilde W_n)$ be the development sample of task instances used by the tuner, and let $U_j$ collect the tuner's internal randomness. We model the budget-$B$ tuner as a map from the development sample to a deployable artifact:
\small
\begin{equation}
\widehat a_j(B;D)=T_{j,B}(D,U_j)\in\mathcal A_j,
\label{eq:pf_output}
\end{equation}
\normalsize
where $B$ records the cost of the tuner. A fixed baseline is the special case in which the returned artifact is constant in $D$.
We further define the procedure-level target as
\small
\begin{equation}
\theta_j(B)=\mathbb E\big[\mu_j(W^{\mathrm{new}},\widehat a_j(B;D))\big],
\label{eq:pf_target}
\end{equation}
\normalsize
where $W^{\mathrm{new}}\sim P$ is independent of the development sample $D$; the expectation is over $D$, $U_j$, and $W^{\mathrm{new}}$. Thus $\theta_j(B)$ is the fresh-data performance of the rule ``use $D$ to run the budget-$B$ tuner, then deploy its returned artifact,'' not the score of one ex post winner on the sample used to choose it. Our goal is to estimate these procedure-level performances on a finite, pre-specified budget grid 
$\mathcal B=\{b_1,\dots,b_L\}$, where $B\in\mathcal B$ denotes a generic tuning budget. 
Section~\ref{sec:siren} analyzes a conditional post-search target with the completed tuning stage and shortlist fixed. Its soft deployment rule need not have the same target as the tuner-returned artifact in~\eqref{eq:pf_target}; Appendix~\ref{app:tuner_reruns} examines their relationship empirically.

\section{SIREN}
\label{sec:siren}

SIREN, short for Selection-aware Inference for Repeated-split EvaluatioN, is our approach for reliably evaluating a system together with its tuner. The task instances used by SIREN are separate from those used by the tuner. During tuning, a series of searches are performed to produce a finite set of candidate artifacts. SIREN then takes this frozen candidate set and an independent finite collection of benchmark items. For example, the benchmark items pool may consist of question-answer pairs with reference answers, summarization inputs with reference summaries or judging criteria, or coding problems with unit tests. This pool is the finite benchmark used to approximate the fresh deployment tasks represented by $W^{\mathrm{new}}$ in \eqref{eq:pf_target}. SIREN repeatedly splits this benchmark pool into a score subset and a held-out subset. In each split, the score subset is used only to set selection weights over the frozen candidates, while the held-out subset is used to evaluate the selected or weighted output after that choice has been made. SIREN then aggregates the resulting scores across splits and uses a Gaussian multiplier bootstrap to form confidence intervals.

\subsection{Protocol}

We first describe one fixed system $j$ and budget $B$, temporarily suppressing these indices. The frozen shortlist is $a_1,\dots,a_K$, with $K=K_{j,B}$, and the evaluation pool is $W_1,\dots,W_M$. In split $r$, $D_r^{\mathrm{score}}$ scores artifacts for selection, while $E_r$ evaluates them after the artifact weights have been formed; the score vectors are $\widehat S_r$ and $\widehat T_r$. Restoring the indices, $\mathcal A_j(B)=\{a_{j,B,1},\dots,a_{j,B,K_{j,B}}\}$. SIREN only scores and reweights these retained artifacts; it does not generate new ones.

For any task instance $w$ in the support of $P$ and any weight vector $q\in\Delta^{K_{j,B}}$, define the weighted mean score of the shortlist as $\mu_j(w,q;B)=\sum_{k=1}^{K_{j,B}} q_k\,\mu_j(w,a_{j,B,k})$. SIREN evaluates the retained shortlist on the separate evaluation pool 
$W_1,\dots,W_M \stackrel{\mathrm{i.i.d.}}{\sim} P$ via a repeated-split procedure. Concretely,
denote $R$ as the fixed number of repeated splits, for each $r\le R$, let $D_r^{\mathrm{score}}$ and $E_r$ be disjoint subsets of $\{1,\dots,M\}$, where the scoring subset $D_r^{\mathrm{score}}$ is used only to score the retained artifacts for selection and the held-out subset $E_r$ is used only for held-out evaluation after the splitting has been made in each round. We denote the sigma-field generated by the repeated-split procedure as  $\mathcal G_M=\sigma\big(\{D_r^{\mathrm{score}},E_r\}_{r=1}^R\big)$. After the subsets are partitioned, for artifact $a_{j,B,k}$, we compute its selection scores $\widehat S_{r,j,k}$ and held-out scores $\widehat T_{r,j,k}$ on split $r$ as
\small
\begin{align}
\widehat S_{r,j,k}(B)
&=
\frac{1}{|D_r^{\mathrm{score}}|}
\sum_{i\in D_r^{\mathrm{score}}}
Z_j\big(W_i,a_{j,B,k},\xi^{S}_{irjk}\big),
\label{eq:siren_dev_score}
\\
\widehat T_{r,j,k}(B)
&=
\frac{1}{|E_r|}
\sum_{i\in E_r}
Z_j\big(W_i,a_{j,B,k},\xi^{T}_{irjk}\big),
\label{eq:siren_hold_score}
\end{align}
\normalsize
and write $\widehat S_{r,j}(B)$ and $\widehat T_{r,j}(B)$ as the corresponding stacked vectors in $\mathbb R^{K_{j,B}}$. The variables $\xi^{S}_{irjk}$ and $\xi^{T}_{irjk}$ capture execution randomness. 

SIREN next uses a selector to convert the selection scores into weights for each split, and combines those weights with the held-out scores and form a weighted average. This aims to reduce dependence on any single partition while preserving the separation between selection and evaluation. We denote a selector as a map $g_{j,B}:\mathbb R^{K_{j,B}}\to\Delta^{K_{j,B}}$ that converts the above selection scores into deployment weights. In the main experiments, the selector is ordinary softmax, $[g_{j,B}(s)]_k=\exp(s_k/\tau)/\sum_{\ell=1}^{K_{j,B}}\exp(s_\ell/\tau)$, with $\tau=1$. Smaller $\tau$ concentrates more weight on the highest-scoring artifact, while larger $\tau$ spreads the weights. The theory permits any smooth selector satisfying the stated conditions.

The split-level selected output is denoted as $\widehat q_{r,j}(B)=g_{j,B}\big(\widehat S_{r,j}(B)\big)$, and the corresponding weighted-average of the held-out score is the form of $\widehat Y_{r,j}(B)=\widehat q_{r,j}(B)^\top \widehat T_{r,j}(B)$. This is the expected score of a randomized single-artifact rule that deploys artifact $k$ with probability $\widehat q_{r,j,k}(B)$; each realized deployment uses one artifact. For computing the final estimator, given nonnegative weights $\omega_1,\dots,\omega_R$ with $\sum_{r=1}^R\omega_r=1$, the SIREN point estimator is
\small
\begin{equation}
\widetilde\theta_{j,R}(B)=\sum_{r=1}^R \omega_r\,\widehat Y_{r,j}(B),
\label{eq:siren_estimator}
\end{equation}
\normalsize
Here $\widehat q_{r,j}(B)$ weights artifacts within split $r$, whereas $\omega_r$ weights the split-level estimates. We use $\omega_r=1/R$ in the experiments; weights proportional to $|E_r|$ are also allowed. The next subsection describes the Gaussian multiplier bootstrap used to form confidence intervals and simultaneous bands.

\subsection{Bootstrap}
Once the estimator in \eqref{eq:siren_estimator} is obtained, SIREN asks how variable this estimator would be if the evaluation pool were redrawn while the completed tuning stage and repeated-split design were held fixed. 
To formalize this target, we define the average score that the frozen shortlist, repeated-split protocol, and selector would produce under the realized split design:
\small
\begin{equation}
\theta_{j,M}^{\mathrm{RS}\mid\mathcal G}(B)
=
\mathbb E\big[\widetilde\theta_{j,R}(B)\mid \mathcal G_M\big].
\label{eq:siren_target}
\end{equation}
\normalsize

The bootstrap therefore aims to characterize the uncertainty gap
$\widetilde\theta_{j,R}(B)-\theta_{j,M}^{\mathrm{RS}\mid\mathcal G}(B)$. 
This gap is not a simple sample-mean error because each benchmark item can affect the estimator in two ways: through held-out evaluation scores and through scoring-subset scores that change the splitwise selection weights. 
Fix $j,B$ and suppress these indices. Let $\widehat\Gamma^T_{ir}$ denote the centered artifact-score vector for item $i$ in split $r$, scaled by $M/|E_r|$ when $i\in E_r$ and zero otherwise; define $\widehat\Gamma^S_{ir}$ analogously on $D_r^{\mathrm{score}}$. The computable item-level contribution is
\small
\begin{equation}
\widehat\psi_i
=\sum_{r=1}^R\omega_r
\left[
\widehat q_r^\top\widehat\Gamma^T_{ir}
+\widehat T_r^\top Dg(\widehat S_r)\widehat\Gamma^S_{ir}
\right].
\end{equation}
\normalsize
The first term is the direct held-out evaluation effect, and the second is the scoring-subset effect through the selector derivative. The complete indexed definitions are given in Appendix~\ref{app:siren_proofs}.
Intuitively, $\widehat\psi_{i,j,B,M}$ estimates the summand in the linear approximation $\sqrt M\{\widetilde\theta_{j,R}(B)-\theta_{j,M}^{\mathrm{RS}\mid\mathcal G}(B)\}
\approx
\frac{1}{\sqrt M}\sum_{i=1}^M \psi_{i,j,B,M}$. For each $(j,B)$, define the empirical average contribution $\bar{\widehat\psi}_{j,B,M}
=
\frac{1}{M}\sum_{i=1}^M \widehat\psi_{i,j,B,M}$, SIREN then draws independent Gaussian multipliers
$\zeta_1,\dots,\zeta_M\stackrel{\mathrm{i.i.d.}}{\sim}N(0,1)$ and forms the joint bootstrap process $\widehat{\mathbf G}_M^\ast
=
\left\{
\widehat G_{j,B}^{\ast}
\right\}_{j\le N,\ B\in\mathcal B}$ with $\widehat G_{j,B}^{\ast}
=
\frac{1}{\sqrt M}
\sum_{i=1}^M
\zeta_i
\big(
\widehat\psi_{i,j,B,M}
-
\bar{\widehat\psi}_{j,B,M}
\big)$. This randomly reweights benchmark items rather than splits or candidate artifacts. 
Using the same multiplier $\zeta_i$ for item $W_i$ across all systems and budgets preserves the dependence created by evaluating the same benchmark pool across the whole grid.

The conditional distribution of $\widehat{\mathbf G}_M^\ast$ is used as an approximation to the sampling distribution of $\left\{
\sqrt M\big(
\widetilde\theta_{j,R}(B)
-
\theta_{j,M}^{\mathrm{RS}\mid\mathcal G}(B)
\big)
\right\}_{j\le N,\ B\in\mathcal B}$. Pointwise intervals use the corresponding coordinatewise bootstrap quantiles, while simultaneous bands use the quantile of $\max_{j\le N,\ B\in\mathcal B}
|\widehat G_{j,B}^{\ast}|$. The same bootstrap process also supports fixed equal-budget comparisons, cross-budget gains, and other pre-specified linear contrasts on the grid. 
The theory below justifies this item-level resampling scheme for the selected repeated-split estimator, including both held-out evaluation noise and the first-order effect of selection.

\begin{figure}[h]
    \centering
    \includegraphics[width=\textwidth]{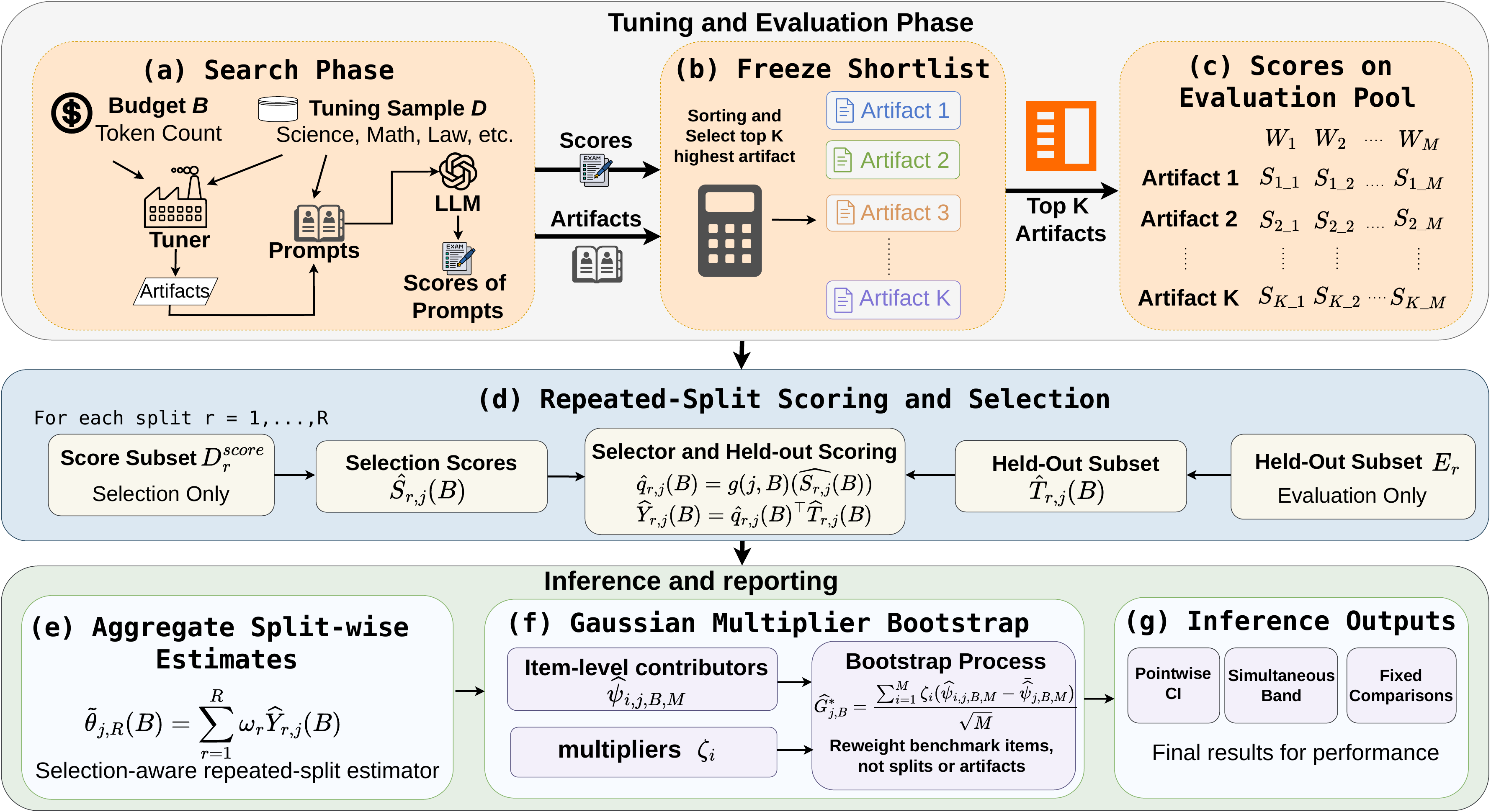}
    \caption{Overview of SIREN. Panels (a)-(b) show tuning with budget $B$ and freezing the shortlist for one system $j$; panel (c) shows the per-item scores of its retained artifacts on evaluation items $W_1,\dots,W_M$; panels (d)-(e) show scoring, held-out evaluation, and aggregation for splits $r=1,\dots,R$; panels (f)-(g) stack the item-level contributions across all $N$ systems and budgets for bootstrap inference.}
    \label{fig:siren_pipeline}
\end{figure}

\subsection{Theoretical guarantees}
\label{sec:theory}

We now justify the two statistical outputs of SIREN: the repeated-split estimate and the bootstrap confidence interval. The formal assumptions, full details of SIREN's computation steps, and proofs are given in Appendix~\ref{app:siren_proofs}. 
All statistical statements are conditioned on the completed tuning stage, so the retained shortlists and selectors are treated as fixed independent objects in SIREN's pipeline.

In the following theorem-ready regime, we consider the number of systems, budgets, repeated splits, and retained artifacts are fixed; split sizes are proportional to $M$; and the selector $g_{j,B}$ is a stabilized smooth map with uniformly bounded derivatives. 
Under these settings, the softmax-type selectors are covered directly, while a discontinuous hard argmax requires smoothing or an additional margin condition. We first define the joint scaled error over the finite system-budget grid as
\[
\mathbf Z_M
=
\left\{
\sqrt M
\big(
\widetilde\theta_{j,R}(B)
-
\theta_{j,M}^{\mathrm{RS}\mid\mathcal G}(B)
\big)
\right\}_{j\le N,\ B\in\mathcal B}.
\]

\begin{theorem}[Selection-aware CLT]
\label{thm:positive_smooth}
Conditioned on the repeated-split design, the joint error vector $\mathbf Z_M$ admits a first-order item-level linearization and converges to a mean-zero Gaussian law over the budget grid $\mathcal{B}$. 
The item-level linearization has one summand per benchmark task instance and includes both sources of uncertainty: the direct held-out evaluation effect and the first-order effect of using scoring subsets to choose splitwise weights.
\end{theorem}

For each item $W_i$, the proof characterizes all of its impact across splits, systems, budgets, and retained artifacts into one item-level vector. 
These vectors are independent across items, although their coordinates are dependent within an item because the same task can appear in multiple split-level comparisons. 
A finite-dimensional CLT then applies to the full score table. 
The nonstandard step is the selection layer: the splitwise weights are chosen from the scoring subset rather than fixed in advance. 
Smoothness of $g_{j,B}$ lets us linearize how scoring-set perturbations change these weights, producing a derivative-weighted selection term in addition to ordinary held-out evaluation noise. 
This is the uncertainty ignored by winner-based reports that treat the selected artifact as fixed.

Theorem~\ref{thm:positive_smooth} explains what SIREN is estimating and motivates the repeated-split construction. 
The target is not the score of a single observed prompt, program, or demonstration set chosen after search, it is instead the performance of a frozen artifact shortlist plus a specified reporting rule. 
The theorem shows that this selected reporting rule still has a tractable large-sample distribution once the effect of selection is included. The following theorem further explains the validity of the bootstrapping approach.

\begin{theorem}[Multiplier bootstrap validity]
\label{thm:bootstrap_validity}
Suppose the computable item-level quantities used by the bootstrap consistently approximate the corresponding first-order item effects in the linearization, in the sense that their average squared discrepancy over benchmark items vanishes as $M$ grows. 
Then the conditional law of the multiplier process $\widehat{\mathbf G}_M^\ast$ consistently approximates the joint law of $\mathbf Z_M$. 
Consequently, the bootstrap quantiles used in Section~\ref{sec:siren} yield asymptotically valid pointwise confidence intervals and simultaneous bands for 
$\theta_{j,M}^{\mathrm{RS}\mid\mathcal G}(B)$ over  the budget grid $\mathcal{B}$.
\end{theorem}

Theorem~\ref{thm:bootstrap_validity} gives the computational payoff. 
SIREN resamples benchmark-item contributions rather than rerunning the tuner or regenerating candidates. 
Because each item contribution already encodes both evaluation noise and selection effects, one multiplier per item preserves dependence across systems, budgets, and splits. 
Thus the same bootstrap supports confidence intervals, simultaneous bands, and fixed grid contrasts. Conceptually, the condition in Theorem~\ref{thm:bootstrap_validity} requires the computable item-level contributions to recover the same two first-order effects from Theorem~\ref{thm:positive_smooth}: the direct held-out score effect and the derivative-weighted selection effect. 
The proof compares the feasible multiplier process with the infeasible linear expansion; conditional on the observed contributions, the bootstrap is Gaussian with empirical covariance, and plug-in consistency makes this covariance match the limiting covariance of $\mathbf Z_M$. 

\section{Experiments}
\label{sec:exp_setup}

Adaptive LLM benchmarking is vulnerable to a winner's curse: when benchmark items are reused during prompt or program search, the reported winner can overstate the fresh-data performance of the full procedure. Our experiments evaluate whether this bias appears after real tuning and whether SIREN effectively estimates its conditional post-search target. We organize these studies around the following three Research Questions (\textbf{RQ}s):

\textbf{(RQ1)} Does SIREN estimate the procedure-level target without winner-induced optimism?

\textbf{(RQ2)} Does selection-aware reporting change tuned-versus-default deployment decisions?

\textbf{(RQ3)} Is budget-aware prompt evaluation over a pool sufficient after adaptive tuning, or is selection-aware inference still needed?

We evaluate SIREN on MMLU-Pro~\cite{wang2024mmlupro} \emph{Math} and \emph{Law} across two tuner families, random search and DSPy, and eleven open-weight instruction-tuned LLMs. For each subject, the tuner uses a small training pool to generate candidate artifacts under the shared token-budget grid $\mathcal B=\{500\mathrm{K},1.5\mathrm{M},3\mathrm{M},6.5\mathrm{M}\}$, and the reporting layer uses a disjoint evaluation pool to estimate procedure-level performance. All methods are compared on the same retained shortlist and evaluation pool, differing only in the final reporting rule. Unless otherwise stated, SIREN uses $R{=}10$ repeated splits, scoring fraction $\rho_{\mathrm{score}}{=}0.5$, softmax selector temperature $\tau{=}1$, uniform split weights $\omega_r$, and $2{,}000$ Gaussian multiplier bootstrap draws.

We compare SIREN with four internal baselines: 
\begin{itemize}
    \item \textbf{M1/M2} (\emph{naive max / naive max + Wald}) evaluate every retained artifact on the full evaluation pool and use the same selected-winner point estimate, $\widehat\theta^{\mathrm{M1}}=\max_k\bar Z_{\cdot,k}$. M2 differs only by adding a conventional Wald interval computed after selection. We combine M1 and M2 in point-estimate comparisons and distinguish them only when evaluating confidence intervals;
    \item \textbf{M3} (\emph{single-split holdout}) partitions the evaluation set into a selection fold and a held-out scoring fold, picks the empirical winner on the selection fold, and reports its mean accuracy on the held-out fold;
    \item \textbf{M4} (\emph{$R$-split argmax with Student-$t$ CI}) repeats the selection-and-evaluation procedure $R$ times with independent splits, takes the average of the per-split held-out means as the point estimate, and forms a Student-$t$ confidence interval from the variance across splits.
\end{itemize}

Appendix~\ref{sec:track1} provides controlled validation of the inference mechanism. Study~A checks multiplier-bootstrap calibration and the predicted $M^{-1/2}$ width scaling. Study~B shows why smooth selection matters near ties: hard-argmax coverage falls to $89.8\%$, whereas softmax coverage remains between $93.7\%$ and $96.0\%$ across the quality-gap sweep. Study~C shows that same-data best-of reporting becomes more optimistic as the number of searched artifacts grows. We also compare against \textbf{PromptEval}~\cite{polo2024efficient}, an IRT-based method for efficient evaluation over a prespecified prompt pool. It observes a fraction $f$ of the $K\times M$ prompt--item score matrix, fits a Rasch model, and imputes the remaining cells. For the Math/Law common-reference comparisons, $\theta^\star$ is computed from $100,000$ random permutations and complementary splits of the fixed observed score matrix. This finite-matrix reference is distinct from the fresh-item target in~\eqref{eq:siren_target}. In the original comparisons, ``bias'' denotes the signed difference from this common soft-rule reference. For M1--M4, it can include both estimation error and differences between reporting rules. Appendix~\ref{app:matched_targets} separately compares each method with its own target. We note that all models are served via vLLM on a single NVIDIA RTX~4090 GPU (24\,GB VRAM) with sequential model loading, and the interval calibration is evaluated separately.

\subsection{RQ1: Procedure-level target and winner-induced bias}
\label{sec:exp_main}

We first answer \textbf{RQ1} by comparing each reported tuned estimate $\widehat\theta_A$ with the Monte Carlo reference $\theta_A^\star$. Each cell is a (subject, tuner, model, budget) configuration on the grid described above, with the model ranging over the 11 open-weight models that appear in all four subject vs tuner combinations. We report three summaries. The column ``bias'' is $\widehat\theta_A-\theta_A^\star$, in percentage points; positive values indicate scores above the common reporting reference. The column ``dir'' counts the number of cells in which the point-estimate ordering agrees with the Monte Carlo ordering $\operatorname{sign}(\widehat\theta_A-\theta_B^\star)
=
\operatorname{sign}(\theta_A^\star-\theta_B^\star)$. The column ``MAE'' averages $|\widehat\theta_A-\theta_A^\star|$ across cells, in percentage points. Exact ties with the default count as incorrect directions. Tables~\ref{tab:math_directional} and~\ref{tab:law_directional} bold the best direction count, signed bias closest to zero, and smallest MAE, including ties at the displayed precision. Per-model tables bold estimates closest to the reference by the same precision rule.

\begin{table}[h]
\centering
\caption{Directional accuracy and common-reference errors on MMLU-Pro Math.}\vspace{-5pt}
\label{tab:math_directional}
\small
\setlength{\tabcolsep}{3pt}
\resizebox{\textwidth}{!}{%
\begin{tabular}{ll|ccc|ccc|ccc|ccc}
\toprule
& & \multicolumn{3}{c|}{M1/M2} & \multicolumn{3}{c|}{M3} & \multicolumn{3}{c|}{M4} & \multicolumn{3}{c}{SIREN} \\
Tuner & $B$ & dir & bias & MAE & dir & bias & MAE & dir & bias & MAE & dir & bias & MAE \\
\midrule
Random search & 500K & $6/11$ & $+5.19$ & $5.19$ & $8/11$ & $+3.83$ & $3.90$ & $7/11$ & $+4.48$ & $4.48$ & $\mathbf{11/11}$ & $\mathbf{-0.09}$ & $\mathbf{0.13}$ \\
 & 1.5M & $5/11$ & $+2.50$ & $2.50$ & $9/11$ & $+1.19$ & $1.59$ & $6/11$ & $+1.84$ & $1.84$ & $\mathbf{10/11}$ & $\mathbf{+0.01}$ & $\mathbf{0.13}$ \\
 & 3M & $4/11$ & $+2.49$ & $2.49$ & $9/11$ & $+0.65$ & $1.17$ & $8/11$ & $+1.45$ & $1.45$ & $\mathbf{10/11}$ & $\mathbf{+0.01}$ & $\mathbf{0.14}$ \\
 & 6.5M & $6/11$ & $+2.09$ & $2.09$ & $9/11$ & $+0.26$ & $1.16$ & $8/11$ & $+1.21$ & $1.35$ & $\mathbf{11/11}$ & $\mathbf{+0.02}$ & $\mathbf{0.17}$ \\
\textit{All} & \textit{(44 cells)} & $21/44$ & $+3.06$ & $3.06$ & $35/44$ & $+1.48$ & $1.95$ & $29/44$ & $+2.25$ & $2.28$ & $\mathbf{42/44}$ & $\mathbf{-0.01}$ & $\mathbf{0.14}$ \\
\midrule
DSPy & 500K & $9/11$ & $+1.37$ & $1.37$ & $6/11$ & $-0.76$ & $1.80$ & $9/11$ & $+0.55$ & $0.94$ & $\mathbf{10/11}$ & $\mathbf{-0.09}$ & $\mathbf{0.21}$ \\
 & 1.5M & $8/11$ & $+1.52$ & $1.52$ & $8/11$ & $-0.54$ & $1.16$ & $10/11$ & $+0.78$ & $1.13$ & $\mathbf{11/11}$ & $\mathbf{-0.03}$ & $\mathbf{0.24}$ \\
 & 3M & $8/11$ & $+1.49$ & $1.49$ & $9/11$ & $\mathbf{+0.00}$ & $1.00$ & $10/11$ & $+0.70$ & $1.01$ & $\mathbf{11/11}$ & $-0.05$ & $\mathbf{0.26}$ \\
 & 6.5M & $8/11$ & $+1.33$ & $1.33$ & $8/11$ & $-0.17$ & $0.90$ & $10/11$ & $+0.63$ & $0.92$ & $\mathbf{11/11}$ & $\mathbf{-0.03}$ & $\mathbf{0.24}$ \\
\textit{All} & \textit{(44 cells)} & $33/44$ & $+1.43$ & $1.43$ & $31/44$ & $-0.37$ & $1.21$ & $39/44$ & $+0.67$ & $1.00$ & $\mathbf{43/44}$ & $\mathbf{-0.05}$ & $\mathbf{0.24}$ \\
\bottomrule
\end{tabular}}
\end{table}

\begin{table}[h]
\centering
\caption{Directional accuracy and common-reference errors on MMLU-Pro Law.}\vspace{-5pt}
\label{tab:law_directional}
\small
\setlength{\tabcolsep}{3pt}
\resizebox{\textwidth}{!}{%
\begin{tabular}{ll|ccc|ccc|ccc|ccc}
\toprule
& & \multicolumn{3}{c|}{M1/M2} & \multicolumn{3}{c|}{M3} & \multicolumn{3}{c|}{M4} & \multicolumn{3}{c}{SIREN} \\
Tuner & $B$ & dir & bias & MAE & dir & bias & MAE & dir & bias & MAE & dir & bias & MAE \\
\midrule
Random search & 500K & $4/11$ & $+1.50$ & $1.50$ & $3/11$ & $+2.13$ & $2.13$ & $8/11$ & $+1.00$ & $1.06$ & $\mathbf{11/11}$ & $\mathbf{+0.05}$ & $\mathbf{0.31}$ \\
 & 1.5M & $2/11$ & $+2.24$ & $2.24$ & $5/11$ & $+1.59$ & $2.49$ & $5/11$ & $+1.37$ & $1.68$ & $\mathbf{10/11}$ & $\mathbf{+0.11}$ & $\mathbf{0.25}$ \\
 & 3M & $4/11$ & $+2.10$ & $2.10$ & $6/11$ & $+1.47$ & $2.11$ & $6/11$ & $+1.18$ & $1.48$ & $\mathbf{11/11}$ & $\mathbf{+0.07}$ & $\mathbf{0.22}$ \\
 & 6.5M & $5/11$ & $+1.77$ & $1.77$ & $6/11$ & $+1.74$ & $1.92$ & $8/11$ & $+1.13$ & $1.21$ & $\mathbf{10/11}$ & $\mathbf{+0.10}$ & $\mathbf{0.22}$ \\
\textit{All} & \textit{(44 cells)} & $15/44$ & $+1.90$ & $1.90$ & $20/44$ & $+1.73$ & $2.16$ & $27/44$ & $+1.17$ & $1.36$ & $\mathbf{42/44}$ & $\mathbf{+0.08}$ & $\mathbf{0.25}$ \\
\midrule
DSPy & 500K & $\mathbf{10/11}$ & $+0.77$ & $0.77$ & $8/11$ & $\mathbf{-0.01}$ & $1.25$ & $8/11$ & $-0.15$ & $0.61$ & $9/11$ & $-0.21$ & $\mathbf{0.34}$ \\
 & 1.5M & $10/11$ & $+1.10$ & $1.10$ & $9/11$ & $+0.72$ & $1.15$ & $\mathbf{11/11}$ & $+0.25$ & $0.63$ & $10/11$ & $\mathbf{-0.17}$ & $\mathbf{0.33}$ \\
 & 3M & $10/11$ & $+1.13$ & $1.13$ & $9/11$ & $+0.77$ & $1.56$ & $\mathbf{11/11}$ & $+0.22$ & $0.44$ & $10/11$ & $\mathbf{-0.20}$ & $\mathbf{0.34}$ \\
 & 6.5M & $10/11$ & $+1.04$ & $1.04$ & $9/11$ & $+0.25$ & $1.37$ & $10/11$ & $\mathbf{-0.01}$ & $0.61$ & $\mathbf{11/11}$ & $-0.21$ & $\mathbf{0.34}$ \\
\textit{All} & \textit{(44 cells)} & $\mathbf{40/44}$ & $+1.01$ & $1.01$ & $35/44$ & $+0.43$ & $1.33$ & $\mathbf{40/44}$ & $\mathbf{+0.08}$ & $0.57$ & $\mathbf{40/44}$ & $-0.20$ & $\mathbf{0.34}$ \\
\bottomrule
\end{tabular}}
\end{table}

Tables~\ref{tab:math_directional} and~\ref{tab:law_directional} show a consistent positive bias for winner-based reporting. M1 is optimistic in all four subject and tuner summaries. M2 has the same point estimate as M1, and therefore inherits the same signed bias; the Wald interval changes only the uncertainty report, not the selected-winner estimate.
M3 and M4 reduce direct same-data reuse but retain hard selection, so their differences from the soft-rule reference are not pure estimation bias. SIREN remains close to the Monte Carlo reference across budgets and tuners, with mean signed errors between $-0.20$ and $+0.08$ percentage points in the four aggregate rows. For \textbf{RQ1}, these tables compare proximity to the common reference. Appendix~\ref{app:matched_targets} evaluates accuracy against fresh-data targets for each reporting rule.

\subsection{Comparison against reporting baselines M1--M4}
\label{sec:exp_baselines}
We next answer \textbf{RQ2} by examining whether the reporting differences in Tables~\ref{tab:math_directional}-\ref{tab:law_directional} change the tuned-versus-default conclusion. For each model and budget, we compare the sign of the reported gap 
$\operatorname{sign}(\widehat\theta_A-\theta_B^\star)$ with the ground-truth direction 
$\operatorname{sign}(\theta_A^\star-\theta_B^\star)$. 
Table~\ref{tab:siren_math_dspy_all_budgets} gives one representative comparison on MMLU-Pro {Math} with the DSPy tuner for eleven models shown as columns. The default reference $\theta_B^\star$ is fixed across budgets and is shown at the top row. For each budget, $\theta_A^\star$ is the Monte Carlo tuned-procedure reference, and each method row reports a score $\widehat\theta_A$ compared with this common reference. The \cmark/\xmark marker records whether $\operatorname{sign}(\widehat\theta_A-\theta_B^\star)$ matches $\operatorname{sign}(\theta_A^\star-\theta_B^\star)$.

\begin{table*}[h]
\centering
\caption{
Unified per-model comparison across budgets on MMLU-Pro Math with the \textbf{DSPy} tuner.}
\label{tab:siren_math_dspy_all_budgets}
\scriptsize
\setlength{\tabcolsep}{2.5pt}
\renewcommand{\arraystretch}{0.92}
\resizebox{\textwidth}{!}{%
\begin{tabular}{llccccccccccc}
\toprule
$B$ & Row
& Qwen3 & Phi-3.5 & Qwen2.5-7B & Llama3.1 & Yi-1.5
& InternLM & Qwen2-7B & Qwen2.5-3B & GLM-4 & Mistr-v0.3 & Mistr-v0.1 \\
\midrule

Ref. & $\theta_B^\star$
& 0.373 & 0.263 & 0.306 & 0.222 & 0.196 & 0.249 & 0.257 & 0.249 & 0.255 & 0.169 & 0.130 \\

\midrule
\multirow{6}{*}{$500K$}
& $\theta_A^\star$
& 0.379 & 0.267 & 0.339 & 0.230 & 0.161 & 0.253 & 0.258 & 0.283 & 0.254 & 0.165 & 0.140 \\
& True dir.
& A{>}B & A{>}B & A{>}B & A{>}B & A{<}B & A{>}B & A{>}B & A{>}B & A{<}B & A{<}B & A{>}B \\
& M1/M2
& 0.388\,\cmark & 0.285\,\cmark & 0.349\,\cmark & 0.240\,\cmark & 0.196\,\xmark & 0.265\,\cmark & 0.265\,\cmark & 0.322\,\cmark & \textbf{0.257}\,\xmark & 0.168\,\cmark & 0.146\,\cmark \\
& M3
& 0.369\,\xmark & 0.281\,\cmark & 0.345\,\cmark & 0.216\,\xmark & 0.177\,\cmark & 0.232\,\xmark & 0.217\,\xmark & 0.304\,\cmark & 0.232\,\cmark & 0.152\,\cmark & 0.121\,\xmark \\
& M4
& 0.382\,\cmark & 0.274\,\cmark & 0.343\,\cmark & \textbf{0.230}\,\cmark & 0.196\,\xmark & 0.262\,\cmark & 0.253\,\xmark & 0.306\,\cmark & 0.245\,\cmark & 0.160\,\cmark & 0.138\,\cmark \\
& SIREN
& \textbf{0.380}\,\cmark & \textbf{0.264}\,\cmark & \textbf{0.342}\,\cmark & \textbf{0.230}\,\cmark & \textbf{0.162}\,\cmark & \textbf{0.254}\,\cmark & \textbf{0.256}\,\xmark & \textbf{0.278}\,\cmark & 0.249\,\cmark & \textbf{0.164}\,\cmark & \textbf{0.139}\,\cmark \\

\midrule
\multirow{6}{*}{$1.5M$}
& $\theta_A^\star$
& 0.383 & 0.278 & 0.344 & 0.251 & 0.193 & 0.254 & 0.262 & 0.290 & 0.255 & 0.166 & 0.142 \\
& True dir.
& A{>}B & A{>}B & A{>}B & A{>}B & A{<}B & A{>}B & A{>}B & A{>}B & A{<}B & A{<}B & A{>}B \\
& M1/M2
& 0.398\,\cmark & 0.305\,\cmark & 0.350\,\cmark & 0.282\,\cmark & 0.214\,\xmark & 0.262\,\cmark & 0.267\,\cmark & 0.325\,\cmark & \textbf{0.257}\,\xmark & 0.174\,\xmark & 0.149\,\cmark \\
& M3
& \textbf{0.382}\,\cmark & 0.267\,\cmark & 0.343\,\cmark & 0.257\,\cmark & 0.198\,\xmark & 0.222\,\xmark & 0.240\,\xmark & 0.312\,\cmark & 0.238\,\cmark & 0.168\,\cmark & 0.133\,\cmark \\
& M4
& 0.396\,\cmark & 0.295\,\cmark & \textbf{0.344}\,\cmark & 0.274\,\cmark & 0.211\,\xmark & 0.262\,\cmark & 0.258\,\cmark & 0.316\,\cmark & 0.248\,\cmark & 0.163\,\cmark & 0.137\,\cmark \\
& SIREN
& 0.387\,\cmark & \textbf{0.275}\,\cmark & 0.348\,\cmark & \textbf{0.250}\,\cmark & \textbf{0.193}\,\cmark & \textbf{0.257}\,\cmark & \textbf{0.260}\,\cmark & \textbf{0.284}\,\cmark & \textbf{0.253}\,\cmark & \textbf{0.165}\,\cmark & \textbf{0.143}\,\cmark \\

\midrule
\multirow{6}{*}{$3M$}
& $\theta_A^\star$
& 0.384 & 0.279 & 0.342 & 0.251 & 0.193 & 0.250 & 0.264 & 0.294 & 0.254 & 0.164 & 0.142 \\
& True dir.
& A{>}B & A{>}B & A{>}B & A{>}B & A{<}B & A{>}B & A{>}B & A{>}B & A{<}B & A{<}B & A{>}B \\
& M1/M2
& 0.398\,\cmark & 0.305\,\cmark & 0.349\,\cmark & 0.285\,\cmark & 0.211\,\xmark & 0.263\,\cmark & 0.269\,\cmark & 0.321\,\cmark & 0.257\,\xmark & 0.176\,\xmark & 0.149\,\cmark \\
& M3
& \textbf{0.382}\,\cmark & 0.302\,\cmark & \textbf{0.342}\,\cmark & 0.262\,\cmark & \textbf{0.193}\,\cmark & 0.246\,\xmark & 0.240\,\xmark & 0.310\,\cmark & 0.241\,\cmark & 0.169\,\cmark & 0.133\,\cmark \\
& M4
& 0.396\,\cmark & 0.293\,\cmark & \textbf{0.342}\,\cmark & 0.275\,\cmark & 0.205\,\xmark & 0.266\,\cmark & 0.259\,\cmark & 0.307\,\cmark & 0.246\,\cmark & 0.168\,\cmark & 0.139\,\cmark \\
& SIREN
& 0.389\,\cmark & \textbf{0.275}\,\cmark & 0.347\,\cmark & \textbf{0.249}\,\cmark & 0.194\,\cmark & \textbf{0.252}\,\cmark & \textbf{0.263}\,\cmark & \textbf{0.287}\,\cmark & \textbf{0.252}\,\cmark & \textbf{0.163}\,\cmark & \textbf{0.143}\,\cmark \\

\midrule
\multirow{6}{*}{$6.5M$}
& $\theta_A^\star$
& 0.384 & 0.274 & 0.343 & 0.252 & 0.193 & 0.252 & 0.263 & 0.302 & 0.255 & 0.169 & 0.143 \\
& True dir.
& A{>}B & A{>}B & A{>}B & A{>}B & A{<}B & A{>}B & A{>}B & A{>}B & A{<}B & A{<}B & A{>}B \\
& M1/M2
& 0.398\,\cmark & 0.293\,\cmark & 0.349\,\cmark & 0.284\,\cmark & 0.208\,\xmark & 0.263\,\cmark & 0.269\,\cmark & 0.330\,\cmark & \textbf{0.257}\,\xmark & 0.176\,\xmark & 0.149\,\cmark \\
& M3
& \textbf{0.382}\,\cmark & 0.288\,\cmark & \textbf{0.342}\,\cmark & 0.262\,\cmark & 0.196\,\xmark & 0.246\,\xmark & 0.240\,\xmark & 0.315\,\cmark & 0.238\,\cmark & \textbf{0.169}\,\cmark & 0.133\,\cmark \\
& M4
& 0.396\,\cmark & 0.283\,\cmark & 0.345\,\cmark & 0.273\,\cmark & 0.202\,\xmark & 0.266\,\cmark & 0.259\,\cmark & 0.320\,\cmark & 0.248\,\cmark & 0.167\,\cmark & 0.139\,\cmark \\
& SIREN
& 0.389\,\cmark & \textbf{0.270}\,\cmark & 0.347\,\cmark & \textbf{0.251}\,\cmark & \textbf{0.194}\,\cmark & \textbf{0.254}\,\cmark & \textbf{0.261}\,\cmark & \textbf{0.296}\,\cmark & \textbf{0.253}\,\cmark & 0.168\,\cmark & \textbf{0.143}\,\cmark \\

\bottomrule
\end{tabular}}
\end{table*}

The table shows that winner-based reporting can change the deployment verdict, especially in small-gap cases. Because M1 and M2 use the same selected-winner point estimate, they inherit the same optimistic shift and can turn near ties or weak regressions into apparent gains. M3 and M4 reduce direct reuse of the evaluation pool, but their hard splitwise selections remain sensitive to the realized split. SIREN remains close to the common soft-rule reference, including in near-tie cases. The reporting differences in Tables~\ref{tab:math_directional}--\ref{tab:law_directional} can therefore change which system would be deployed. The matched-target analysis distinguishes estimation error from the choice of reporting rule. We refer to the remaining subject vs tuner cells in Tables~\ref{tab:siren_math_rs_all_budgets}, \ref{tab:siren_law_rs_all_budgets}, and~\ref{tab:siren_law_dspy_all_budgets} in the appendix. These results answer \textbf{RQ2} affirmatively. Selection-aware reporting can change tuned-versus-default conclusions because the same-data winner and the tune-then-deploy procedure are different statistical objects. SIREN evaluates tuned gains under its specified soft-selection rule using held-out splitwise scores.

\subsection{Comparison against PromptEval}
\label{sec:exp_prompteval}

To answer \textbf{RQ3}, we compared PromptEval~\cite{polo2024efficient}, the closest external baseline to our setting since it also studies budget-aware evaluation over a prompt pool. We sweep PromptEval's cell-observation fraction $f \in \{0.05,0.10,0.20,0.40,0.60,0.80,1.00\}$, where $f$ is the fraction of the $K\times M$ prompt--item score matrix observed before fitting the Rasch IRT model and imputing the remaining cells. Small $f$ corresponds to the sparse budget-constrained regime. When $f{=}1$, PromptEval gives the same selected-winner point estimate as M1 in our comparison. PromptEval addresses efficient evaluation over a fixed prompt pool. SIREN instead begins after an adaptive tuner has frozen a shortlist and accounts for splitwise selection in both the point estimate and its uncertainty.

\begin{table*}[h]
\centering
\caption{
Unified PromptEval~\cite{polo2024efficient} vs SIREN comparison on MMLU-Pro Math with the \textbf{DSPy} tuner.}
\label{tab:prompteval_math_dspy_all_budgets}
\scriptsize
\setlength{\tabcolsep}{2.5pt}
\renewcommand{\arraystretch}{0.92}
\resizebox{\textwidth}{!}{%
\begin{tabular}{llccccccccccc}
\toprule
$B$ & Row
& Qwen3 & Phi-3.5 & Qwen2.5-7B & Llama3.1 & Yi-1.5
& InternLM & Qwen2-7B & Qwen2.5-3B & GLM-4 & Mistr-v0.3 & Mistr-v0.1 \\
\midrule

Ref. & $\theta_B^\star$
& 0.373 & 0.263 & 0.306 & 0.222 & 0.196 & 0.249 & 0.257 & 0.249 & 0.255 & 0.169 & 0.130 \\

\midrule
\multirow{10}{*}{$500K$}
& $\theta_A^\star$
& 0.379 & 0.267 & 0.339 & 0.230 & 0.161 & 0.253 & 0.258 & 0.283 & 0.254 & 0.165 & 0.140 \\
& True dir.
& A{>}B & A{>}B & A{>}B & A{>}B & A{<}B & A{>}B & A{>}B & A{>}B & A{<}B & A{<}B & A{>}B \\
& PE $f{=}0.05$
& 0.254\,\xmark & 0.193\,\xmark & 0.267\,\xmark & 0.136\,\xmark & 0.109\,\cmark & 0.180\,\xmark & 0.157\,\xmark & 0.181\,\xmark & 0.164\,\cmark & 0.092\,\cmark & 0.076\,\xmark \\
& PE $f{=}0.10$
& 0.305\,\xmark & 0.204\,\xmark & 0.267\,\xmark & 0.180\,\xmark & 0.115\,\cmark & 0.199\,\xmark & 0.194\,\xmark & 0.227\,\xmark & 0.226\,\cmark & 0.125\,\cmark & 0.117\,\xmark \\
& PE $f{=}0.20$
& 0.401\,\cmark & \textbf{0.267}\,\cmark & 0.349\,\cmark & 0.245\,\cmark & 0.164\,\cmark & 0.265\,\cmark & 0.274\,\cmark & 0.305\,\cmark & 0.286\,\xmark & 0.175\,\xmark & 0.162\,\cmark \\
& PE $f{=}0.40$
& 0.399\,\cmark & 0.281\,\cmark & 0.346\,\cmark & 0.252\,\cmark & 0.164\,\cmark & 0.266\,\cmark & 0.270\,\cmark & 0.322\,\cmark & 0.268\,\xmark & 0.183\,\xmark & 0.149\,\cmark \\
& PE $f{=}0.60$
& 0.399\,\cmark & 0.285\,\cmark & 0.346\,\cmark & 0.240\,\cmark & 0.182\,\cmark & 0.265\,\cmark & 0.269\,\cmark & 0.327\,\cmark & 0.267\,\xmark & 0.174\,\xmark & 0.153\,\cmark \\
& PE $f{=}0.80$
& 0.391\,\cmark & 0.286\,\cmark & 0.352\,\cmark & 0.237\,\cmark & 0.192\,\cmark & 0.263\,\cmark & 0.268\,\cmark & 0.326\,\cmark & 0.265\,\xmark & 0.171\,\xmark & 0.149\,\cmark \\
& PE $f{=}1.00$ (M1)
& 0.388\,\cmark & 0.285\,\cmark & 0.349\,\cmark & 0.240\,\cmark & 0.196\,\xmark & 0.265\,\cmark & 0.265\,\cmark & 0.322\,\cmark & \textbf{0.257}\,\xmark & 0.168\,\cmark & 0.146\,\cmark \\
& SIREN
& \textbf{0.380}\,\cmark & 0.264\,\cmark & \textbf{0.342}\,\cmark & \textbf{0.230}\,\cmark & \textbf{0.162}\,\cmark & \textbf{0.254}\,\cmark & \textbf{0.256}\,\xmark & \textbf{0.278}\,\cmark & 0.249\,\cmark & \textbf{0.164}\,\cmark & \textbf{0.139}\,\cmark \\

\midrule
\multirow{10}{*}{$1.5M$}
& $\theta_A^\star$
& 0.383 & 0.278 & 0.344 & 0.251 & 0.193 & 0.254 & 0.262 & 0.290 & 0.255 & 0.166 & 0.142 \\
& True dir.
& A{>}B & A{>}B & A{>}B & A{>}B & A{<}B & A{>}B & A{>}B & A{>}B & A{<}B & A{<}B & A{>}B \\
& PE $f{=}0.05$
& 0.271\,\xmark & 0.175\,\xmark & 0.212\,\xmark & 0.174\,\xmark & 0.123\,\cmark & 0.173\,\xmark & 0.168\,\xmark & 0.227\,\xmark & 0.151\,\cmark & 0.102\,\cmark & 0.079\,\xmark \\
& PE $f{=}0.10$
& 0.304\,\xmark & 0.226\,\xmark & 0.265\,\xmark & 0.218\,\xmark & 0.172\,\cmark & 0.205\,\xmark & 0.207\,\xmark & 0.250\,\cmark & 0.209\,\cmark & 0.133\,\cmark & 0.112\,\xmark \\
& PE $f{=}0.20$
& 0.393\,\cmark & 0.299\,\cmark & 0.353\,\cmark & 0.273\,\cmark & 0.224\,\xmark & 0.268\,\cmark & 0.279\,\cmark & 0.321\,\cmark & 0.264\,\xmark & 0.181\,\xmark & 0.151\,\cmark \\
& PE $f{=}0.40$
& 0.403\,\cmark & 0.302\,\cmark & \textbf{0.345}\,\cmark & 0.291\,\cmark & 0.229\,\xmark & 0.269\,\cmark & 0.276\,\cmark & 0.318\,\cmark & 0.264\,\xmark & 0.180\,\xmark & 0.151\,\cmark \\
& PE $f{=}0.60$
& 0.403\,\cmark & 0.307\,\cmark & 0.346\,\cmark & 0.278\,\cmark & 0.218\,\xmark & 0.265\,\cmark & 0.273\,\cmark & 0.320\,\cmark & 0.259\,\xmark & 0.171\,\xmark & 0.154\,\cmark \\
& PE $f{=}0.80$
& 0.403\,\cmark & 0.308\,\cmark & 0.349\,\cmark & 0.283\,\cmark & 0.210\,\xmark & 0.262\,\cmark & 0.267\,\cmark & 0.324\,\cmark & \textbf{0.257}\,\xmark & 0.170\,\xmark & 0.153\,\cmark \\
& PE $f{=}1.00$ (M1)
& 0.398\,\cmark & 0.305\,\cmark & 0.350\,\cmark & 0.282\,\cmark & 0.214\,\xmark & 0.262\,\cmark & 0.267\,\cmark & 0.325\,\cmark & \textbf{0.257}\,\xmark & 0.174\,\xmark & 0.149\,\cmark \\
& SIREN
& \textbf{0.387}\,\cmark & \textbf{0.275}\,\cmark & 0.348\,\cmark & \textbf{0.250}\,\cmark & \textbf{0.193}\,\cmark & \textbf{0.257}\,\cmark & \textbf{0.260}\,\cmark & \textbf{0.284}\,\cmark & \textbf{0.253}\,\cmark & \textbf{0.165}\,\cmark & \textbf{0.143}\,\cmark \\

\midrule
\multirow{10}{*}{$3M$}
& $\theta_A^\star$
& 0.384 & 0.279 & 0.342 & 0.251 & 0.193 & 0.250 & 0.264 & 0.294 & 0.254 & 0.164 & 0.142 \\
& True dir.
& A{>}B & A{>}B & A{>}B & A{>}B & A{<}B & A{>}B & A{>}B & A{>}B & A{<}B & A{<}B & A{>}B \\
& PE $f{=}0.05$
& 0.298\,\xmark & 0.181\,\xmark & 0.221\,\xmark & 0.182\,\xmark & 0.129\,\cmark & 0.190\,\xmark & 0.167\,\xmark & 0.229\,\xmark & 0.151\,\cmark & 0.107\,\cmark & 0.067\,\xmark \\
& PE $f{=}0.10$
& 0.326\,\xmark & 0.215\,\xmark & 0.272\,\xmark & 0.230\,\cmark & 0.153\,\cmark & 0.193\,\xmark & 0.197\,\xmark & 0.261\,\cmark & 0.213\,\cmark & 0.139\,\cmark & 0.113\,\xmark \\
& PE $f{=}0.20$
& 0.412\,\cmark & \textbf{0.281}\,\cmark & 0.365\,\cmark & 0.273\,\cmark & 0.228\,\xmark & 0.263\,\cmark & 0.274\,\cmark & 0.335\,\cmark & 0.272\,\xmark & 0.177\,\xmark & 0.151\,\cmark \\
& PE $f{=}0.40$
& 0.397\,\cmark & 0.310\,\cmark & 0.353\,\cmark & 0.280\,\cmark & 0.217\,\xmark & 0.268\,\cmark & 0.273\,\cmark & 0.322\,\cmark & 0.279\,\xmark & 0.171\,\xmark & 0.154\,\cmark \\
& PE $f{=}0.60$
& 0.410\,\cmark & 0.306\,\cmark & 0.348\,\cmark & 0.284\,\cmark & 0.203\,\xmark & 0.262\,\cmark & 0.273\,\cmark & 0.316\,\cmark & 0.260\,\xmark & 0.172\,\xmark & 0.149\,\cmark \\
& PE $f{=}0.80$
& 0.402\,\cmark & 0.309\,\cmark & \textbf{0.346}\,\cmark & 0.290\,\cmark & 0.212\,\xmark & 0.261\,\cmark & 0.272\,\cmark & 0.325\,\cmark & 0.257\,\xmark & 0.176\,\xmark & 0.153\,\cmark \\
& PE $f{=}1.00$ (M1)
& 0.398\,\cmark & 0.305\,\cmark & 0.349\,\cmark & 0.285\,\cmark & 0.211\,\xmark & 0.263\,\cmark & 0.269\,\cmark & 0.321\,\cmark & 0.257\,\xmark & 0.176\,\xmark & 0.149\,\cmark \\
& SIREN
& \textbf{0.389}\,\cmark & 0.275\,\cmark & 0.347\,\cmark & \textbf{0.249}\,\cmark & \textbf{0.194}\,\cmark & \textbf{0.252}\,\cmark & \textbf{0.263}\,\cmark & \textbf{0.287}\,\cmark & \textbf{0.252}\,\cmark & \textbf{0.163}\,\cmark & \textbf{0.143}\,\cmark \\

\midrule
\multirow{10}{*}{$6.5M$}
& $\theta_A^\star$
& 0.384 & 0.274 & 0.343 & 0.252 & 0.193 & 0.252 & 0.263 & 0.302 & 0.255 & 0.169 & 0.143 \\
& True dir.
& A{>}B & A{>}B & A{>}B & A{>}B & A{<}B & A{>}B & A{>}B & A{>}B & A{<}B & A{<}B & A{>}B \\
& PE $f{=}0.05$
& 0.302\,\xmark & 0.169\,\xmark & 0.215\,\xmark & 0.179\,\xmark & 0.115\,\cmark & 0.163\,\xmark & 0.192\,\xmark & 0.231\,\xmark & 0.153\,\cmark & 0.110\,\cmark & 0.087\,\xmark \\
& PE $f{=}0.10$
& 0.309\,\xmark & 0.216\,\xmark & 0.269\,\xmark & 0.227\,\cmark & 0.151\,\cmark & 0.196\,\xmark & 0.207\,\xmark & 0.260\,\cmark & 0.204\,\cmark & 0.140\,\cmark & 0.126\,\xmark \\
& PE $f{=}0.20$
& 0.397\,\cmark & 0.283\,\cmark & 0.353\,\cmark & 0.276\,\cmark & 0.224\,\xmark & 0.274\,\cmark & 0.269\,\cmark & 0.330\,\cmark & 0.264\,\xmark & 0.184\,\xmark & 0.164\,\cmark \\
& PE $f{=}0.40$
& 0.411\,\cmark & 0.296\,\cmark & 0.347\,\cmark & 0.295\,\cmark & 0.210\,\xmark & 0.269\,\cmark & 0.272\,\cmark & 0.338\,\cmark & 0.268\,\xmark & 0.187\,\xmark & 0.150\,\cmark \\
& PE $f{=}0.60$
& 0.403\,\cmark & 0.294\,\cmark & \textbf{0.345}\,\cmark & 0.287\,\cmark & 0.211\,\xmark & 0.264\,\cmark & 0.273\,\cmark & 0.328\,\cmark & 0.260\,\xmark & 0.184\,\xmark & 0.148\,\cmark \\
& PE $f{=}0.80$
& 0.401\,\cmark & 0.291\,\cmark & 0.347\,\cmark & 0.290\,\cmark & 0.212\,\xmark & 0.262\,\cmark & 0.270\,\cmark & 0.332\,\cmark & 0.258\,\xmark & 0.180\,\xmark & 0.150\,\cmark \\
& PE $f{=}1.00$ (M1)
& 0.398\,\cmark & 0.293\,\cmark & 0.349\,\cmark & 0.284\,\cmark & 0.208\,\xmark & 0.263\,\cmark & 0.269\,\cmark & 0.330\,\cmark & \textbf{0.257}\,\xmark & 0.176\,\xmark & 0.149\,\cmark \\
& SIREN
& \textbf{0.389}\,\cmark & \textbf{0.270}\,\cmark & 0.347\,\cmark & \textbf{0.251}\,\cmark & \textbf{0.194}\,\cmark & \textbf{0.254}\,\cmark & \textbf{0.261}\,\cmark & \textbf{0.296}\,\cmark & \textbf{0.253}\,\cmark & \textbf{0.168}\,\cmark & \textbf{0.143}\,\cmark \\

\bottomrule
\end{tabular}}
\end{table*}
The comparison 
in Table~\ref{tab:prompteval_math_dspy_all_budgets} shows that PromptEval faces a bias trade-off that is not resolved by choosing a different cell-budget fraction. The sign of PromptEval's error depends on the observation fraction: sparse PromptEval is too pessimistic, while dense PromptEval approaches the optimistic same-data winner.  SIREN does not interpolate between these two regimes. It evaluates the post-search reporting rule on held-out split evaluations and uses item-level resampling to account for the first-order effect of selection. Across the four budgets in Table~\ref{tab:prompteval_per_budget_math_dspy}, SIREN's signed error stays between $-0.09$ and $-0.03$ percentage points, while every PromptEval fraction has substantially larger absolute bias. Full  PromptEval tables for the other subject-tuner cells appear in Appendix~\ref{sec:additional_experiments}.

\begin{table}[htbp]
\centering
\caption{PromptEval on \textbf{MMLU-Pro Math} for DSPy.}
\label{tab:prompteval_per_budget_math_dspy}
\scriptsize
\setlength{\tabcolsep}{2.2pt}
\renewcommand{\arraystretch}{0.82}
\resizebox{0.75\linewidth}{!}{%
\begin{tabular}{@{}c|cc|cc|cc|cc@{}}
\toprule
& \multicolumn{2}{c|}{500K} 
& \multicolumn{2}{c|}{1.5M} 
& \multicolumn{2}{c|}{3M} 
& \multicolumn{2}{c}{6.5M} \\
$f$ 
& dir & bias 
& dir & bias 
& dir & bias 
& dir & bias \\
\midrule
$0.05$ & $3/11$ & $-8.37$ & $3/11$ & $-8.75$ & $3/11$ & $-8.15$ & $3/11$ & $-8.29$ \\
$0.10$ & $3/11$ & $-5.20$ & $4/11$ & $-4.69$ & $5/11$ & $-4.59$ & $5/11$ & $-4.76$ \\
$0.20$ & $9/11$ & $+1.49$ & $8/11$ & $+1.70$ & $8/11$ & $+1.93$ & $8/11$ & $+1.71$ \\
$0.40$ & $9/11$ & $+1.56$ & $8/11$ & $+1.91$ & $8/11$ & $+1.86$ & $8/11$ & $+1.95$ \\
$0.60$ & $9/11$ & $+1.62$ & $8/11$ & $+1.60$ & $8/11$ & $+1.48$ & $8/11$ & $+1.52$ \\
$0.80$ & $9/11$ & $+1.55$ & $8/11$ & $+1.53$ & $8/11$ & $+1.68$ & $8/11$ & $+1.49$ \\
$1.00$ & $9/11$ & $+1.37$ & $8/11$ & $+1.52$ & $8/11$ & $+1.49$ & $8/11$ & $+1.33$ \\
\midrule
SIREN & $\mathbf{10/11}$ & $\mathbf{-0.09}$ & $\mathbf{11/11}$ & $\mathbf{-0.03}$ & $\mathbf{11/11}$ & $\mathbf{-0.05}$ & $\mathbf{11/11}$ & $\mathbf{-0.03}$ \\
\bottomrule
\end{tabular}}
\end{table}

Therefore, budget-aware prompt scoring is not the same as selection-aware inference. PromptEval is designed for efficient evaluation over a fixed prompt pool. In our post-tuning setting, however, the shortlist has already been produced by an adaptive tuner, and the target is the fresh-data performance of a specified reporting rule rather than the imputed performance distribution over prompts. In the results above, sparse PromptEval can have high directional accuracy because shrinkage moves estimates to the same side of $\theta_B^\star$ as the truth, rather than because it estimates $\theta_A^\star$ accurately. Dense PromptEval approaches the full-matrix selected winner, whose difference from the soft-rule reference can also reflect the deployment rule. SIREN targets the reporting quantity directly by evaluating the selected repeated-split procedure and accounting for the first-order effect of selection.

\subsection{Additional empirical checks}
\label{sec:added_checks}

The fresh-data analysis of $176$ Math/Law cells gives M1/M2 and SIREN mean absolute errors of $0.22$ and $0.24$\,pp against their own targets, with $165/176$ correct directions for both (Appendix~\ref{app:matched_targets}). This separates estimator accuracy from reporting-rule choice.

On Law, SIREN differs from matched outer-test evaluation of its soft rule by $+0.66$\,pp (Appendix~\ref{app:nested_cv}). Fresh tuner reruns separate reporting error, tuner variation, and soft-versus-hard deployment differences, without establishing target equivalence (Appendix~\ref{app:tuner_reruns}). Appendix~\ref{app:wald_checks} separates interval centering from the selection contribution to variance.

Under stochastic decoding, bootstrap coverage is $95.05\%$ for both single-decode and fractional scores (Appendix~\ref{app:stochastic_decoding}). Philosophy and Economics extend the domain comparisons; HumanEval+ Mini adds execution-scored programs (Appendices~\ref{app:additional_domains} and~\ref{app:code_generation}). SIREN remains close to its specified reporting reference.

\vspace{-2mm}
\section{Conclusion}
\vspace{-2mm}
Adaptive prompt and program search changes what an LLM benchmark score estimates. Once benchmark items are reused for tuning, the reported winner is a data-dependent artifact, not a fixed system evaluated on fresh data. SIREN addresses this tune-then-deploy setting by freezing the post-search shortlist, separating splitwise selection from held-out evaluation, and bootstrapping item-level contributions without rerunning the tuner. In the fixed-shortlist smooth-selector regime, this supports valid simultaneous inference over the budget grid. Experiments show that same-data winner reports can be optimistic and can change tuned-versus-default conclusions. The additional studies distinguish reporting-rule differences from estimation error and extend the evidence to stochastic decoding and execution-scored code generation. Fully open-ended adaptive search remains future work.

\clearpage
\bibliographystyle{plain}
\bibliography{main}

\clearpage
\input{checklist}

\clearpage
\appendix
\onecolumn

\section{Related works}
\label{app:related_work}

A first line of related work develops methods that use data to search over prompts, demonstrations, and larger language-model programs. Early automatic prompting methods such as AutoPrompt optimize discrete trigger tokens for masked language models, while Automatic Prompt Engineer treats instructions as candidate programs generated and selected by an LLM \cite{shin2020autoprompt,zhou2023largepromptengineers}. More recent prompt optimizers make the search loop itself increasingly adaptive: OPRO uses an LLM as a derivative-free optimizer that proposes new instructions from previous scores, and EvoPrompt connects LLM prompt generation with evolutionary search \cite{yang2024large,guo2024evoprompt}. At the program level, DSPy abstracts LM pipelines into parameterized modules and compiles them against a task metric, while MIPRO extends this idea by jointly optimizing instructions and few-shot demonstrations for multi-stage LM programs \cite{khattab2023dspy,opsahlong2024optimizing}. Other work makes the evaluation budget explicit: TRIPLE formulates prompt selection as fixed-budget best-arm identification, and OPTS uses bandit-based strategy selection inside prompt optimization \cite{shi2024triple,ashizawa2025opts}. These methods are natural upstream tuners for our setting, including the random-search and DSPy tuners used in our experiments. Our contribution is complementary: rather than proposing another optimizer, we study conditional post-search reporting after a budgeted tuner has frozen a shortlist.

A second line of work asks how to make model evaluation more stable, cheaper, or more statistically meaningful. Benchmarks such as MMLU-Pro increase task difficulty and robustness by emphasizing reasoning-focused questions and larger answer sets, but even fixed-benchmark reports can be fragile because prompt templates affect both absolute scores and model rankings \cite{wang2024mmlupro,mizrahi2024state}. PromptEval addresses this prompt sensitivity by estimating a performance distribution over a prespecified prompt pool with an IRT-style model under limited cell budgets, making it the closest external baseline in our experiments \cite{polo2024efficient}. ReliableEval instead treats evaluation as stochastic over meaning-preserving prompt perturbations and uses method-of-moments reasoning to determine how many perturbations are needed, while Prompt Risk Control selects prompts using upper bounds on deployment-relevant risk measures \cite{lior2025reliableeval,zollo2024prompt}. A parallel body of efficient-evaluation work reduces benchmark cost through curated subsets, anchor points, matrix factorization, collaborative filtering, stratified or model-assisted sampling, and active querying with valid confidence intervals \cite{polo2024tinybenchmarks,vivek2024anchor,zhou2025speeding,zhong2025efficient,fogliato2024framework,zrnic2024active,he2026suppressing,gligoric2025can,angelopoulos2025cost,zhang2025benchmark,wu2026efficient}. Our internal baselines M1--M4 instantiate common reporting choices in this space, same-data winner reporting, a Wald interval around the winner, a single holdout split, and repeated split-level argmax with a Student-$t$ interval, and illustrate that budget awareness or data splitting alone is not the same as selection-aware inference. Our paper targets a different estimand from these works: the post-search, procedure-level performance of a frozen shortlist plus a specified reporting rule, together with uncertainty that accounts for both held-out evaluation noise and the first-order effect of selection.

A third line of work studies how repeated data use and data-dependent selection invalidate naive uncertainty estimates. The reusable holdout and related adaptive data-analysis theory show that repeatedly querying the same validation data can destroy nominal guarantees unless the reporting protocol controls information leakage; leaderboard-oriented methods such as the Ladder pursue similar goals for machine-learning competitions \cite{dwork2015reusable,dwork2015preserving,blum2015ladder}. Post-selection inference studies valid inference after a model, coefficient, or hypothesis has been chosen from the data, using tools such as simultaneous inference, selective error control, exact conditioning for the lasso, and randomized selection \cite{berk2013valid,fithian2014optimal,lee2016exact,tian2018selective}. The winner's curse literature makes the bias mechanism explicit: the empirically best option is likely to have benefited from positive noise, so conventional estimates and confidence intervals for selected winners can be optimistic \cite{efron2011tweedie,andrews2024inference}. Finally, bootstrap and resampling methods provide the computational machinery for quantifying uncertainty in complex estimators, including classical bootstrap corrections and Gaussian multiplier bootstraps for maxima over many coordinates \cite{efron1997improvements,chernozhukov2013gaussian}. Our work brings these ideas into adaptive LLM benchmarking by defining the selected tune-then-deploy procedure as the inferential target and by deriving an item-level repeated-split linearization whose multiplier bootstrap simultaneously covers systems, budgets, and fixed contrasts without rerunning the tuner.

\section{Proofs for Section~\ref{sec:siren}}
\label{app:siren_proofs}

\subsection{Formal setup and assumptions}
\label{app:siren_formal_setup}

This appendix proves the main-text guarantees in four steps. Lemma~\ref{lem:siren_input_representation} and Proposition~\ref{prop:input_clt} establish the score-table representation and conditional CLT. Lemmas~\ref{lem:siren_selector_linearization}--\ref{lem:siren_aggregate} linearize the selector and aggregate the split-level expansions, leading to Theorem~\ref{thm:siren_linearization_formal}, the formal version of Theorem~\ref{thm:positive_smooth}. Theorem~\ref{thm:siren_bootstrap_formal} proves Theorem~\ref{thm:bootstrap_validity}, and Proposition~\ref{prop:siren_plugin} gives a sufficient condition for consistency of the computable contributions. 

Let $\mathcal F^{\mathrm{tune}}$ denote the information generated by the completed tuning stage, including the retained shortlists and selectors. 
All arguments are conditional on $\mathcal F^{\mathrm{tune}}$; consequently, the retained shortlists and selectors are treated as fixed reporting-layer objects, and this conditioning is suppressed below. Asymptotics are taken as $M\to\infty$ with the number of systems, budgets, repeated splits, and retained artifacts fixed.

For each split, system, and budget, define the population score vectors
\begin{equation}
s_{r,j,M}(B)
=
\mathbb E\big[\widehat S_{r,j}(B)\mid \mathcal G_M\big],
\qquad
t_{r,j,M}(B)
=
\mathbb E\big[\widehat T_{r,j}(B)\mid \mathcal G_M\big],
\qquad
q_{r,j,M}(B)
=
g_{j,B}\big(s_{r,j,M}(B)\big).
\label{eq:app_population_score_vectors}
\end{equation}
For a differentiable selector $g_{j,B}$, let $Dg_{j,B}(x)$ denote its Jacobian at $x$. 
We also define the centered repeated-split score table
\begin{equation}
\mathbb U_M
=
\Big\{
\widehat S_{r,j}(B)-s_{r,j,M}(B),\
\widehat T_{r,j}(B)-t_{r,j,M}(B)
\Big\}_{r\le R,\ j\le N,\ B\in\mathcal B}.
\label{eq:app_siren_U}
\end{equation}
Finally, define the joint scaled error vector over the finite system-budget grid as
\begin{equation}
\mathbf Z_M
=
\Big\{
\sqrt M
\big(
\widetilde\theta_{j,R}(B)
-
\theta_{j,M}^{\mathrm{RS}\mid\mathcal G}(B)
\big)
\Big\}_{j\le N,\ B\in\mathcal B}.
\label{eq:app_siren_ZM}
\end{equation}

\begin{assumption}[Fixed-shortlist smooth-selector regime]
\label{ass:siren_regime}
Fix the number of systems $N$, budget grid $\mathcal B$, and number of repeated splits $R$, and assume $\sup_{j\le N,\ B\in\mathcal B}K_{j,B}<\infty$. 
The split design $\mathcal G_M$ is independent of the benchmark items and their execution randomness. 
There exist deterministic integers $m_M,\ell_M$ and constants $\rho_{\mathrm{score}},\rho_{\mathrm{eval}}\in(0,1)$ such that
$|D_r^{\mathrm{score}}|=m_M$, $|E_r|=\ell_M$, $m_M/M\to\rho_{\mathrm{score}}$, and $\ell_M/M\to\rho_{\mathrm{eval}}$ for every $r\le R$. 
All score variables take values in $[0,1]$.

For each item $i$, let $\mathcal R_i$ collect $W_i$ together with all execution randomness used to score item $i$ across splits, systems, budgets, and retained artifacts. 
The tuples $\mathcal R_1,\dots,\mathcal R_M$ are i.i.d. and independent of $\mathcal G_M$. 
No independence is required among the coordinates within a fixed $\mathcal R_i$.

There exist constants $\delta_0>0$ and $C_g<\infty$ such that, for all sufficiently large $M$ and all $r,j,B$, the selector $g_{j,B}$ is twice continuously differentiable on
$\{x:\|x-s_{r,j,M}(B)\|\le\delta_0\}$, with first and second derivatives uniformly bounded by $C_g$. 
The weights $\omega_r$ are $\mathcal G_M$-measurable, nonnegative, and sum to one.
\end{assumption}

\subsection{Coordinate formulas and plug-in contributions}
\label{app:siren_coordinates}

The main text uses item-level contributions only conceptually. 
We now give the exact coordinate formulas used in the formal statements and proofs. 
Let
\begin{equation}
d_U
=
2R\sum_{j\le N,\ B\in\mathcal B}K_{j,B}
\label{eq:app_du}
\end{equation}
denote the dimension of the stacked score table. For the score vectors in \eqref{eq:siren_dev_score} and \eqref{eq:siren_hold_score}, define the coordinatewise item-level contributions
\begin{equation}
\Gamma^{S}_{i,r,j,k,M}(B)
=
\frac{M}{|D_r^{\mathrm{score}}|}
\mathbf 1\{i\in D_r^{\mathrm{score}}\}
\Big(
Z_j\big(W_i,a_{j,B,k},\xi_{irjBk}^{S}\big)-s_{r,j,k,M}(B)
\Big),
\label{eq:siren_gamma_s_coord}
\end{equation}
\begin{equation}
\Gamma^{T}_{i,r,j,k,M}(B)
=
\frac{M}{|E_r|}
\mathbf 1\{i\in E_r\}
\Big(
Z_j\big(W_i,a_{j,B,k},\xi_{irjBk}^{T}\big)-t_{r,j,k,M}(B)
\Big),
\label{eq:siren_gamma_t_coord}
\end{equation}
and their block vectors
\begin{equation}
\Gamma^{S}_{i,r,j,M}(B)
=
\big(
\Gamma^{S}_{i,r,j,1,M}(B),\dots,
\Gamma^{S}_{i,r,j,K_{j,B},M}(B)
\big)^\top,
\label{eq:siren_gamma_s_block}
\end{equation}
\begin{equation}
\Gamma^{T}_{i,r,j,M}(B)
=
\big(
\Gamma^{T}_{i,r,j,1,M}(B),\dots,
\Gamma^{T}_{i,r,j,K_{j,B},M}(B)
\big)^\top .
\label{eq:siren_gamma_t_block}
\end{equation}
The stacked vector $\sqrt M\mathbb U_M
=
\frac{1}{\sqrt M}\sum_{i=1}^M \Gamma_{i,M}$ is obtained by stacking the blocks in \eqref{eq:siren_gamma_s_block} and \eqref{eq:siren_gamma_t_block} over the finite index set $\{(r,j,B):r\le R,\ j\le N,\ B\in\mathcal B\}$ under any fixed deterministic ordering. The coordinate definitions imply the score-table identities
\begin{equation}
\sqrt M\big(\widehat S_{r,j}(B)-s_{r,j,M}(B)\big)
=
\frac{1}{\sqrt M}\sum_{i=1}^M\Gamma^{S}_{i,r,j,M}(B),
\label{eq:siren_score_expansion_s}
\end{equation}
\begin{equation}
\sqrt M\big(\widehat T_{r,j}(B)-t_{r,j,M}(B)\big)
=
\frac{1}{\sqrt M}\sum_{i=1}^M\Gamma^{T}_{i,r,j,M}(B).
\label{eq:siren_score_expansion_t}
\end{equation}

For the selected repeated-split estimator, define the population item-level contribution
\begin{equation}
\psi_{i,j,B,M}
=
\sum_{r=1}^R \omega_r
\Big[
q_{r,j,M}(B)^\top \Gamma^{T}_{i,r,j,M}(B)
+
t_{r,j,M}(B)^\top
Dg_{j,B}\big(s_{r,j,M}(B)\big)
\Gamma^{S}_{i,r,j,M}(B)
\Big].
\label{eq:siren_psi}
\end{equation}
The first term is the direct held-out evaluation contribution of item $W_i$. 
The second term is the first-order selection contribution: when $W_i$ appears in a scoring subset, it perturbs the selection scores and hence the splitwise weights through the derivative of $g_{j,B}$.

Stack these contributions over the finite system-budget grid as
\begin{equation}
\boldsymbol\Psi_{i,M}
=
\{\psi_{i,j,B,M}\}_{j\le N,\ B\in\mathcal B},
\qquad
\Omega_M
=
\frac{1}{M}
\sum_{i=1}^M
\mathbb E\big[
\boldsymbol\Psi_{i,M}\boldsymbol\Psi_{i,M}^{\top}
\mid \mathcal G_M
\big].
\label{eq:app_Psi_Omega}
\end{equation}
For the computable bootstrap contributions, define the sample-centered coordinatewise analogues
\begin{equation}
\widehat\Gamma^{S}_{i,r,j,k,M}(B)
=
\frac{M}{|D_r^{\mathrm{score}}|}
\mathbf 1\{i\in D_r^{\mathrm{score}}\}
\Big(
Z_j\big(W_i,a_{j,B,k},\xi_{irjBk}^{S}\big)-\widehat S_{r,j,k}(B)
\Big),
\label{eq:siren_hat_gamma_s_coord}
\end{equation}
\begin{equation}
\widehat\Gamma^{T}_{i,r,j,k,M}(B)
=
\frac{M}{|E_r|}
\mathbf 1\{i\in E_r\}
\Big(
Z_j\big(W_i,a_{j,B,k},\xi_{irjBk}^{T}\big)-\widehat T_{r,j,k}(B)
\Big),
\label{eq:siren_hat_gamma_t_coord}
\end{equation}
and the corresponding block vectors
\begin{equation}
\widehat\Gamma^{S}_{i,r,j,M}(B)
=
\big(
\widehat\Gamma^{S}_{i,r,j,1,M}(B),\dots,
\widehat\Gamma^{S}_{i,r,j,K_{j,B},M}(B)
\big)^\top,
\label{eq:siren_hat_gamma_s_block}
\end{equation}
\begin{equation}
\widehat\Gamma^{T}_{i,r,j,M}(B)
=
\big(
\widehat\Gamma^{T}_{i,r,j,1,M}(B),\dots,
\widehat\Gamma^{T}_{i,r,j,K_{j,B},M}(B)
\big)^\top .
\label{eq:siren_hat_gamma_t_block}
\end{equation}
The plug-in contribution estimator used in the bootstrap subsection is
\begin{equation}
\widehat\psi_{i,j,B,M}
=
\sum_{r=1}^R\omega_r
\Big[
\widehat q_{r,j}(B)^\top\widehat\Gamma^{T}_{i,r,j,M}(B)
+
\widehat T_{r,j}(B)^\top Dg_{j,B}\big(\widehat S_{r,j}(B)\big)
\widehat\Gamma^{S}_{i,r,j,M}(B)
\Big].
\label{eq:siren_hat_psi}
\end{equation}
Here the first term is the sample analogue of the direct held-out contribution, and the second term is the sample analogue of the first-order selection contribution. We also write
\begin{equation}
\widehat{\boldsymbol\Psi}_{i,M}
=
\{\widehat\psi_{i,j,B,M}\}_{j\le N,\ B\in\mathcal B},
\qquad
\mathcal H_M
=
\sigma\big(
\mathcal G_M,
\widehat{\boldsymbol\Psi}_{1,M},\dots,\widehat{\boldsymbol\Psi}_{M,M}
\big).
\label{eq:app_hatPsi_H}
\end{equation}

\subsection{Conditional convergence convention and score-table CLT}
\label{app:siren_score_table_clt}

For completeness, we state the conditional convergence convention used in the formal results below. We write $X_M\rightsquigarrow X$ conditionally on $\mathcal G_M$ in probability if
\begin{equation}
\sup_{h\in BL_1}
\left|
\mathbb E\big[h(X_M)\mid\mathcal G_M\big]
-
\mathbb E\big[h(X)\big]
\right|
\to_p 0,
\label{eq:app_conditional_weak}
\end{equation}
where $BL_1$ is the class of real-valued functions bounded by one and Lipschitz with constant at most one.

\begin{lemma}[Exact score-table representation]
\label{lem:siren_input_representation}
Under Assumption~\ref{ass:siren_regime},
\begin{equation}
\sqrt M\,\mathbb U_M
=
\frac{1}{\sqrt M}\sum_{i=1}^M\Gamma_{i,M},
\label{eq:lem_input_repr_main}
\end{equation}
where $\Gamma_{i,M}$ is the stacked vector obtained from \eqref{eq:siren_gamma_s_coord} and \eqref{eq:siren_gamma_t_coord}. Conditional on $\mathcal G_M$, the vectors $\Gamma_{1,M},\dots,\Gamma_{M,M}$ are independent, mean zero, and satisfy
\begin{equation}
\max_{1\le i\le M}
\mathbb E\big[\|\Gamma_{i,M}\|^2\mid\mathcal G_M\big]
\le C
\label{eq:lem_input_repr_moment}
\end{equation}
for a deterministic constant $C<\infty$ and all sufficiently large $M$. Consequently,
\begin{equation}
\max_{r,j,B,k}
M\,\mathbb E\Big[
\big(\widehat S_{r,j,k}(B)-s_{r,j,k,M}(B)\big)^2
\mid\mathcal G_M
\Big]
\le C
\label{eq:lem_input_repr_S2}
\end{equation}
and
\begin{equation}
\max_{r,j,B,k}
M\,\mathbb E\Big[
\big(\widehat T_{r,j,k}(B)-t_{r,j,k,M}(B)\big)^2
\mid\mathcal G_M
\Big]
\le C.
\label{eq:lem_input_repr_T2}
\end{equation}
\end{lemma}

\begin{proof}
For a scoring coordinate,
\begin{equation}
\widehat S_{r,j,k}(B)
=
\frac{1}{|D_r^{\mathrm{score}}|}
\sum_{i=1}^M
\mathbf 1\{i\in D_r^{\mathrm{score}}\}
Z_j\big(W_i,a_{j,B,k},\xi_{irjBk}^{S}\big).
\label{eq:proof_input_repr_S_def}
\end{equation}
Subtracting $s_{r,j,k,M}(B)$ and multiplying by $\sqrt M$ gives
\begin{align}
\sqrt M\big(\widehat S_{r,j,k}(B)-s_{r,j,k,M}(B)\big)
&=
\frac{1}{\sqrt M}\sum_{i=1}^M
\frac{M}{|D_r^{\mathrm{score}}|}
\mathbf 1\{i\in D_r^{\mathrm{score}}\}
\notag\\
&\qquad\times
\Big(
Z_j\big(W_i,a_{j,B,k},\xi_{irjBk}^{S}\big)-s_{r,j,k,M}(B)
\Big),
\label{eq:proof_input_repr_S_expansion}
\end{align}
which is the corresponding coordinate of $M^{-1/2}\sum_{i=1}^M\Gamma_{i,M}$. The same calculation with \eqref{eq:siren_hold_score} gives the held-out coordinates. Stacking all coordinates proves \eqref{eq:lem_input_repr_main}.

Conditional on $\mathcal G_M$, the split membership indicators and denominators are deterministic. The vector $\Gamma_{i,M}$ depends only on $W_i$ and the execution randomness attached to item $i$. By Assumption~\ref{ass:siren_regime}, these item-level tuples are independent across $i$, and therefore $\Gamma_{1,M},\dots,\Gamma_{M,M}$ are conditionally independent.

For mean zero, consider a scoring coordinate. If $i\notin D_r^{\mathrm{score}}$, then the corresponding contribution is zero. If $i\in D_r^{\mathrm{score}}$, then the included summands in \eqref{eq:proof_input_repr_S_def} are conditionally identically distributed given $\mathcal G_M$, so
\begin{equation}
s_{r,j,k,M}(B)
=
\mathbb E\Big[
Z_j\big(W_i,a_{j,B,k},\xi_{irjBk}^{S}\big)
\mid\mathcal G_M
\Big].
\label{eq:proof_input_repr_s_id}
\end{equation}
The same argument applies to held-out coordinates. Hence
\begin{equation}
\mathbb E\big[\Gamma_{i,M}\mid\mathcal G_M\big]=0.
\label{eq:proof_input_repr_meanzero}
\end{equation}

The dimension $d_U$ is fixed by Assumption~\ref{ass:siren_regime}. Also, the split fractions stay away from zero, so there exists $L<\infty$ such that
\begin{equation}
\max_{r\le R}\frac{M}{|D_r^{\mathrm{score}}|}\le L,
\qquad
\max_{r\le R}\frac{M}{|E_r|}\le L
\label{eq:proof_input_repr_L}
\end{equation}
for all sufficiently large $M$. Since scores lie in $[0,1]$, every coordinate of $\Gamma_{i,M}$ is bounded by $L$ in absolute value. Thus $\|\Gamma_{i,M}\|^2\le C$ for a deterministic constant $C<\infty$, proving \eqref{eq:lem_input_repr_moment}.

Finally, by \eqref{eq:proof_input_repr_S_expansion}, conditional independence, and mean zero,
\begin{align}
M\,\mathbb E\Big[
\big(\widehat S_{r,j,k}(B)-s_{r,j,k,M}(B)\big)^2
\mid\mathcal G_M
\Big]
&=
\frac{1}{M}\sum_{i=1}^M
\mathbb E\Big[
\big(\Gamma^{S}_{i,r,j,k,M}(B)\big)^2
\mid\mathcal G_M
\Big]
\notag\\
&\le C.
\label{eq:proof_input_repr_S2}
\end{align}
The proof of \eqref{eq:lem_input_repr_T2} is identical.
\end{proof}

\begin{assumption}[Score-table covariance stabilization]
\label{ass:siren_score_table_covariance}
Let
\begin{equation}
\Sigma_M
=
\frac{1}{M}
\sum_{i=1}^M
\mathbb E\big[\Gamma_{i,M}\Gamma_{i,M}^\top\mid\mathcal G_M\big].
\label{eq:siren_sigma_M}
\end{equation}
There exists a deterministic positive semidefinite matrix $\Sigma$ such that
\begin{equation}
\|\Sigma_M-\Sigma\|_{\mathrm{op}}\to_p 0.
\label{eq:siren_sigma_conv}
\end{equation}
\end{assumption}

\begin{proposition}[Score-table CLT]
\label{prop:input_clt}
Under Assumptions~\ref{ass:siren_regime} and \ref{ass:siren_score_table_covariance},
\begin{equation}
\sqrt M\,\mathbb U_M
\rightsquigarrow
N(0,\Sigma)
\qquad
\text{conditionally on }\mathcal G_M\text{ in probability}.
\label{eq:prop_input_clt}
\end{equation}
In particular, $\sqrt M\,\mathbb U_M=O_p(1)$.
\end{proposition}

\begin{proof}
By Lemma~\ref{lem:siren_input_representation},
\begin{equation}
\sqrt M\,\mathbb U_M
=
\frac{1}{\sqrt M}\sum_{i=1}^M\Gamma_{i,M},
\label{eq:proof_prop_input_repr}
\end{equation}
where, conditional on $\mathcal G_M$, the summands are independent, mean zero, and uniformly bounded.

Fix a deterministic vector $u\in\mathbb R^{d_U}$ and define
\begin{equation}
Y_{i,M}(u)=\frac{1}{\sqrt M}u^\top\Gamma_{i,M}.
\label{eq:proof_prop_Y}
\end{equation}
Then
\begin{equation}
\sum_{i=1}^M\mathbb E\big[Y_{i,M}(u)^2\mid\mathcal G_M\big]
=
u^\top\Sigma_M u.
\label{eq:proof_prop_varsum}
\end{equation}
By Assumption~\ref{ass:siren_score_table_covariance}, this variance converges in probability to $u^\top\Sigma u$.

The Lindeberg condition is immediate. Lemma~\ref{lem:siren_input_representation} and bounded scores imply that there exists $C_0<\infty$ such that $\max_i\|\Gamma_{i,M}\|\le C_0$ for all sufficiently large $M$. Hence
\begin{equation}
\max_{1\le i\le M}|Y_{i,M}(u)|
\le
\frac{\|u\|C_0}{\sqrt M}
\to0,
\label{eq:proof_prop_ymax}
\end{equation}
which implies the conditional Lindeberg condition for every fixed $u$.

Along any subsequence there is a further subsequence on which $\|\Sigma_M-\Sigma\|_{\mathrm{op}}\to0$ almost surely. On that further subsequence, the conditional Lindeberg--Feller theorem gives
\begin{equation}
u^\top\sqrt M\,\mathbb U_M
\rightsquigarrow
N(0,u^\top\Sigma u)
\qquad
\text{conditionally on }\mathcal G_M
\label{eq:proof_prop_scalar}
\end{equation}
almost surely. Since the dimension is fixed, Cram\'er--Wold gives \eqref{eq:prop_input_clt}. The bound $\sqrt M\,\mathbb U_M=O_p(1)$ follows from the same display.
\end{proof}

\subsection{Selection-aware linearization}
\label{app:siren_linearization}

\begin{lemma}[Selector linearization]
\label{lem:siren_selector_linearization}
Under Assumption~\ref{ass:siren_regime},
\begin{equation}
\max_{r\le R,\ j\le N,\ B\in\mathcal B}
\left\|
\widehat q_{r,j}(B)
-
q_{r,j,M}(B)
-
Dg_{j,B}\big(s_{r,j,M}(B)\big)
\big(
\widehat S_{r,j}(B)-s_{r,j,M}(B)
\big)
\right\|
=
o_p(M^{-1/2}).
\label{eq:lem_selector_linearization}
\end{equation}
\end{lemma}

\begin{proof}
By Lemma~\ref{lem:siren_input_representation}, each coordinate of $\widehat S_{r,j}(B)-s_{r,j,M}(B)$ has conditional second moment bounded by $C/M$. Since the index set is finite,
\begin{equation}
\max_{r,j,B}
\big\|
\widehat S_{r,j}(B)-s_{r,j,M}(B)
\big\|
=
O_p(M^{-1/2}).
\label{eq:proof_selector_Sh}
\end{equation}
Fix $r,j,B$ and write $\widehat S=\widehat S_{r,j}(B)$, $s=s_{r,j,M}(B)$, $g=g_{j,B}$, and $h=\widehat S-s$. Assumption~\ref{ass:siren_regime} gives a neighborhood of $s$ on which $g$ is twice continuously differentiable with uniformly bounded second derivatives. On the event $\{\|h\|\le\delta_0\}$, Taylor's theorem yields
\begin{equation}
g(s+h)
=
g(s)+Dg(s)h+R(h),
\qquad
\|R(h)\|\le C\|h\|^2,
\label{eq:proof_selector_taylor}
\end{equation}
with $C$ uniform over the finite index set. By \eqref{eq:proof_selector_Sh}, the event on which all relevant $\|h\|$ are at most $\delta_0$ has probability tending to one. Hence
\begin{align}
&\max_{r,j,B}
\left\|
\widehat q_{r,j}(B)
-
q_{r,j,M}(B)
-
Dg_{j,B}\big(s_{r,j,M}(B)\big)
\big(
\widehat S_{r,j}(B)-s_{r,j,M}(B)
\big)
\right\|
\notag\\
&\qquad\le
C
\left(
\max_{r,j,B}
\big\|
\widehat S_{r,j}(B)-s_{r,j,M}(B)
\big\|
\right)^2
=
O_p(M^{-1}),
\label{eq:proof_selector_bound}
\end{align}
which proves the claim.
\end{proof}

\begin{lemma}[Selected split-level expansion]
\label{lem:siren_split_expansion}
Under Assumption~\ref{ass:siren_regime}, jointly over $r\le R$, $j\le N$, and $B\in\mathcal B$,
\begin{align}
\widehat Y_{r,j}(B)-\mathbb E\big[\widehat Y_{r,j}(B)\mid\mathcal G_M\big]
&=
q_{r,j,M}(B)^\top
\big(
\widehat T_{r,j}(B)-t_{r,j,M}(B)
\big)
\notag\\
&\quad+
t_{r,j,M}(B)^\top
Dg_{j,B}\big(s_{r,j,M}(B)\big)
\big(
\widehat S_{r,j}(B)-s_{r,j,M}(B)
\big)
\notag\\
&\quad+
o_p(M^{-1/2}).
\label{eq:lem_split_expansion}
\end{align}
\end{lemma}

\begin{proof}
Fix $r,j,B$ and abbreviate
\small
\begin{equation}
\widehat S=\widehat S_{r,j}(B),\;
s=s_{r,j,M}(B),\;
\widehat T=\widehat T_{r,j}(B),\;
t=t_{r,j,M}(B),\;
A=Dg_{j,B}(s),\;
\widehat q=\widehat q_{r,j}(B),\;
q=q_{r,j,M}(B).
\label{eq:proof_split_notation}
\end{equation}
\normalsize
By Lemma~\ref{lem:siren_selector_linearization},
\begin{equation}
\widehat q=q+A(\widehat S-s)+R_q,
\qquad
\|R_q\|=o_p(M^{-1/2}).
\label{eq:proof_split_qexpansion}
\end{equation}
The proof of Lemma~\ref{lem:siren_selector_linearization} also gives the sharper bound $\|R_q\|=O_p(M^{-1})$. Moreover, since $g$ maps into a simplex, the scores are bounded, and $A$ is uniformly bounded, the Taylor remainder is uniformly bounded off the local event. Combining this boundedness with the conditional second-moment bound for $\widehat S-s$ gives
\begin{equation}
\mathbb E\big[\|R_q\|\mid\mathcal G_M\big]=O_p(M^{-1}),
\qquad
\mathbb E\big[\|R_q\|^2\mid\mathcal G_M\big]=O_p(M^{-1}).
\label{eq:proof_split_rq_moments}
\end{equation}

Expanding $\widehat Y=\widehat q^\top\widehat T$ gives
\begin{align}
\widehat Y
&=
\big(q+A(\widehat S-s)+R_q\big)^\top
\big(t+(\widehat T-t)\big)
\notag\\
&=
q^\top t
+
q^\top(\widehat T-t)
+
t^\top A(\widehat S-s)
+
\mathcal R,
\label{eq:proof_split_expand}
\end{align}
where
\begin{equation}
\mathcal R
=
t^\top R_q
+
\big(A(\widehat S-s)\big)^\top(\widehat T-t)
+
R_q^\top(\widehat T-t).
\label{eq:proof_split_R}
\end{equation}
Since $q,t,A$ are $\mathcal G_M$-measurable and $s,t$ are the conditional means,
\begin{equation}
\mathbb E\big[q^\top(\widehat T-t)\mid\mathcal G_M\big]=0,
\qquad
\mathbb E\big[t^\top A(\widehat S-s)\mid\mathcal G_M\big]=0.
\label{eq:proof_split_center_lead}
\end{equation}
Therefore
\begin{equation}
\widehat Y-
\mathbb E\big[\widehat Y\mid\mathcal G_M\big]
=
q^\top(\widehat T-t)
+
t^\top A(\widehat S-s)
+
\mathcal R-
\mathbb E\big[\mathcal R\mid\mathcal G_M\big].
\label{eq:proof_split_centered}
\end{equation}

It remains to show that the centered remainder is $o_p(M^{-1/2})$. Lemma~\ref{lem:siren_input_representation} implies
\begin{equation}
\|\widehat S-s\|=O_p(M^{-1/2}),
\qquad
\|\widehat T-t\|=O_p(M^{-1/2}),
\label{eq:proof_split_ST_rate}
\end{equation}
jointly over the finite index set, and also gives conditional second moments of order $M^{-1}$. The first remainder term is $O_p(M^{-1})$ with conditional expectation $O_p(M^{-1})$ by \eqref{eq:proof_split_rq_moments}. For the second term, Cauchy--Schwarz and the conditional second-moment bounds give
\begin{equation}
\mathbb E\Big[
\big|
\big(A(\widehat S-s)\big)^\top(\widehat T-t)
\big|
\mid\mathcal G_M
\Big]
=
O_p(M^{-1}),
\label{eq:proof_split_R2}
\end{equation}
and the term itself is $O_p(M^{-1})$. The third term is handled similarly using \eqref{eq:proof_split_rq_moments} and the conditional second-moment bound for $\widehat T-t$. Thus
\begin{equation}
\mathcal R-
\mathbb E\big[\mathcal R\mid\mathcal G_M\big]
=
o_p(M^{-1/2}).
\label{eq:proof_split_Rsmall}
\end{equation}
Substituting \eqref{eq:proof_split_Rsmall} into \eqref{eq:proof_split_centered} proves the lemma.
\end{proof}

\begin{lemma}[Aggregation over repeated splits]
\label{lem:siren_aggregate}
Under Assumption~\ref{ass:siren_regime}, jointly over $j\le N$ and $B\in\mathcal B$,
\begin{equation}
\sqrt M
\Big(
\widetilde\theta_{j,R}(B)-\theta_{j,M}^{\mathrm{RS}\mid\mathcal G}(B)
\Big)
=
\frac{1}{\sqrt M}\sum_{i=1}^M\psi_{i,j,B,M}
+
o_p(1).
\label{eq:lem_aggregate_linear}
\end{equation}
If $ \Omega_M
=
\frac{1}{M}
\sum_{i=1}^M
\mathbb E\big[\boldsymbol\Psi_{i,M}\boldsymbol\Psi_{i,M}^\top\mid \mathcal G_M\big]$ satisfies $\|\Omega_M-\Omega\|_{\mathrm{op}}\to_p 0$, then
\begin{equation}
\mathbf Z_M
\rightsquigarrow
N(0,\Omega)
\qquad
\text{conditionally on }\mathcal G_M\text{ in probability}.
\label{eq:lem_aggregate_clt}
\end{equation}
\end{lemma}

\begin{proof}
By definition of the target and the $\mathcal G_M$-measurability of the weights,
\begin{equation}
\widetilde\theta_{j,R}(B)-\theta_{j,M}^{\mathrm{RS}\mid\mathcal G}(B)
=
\sum_{r=1}^R\omega_r
\Big(
\widehat Y_{r,j}(B)-
\mathbb E\big[\widehat Y_{r,j}(B)\mid\mathcal G_M\big]
\Big).
\label{eq:proof_aggregate_start}
\end{equation}
Multiplying by $\sqrt M$ and substituting Lemma~\ref{lem:siren_split_expansion} yields
\begin{align}
\sqrt M
\Big(
\widetilde\theta_{j,R}(B)-\theta_{j,M}^{\mathrm{RS}\mid\mathcal G}(B)
\Big)
&=
\sum_{r=1}^R\omega_r q_{r,j,M}(B)^\top
\sqrt M\big(\widehat T_{r,j}(B)-t_{r,j,M}(B)\big)
\notag\\
&\quad+
\sum_{r=1}^R\omega_r t_{r,j,M}(B)^\top Dg_{j,B}\big(s_{r,j,M}(B)\big)
\sqrt M\big(\widehat S_{r,j}(B)-s_{r,j,M}(B)\big)
\notag\\
&\quad+
o_p(1).
\label{eq:proof_aggregate_after_split}
\end{align}
Using \eqref{eq:siren_score_expansion_s} and \eqref{eq:siren_score_expansion_t}, the right-hand side becomes
\begin{align}
\frac{1}{\sqrt M}\sum_{i=1}^M
\sum_{r=1}^R\omega_r
\Big[
q_{r,j,M}(B)^\top\Gamma^{T}_{i,r,j,M}(B)
+
t_{r,j,M}(B)^\top Dg_{j,B}\big(s_{r,j,M}(B)\big)
\Gamma^{S}_{i,r,j,M}(B)
\Big]
+
o_p(1).
\label{eq:proof_aggregate_switch}
\end{align}
The inner sum is exactly $\psi_{i,j,B,M}$ in \eqref{eq:siren_psi}, proving \eqref{eq:lem_aggregate_linear}.

For the conditional Gaussian limit, let $d=N|\mathcal B|$ and fix $u\in\mathbb R^d$. Define
\begin{equation}
Z_{i,M}(u)=\frac{1}{\sqrt M}u^\top\boldsymbol\Psi_{i,M}.
\label{eq:proof_aggregate_Z}
\end{equation}
Conditional on $\mathcal G_M$, the variables $Z_{1,M}(u),\dots,Z_{M,M}(u)$ are independent and mean zero, because each $\boldsymbol\Psi_{i,M}$ is a $\mathcal G_M$-measurable linear transformation of $\Gamma_{i,M}$. Their variance sum is
\begin{equation}
\sum_{i=1}^M\mathbb E\big[Z_{i,M}(u)^2\mid\mathcal G_M\big]
=
u^\top\Omega_M u.
\label{eq:proof_aggregate_varsum}
\end{equation}
By $\|\Omega_M-\Omega\|_{\mathrm{op}}\to_p 0$, this variance converges in probability to $u^\top\Omega u$. Since scores, shortlist sizes, and selector derivatives are uniformly bounded, there exists $C_\Psi<\infty$ such that
\begin{equation}
\max_{1\le i\le M}\|\boldsymbol\Psi_{i,M}\|
\le C_\Psi
\label{eq:proof_aggregate_psibound}
\end{equation}
for all sufficiently large $M$. Hence $\max_i|Z_{i,M}(u)|\le\|u\|C_\Psi/\sqrt M\to0$, giving the conditional Lindeberg condition. The conditional Lindeberg--Feller theorem, followed by Cram\'er--Wold and the usual subsequence argument for convergence in probability of the conditional laws, gives
\begin{equation}
\frac{1}{\sqrt M}\sum_{i=1}^M\boldsymbol\Psi_{i,M}
\rightsquigarrow
N(0,\Omega)
\qquad
\text{conditionally on }\mathcal G_M\text{ in probability}.
\label{eq:proof_aggregate_sumclt}
\end{equation}
Combining \eqref{eq:proof_aggregate_sumclt} with \eqref{eq:lem_aggregate_linear} proves \eqref{eq:lem_aggregate_clt}.
\end{proof}

\begin{theorem}[Formal selection-aware linearization]
\label{thm:siren_linearization_formal}
Under Assumption~\ref{ass:siren_regime}, jointly over $j\le N$ and $B\in\mathcal B$,
\begin{equation}
\mathbf Z_M
=
\frac{1}{\sqrt M}
\sum_{i=1}^M
\boldsymbol\Psi_{i,M}
+
o_p(1).
\label{eq:thm_formal_linearization}
\end{equation}
Furthermore, if there exists a deterministic positive semidefinite matrix $\Omega$ such that
$\|\Omega_M-\Omega\|_{\mathrm{op}}\to_p0$, then
\begin{equation}
\mathbf Z_M
\rightsquigarrow
N(0,\Omega)
\qquad
\text{conditionally on }\mathcal G_M\text{ in probability}.
\label{eq:thm_formal_clt}
\end{equation}
\end{theorem}

\begin{proof}
The linear representation \eqref{eq:thm_formal_linearization} is exactly the first conclusion of Lemma~\ref{lem:siren_aggregate}. 
The conditional Gaussian limit \eqref{eq:thm_formal_clt} is the second conclusion of Lemma~\ref{lem:siren_aggregate}.
\end{proof}

\subsection{Multiplier bootstrap validity}
\label{app:siren_bootstrap_proof}

\begin{theorem}[Formal multiplier bootstrap validity]
\label{thm:siren_bootstrap_formal}
Work under Assumption~\ref{ass:siren_regime}, and suppose there exists a deterministic positive semidefinite matrix $\Omega$ such that
$\|\Omega_M-\Omega\|_{\mathrm{op}}\to_p0$. 
Assume further that the computable contributions satisfy
\begin{equation}
\frac{1}{M}
\sum_{i=1}^M
\big\|
\widehat{\boldsymbol\Psi}_{i,M}
-
\boldsymbol\Psi_{i,M}
\big\|^2
\to_p0.
\label{eq:app_plugin_condition}
\end{equation}
Let
\begin{equation}
\bar{\widehat{\boldsymbol\Psi}}_M
=
\frac{1}{M}\sum_{i=1}^M\widehat{\boldsymbol\Psi}_{i,M},
\qquad
\widehat{\mathbf G}_M^\ast
=
\frac{1}{\sqrt M}
\sum_{i=1}^M
\zeta_i
\big(
\widehat{\boldsymbol\Psi}_{i,M}
-
\bar{\widehat{\boldsymbol\Psi}}_M
\big),
\label{eq:app_boot_process}
\end{equation}
where $\zeta_1,\dots,\zeta_M\stackrel{\mathrm{i.i.d.}}{\sim}N(0,1)$ are independent of the data. 
Then, with $G\sim N(0,\Omega)$ and $d=N|\mathcal B|$,
\begin{equation}
\sup_{h\in BL_1(\mathbb R^d)}
\left|
\mathbb E\big[h(\widehat{\mathbf G}_M^\ast)\mid \mathcal H_M\big]
-
\mathbb E\big[h(G)\big]
\right|
\to_p0.
\label{eq:app_boot_to_gaussian}
\end{equation}
Moreover,
\begin{equation}
\sup_{h\in BL_1(\mathbb R^d)}
\left|
\mathbb E\big[h(\widehat{\mathbf G}_M^\ast)\mid \mathcal H_M\big]
-
\mathbb E\big[h(\mathbf Z_M)\mid \mathcal G_M\big]
\right|
\to_p0.
\label{eq:app_boot_valid}
\end{equation}
\end{theorem}

\begin{proof}
Use the notation in \eqref{eq:app_boot_process}. 
Conditional on $\mathcal H_M$, the vectors $\widehat{\boldsymbol\Psi}_{i,M}-\bar{\widehat{\boldsymbol\Psi}}_M$ are fixed and the multipliers are i.i.d. standard Gaussian. Hence
\begin{equation}
\widehat{\mathbf G}_M^\ast\mid\mathcal H_M
\sim
N\big(0,\widehat\Omega_M^\ast\big),
\label{eq:proof_boot_cond_gaussian}
\end{equation}
where
\begin{equation}
\widehat\Omega_M^\ast
=
\frac{1}{M}\sum_{i=1}^M
\big(
\widehat{\boldsymbol\Psi}_{i,M}-\bar{\widehat{\boldsymbol\Psi}}_M
\big)
\big(
\widehat{\boldsymbol\Psi}_{i,M}-\bar{\widehat{\boldsymbol\Psi}}_M
\big)^\top.
\label{eq:proof_boot_cov}
\end{equation}
Define the infeasible centered covariance
\begin{equation}
\bar{\boldsymbol\Psi}_M
=
\frac{1}{M}\sum_{i=1}^M\boldsymbol\Psi_{i,M},
\qquad
\widetilde\Omega_M
=
\frac{1}{M}\sum_{i=1}^M
\big(
\boldsymbol\Psi_{i,M}-\bar{\boldsymbol\Psi}_M
\big)
\big(
\boldsymbol\Psi_{i,M}-\bar{\boldsymbol\Psi}_M
\big)^\top .
\label{eq:proof_boot_truecov}
\end{equation}

We first show that $\widetilde\Omega_M$ is close to $\Omega$. Since the $\boldsymbol\Psi_{i,M}$ are conditionally independent, mean zero, and uniformly bounded, each entry of
\begin{equation}
\frac{1}{M}\sum_{i=1}^M
\boldsymbol\Psi_{i,M}\boldsymbol\Psi_{i,M}^\top-
\Omega_M
\label{eq:proof_boot_uncentered_obj}
\end{equation}
has conditional variance of order $M^{-1}$. The dimension is fixed, so
\begin{equation}
\left\|
\frac{1}{M}\sum_{i=1}^M
\boldsymbol\Psi_{i,M}\boldsymbol\Psi_{i,M}^\top-
\Omega_M
\right\|_{\mathrm{op}}
\to_p0.
\label{eq:proof_boot_uncentered}
\end{equation}
Moreover,
\begin{equation}
\mathbb E\big[\|\sqrt M\,\bar{\boldsymbol\Psi}_M\|^2\mid\mathcal G_M\big]
=
\mathrm{tr}(\Omega_M)
=
O_p(1),
\label{eq:proof_boot_barPsi_second}
\end{equation}
where the last equality follows from $\|\Omega_M-\Omega\|_{\mathrm{op}}\to_p 0$. Hence $\bar{\boldsymbol\Psi}_M\bar{\boldsymbol\Psi}_M^\top=o_p(1)$, and
\begin{equation}
\|\widetilde\Omega_M-\Omega\|_{\mathrm{op}}
\to_p0.
\label{eq:proof_boot_truecov_conv}
\end{equation}

We next compare $\widehat\Omega_M^\ast$ with $\widetilde\Omega_M$. Let
\begin{equation}
a_{i,M}=\widehat{\boldsymbol\Psi}_{i,M}-\bar{\widehat{\boldsymbol\Psi}}_M,
\qquad
b_{i,M}=\boldsymbol\Psi_{i,M}-\bar{\boldsymbol\Psi}_M.
\label{eq:proof_boot_ab}
\end{equation}
Then
\begin{align}
\|\widehat\Omega_M^\ast-\widetilde\Omega_M\|_{\mathrm{op}}
&\le
\frac{1}{M}\sum_{i=1}^M
\|a_{i,M}-b_{i,M}\|\,\|a_{i,M}\|
+
\frac{1}{M}\sum_{i=1}^M
\|b_{i,M}\|\,\|a_{i,M}-b_{i,M}\|
\notag\\
&\le
\left(
\frac{1}{M}\sum_{i=1}^M\|a_{i,M}-b_{i,M}\|^2
\right)^{1/2}
\left(
\frac{1}{M}\sum_{i=1}^M\|a_{i,M}\|^2
\right)^{1/2}
\notag\\
&\quad+
\left(
\frac{1}{M}\sum_{i=1}^M\|b_{i,M}\|^2
\right)^{1/2}
\left(
\frac{1}{M}\sum_{i=1}^M\|a_{i,M}-b_{i,M}\|^2
\right)^{1/2}.
\label{eq:proof_boot_cov_compare}
\end{align}
Since
\begin{equation}
a_{i,M}-b_{i,M}
=
\big(\widehat{\boldsymbol\Psi}_{i,M}-\boldsymbol\Psi_{i,M}\big)
-
\big(\bar{\widehat{\boldsymbol\Psi}}_M-\bar{\boldsymbol\Psi}_M\big),
\label{eq:proof_boot_diffab}
\end{equation}
the plug-in condition of $\frac{1}{M}\sum_{i=1}^M
\big\|
\widehat{\boldsymbol\Psi}_{i,M}-\boldsymbol\Psi_{i,M}
\big\|^2
\to_p 0$ implies
\begin{equation}
\frac{1}{M}\sum_{i=1}^M\|a_{i,M}-b_{i,M}\|^2
\to_p0.
\label{eq:proof_boot_absmall}
\end{equation}
The remaining two factors in \eqref{eq:proof_boot_cov_compare} are $O_p(1)$: this is immediate for the $b_{i,M}$ factor from uniform boundedness, and for the $a_{i,M}$ factor from the condition of  $\frac{1}{M}\sum_{i=1}^M
\big\|
\widehat{\boldsymbol\Psi}_{i,M}-\boldsymbol\Psi_{i,M}
\big\|^2
\to_p 0$. Thus
\begin{equation}
\|\widehat\Omega_M^\ast-\widetilde\Omega_M\|_{\mathrm{op}}
\to_p0.
\label{eq:proof_boot_cov_est}
\end{equation}
Combining \eqref{eq:proof_boot_truecov_conv} and \eqref{eq:proof_boot_cov_est} gives
\begin{equation}
\|\widehat\Omega_M^\ast-\Omega\|_{\mathrm{op}}
\to_p0.
\label{eq:proof_boot_cov_final}
\end{equation}

Let $U\sim N(0,I_d)$, $d=N|\mathcal B|$, be independent of everything else, and write
\begin{equation}
A_M=(\widehat\Omega_M^\ast)^{1/2},
\qquad
A=\Omega^{1/2}.
\label{eq:proof_boot_roots}
\end{equation}
The conditional law of $\widehat{\mathbf G}_M^\ast$ given $\mathcal H_M$ is the law of $A_MU$, while $G\sim N(0,\Omega)$ has the law of $AU$. Since matrix square roots are continuous on the cone of positive semidefinite matrices in fixed dimension, \eqref{eq:proof_boot_cov_final} implies $\|A_M-A\|_{\mathrm{op}}\to_p0$. Therefore
\begin{align}
\sup_{h\in BL_1}
\left|
\mathbb E\big[h(\widehat{\mathbf G}_M^\ast)\mid\mathcal H_M\big]
-
\mathbb E\big[h(G)\big]
\right|
&\le
\|A_M-A\|_{\mathrm{op}}\,\mathbb E\|U\|
\to_p0.
\label{eq:proof_boot_bl}
\end{align}
This proves that $\sup_{h\in BL_1}
\left|
\mathbb E\big[h(\widehat{\mathbf G}_M^\ast)\mid \mathcal H_M\big]
-
\mathbb E\big[h(G)\big]
\right|
\to_p 0$. The validity statement $\sup_{h\in BL_1}
\left|
\mathbb E\big[h(\widehat{\mathbf G}_M^\ast)\mid \mathcal H_M\big]
-
\mathbb E\big[h(\mathbf Z_M)\mid \mathcal G_M\big]
\right|
\to_p 0$
follows by the triangle inequality and Theorem~\ref{thm:siren_linearization_formal}.
\end{proof}

The plug-in condition \eqref{eq:app_plugin_condition} is stated abstractly in Theorem~\ref{thm:siren_bootstrap_formal}. 
The next proposition gives a sufficient condition for this requirement in the smooth-selector regime.

\subsection{Plug-in consistency for the computable contributions}
\label{app:siren_plugin_consistency}

\begin{proposition}[Sufficient condition for plug-in consistency]
\label{prop:siren_plugin}
Work under Assumption~\ref{ass:siren_regime}. Suppose there exist constants $L_1,L_2<\infty$ such that, for every $j\le N$, $B\in\mathcal B$, and all $x,y\in[0,1]^{K_{j,B}}$,
\begin{equation}
\|Dg_{j,B}(x)\|_{\mathrm{op}}\le L_1,
\qquad
\|Dg_{j,B}(x)-Dg_{j,B}(y)\|_{\mathrm{op}}
\le L_2\|x-y\|.
\label{eq:prop_plugin_lip}
\end{equation}
Then
\begin{equation}
\frac{1}{M}\sum_{i=1}^M
\big\|
\widehat{\boldsymbol\Psi}_{i,M}-\boldsymbol\Psi_{i,M}
\big\|^2
=
O_p(M^{-1}),
\label{eq:prop_plugin_rate}
\end{equation}
and in particular $\frac{1}{M}\sum_{i=1}^M
\big\|
\widehat{\boldsymbol\Psi}_{i,M}-\boldsymbol\Psi_{i,M}
\big\|^2
\to_p 0$ holds.
\end{proposition}

\begin{proof}
By Lemma~\ref{lem:siren_input_representation}, the coordinatewise conditional second moments of $\widehat S_{r,j}(B)-s_{r,j,M}(B)$ and $\widehat T_{r,j}(B)-t_{r,j,M}(B)$ are of order $M^{-1}$. Since the index set is finite,
\begin{equation}
\big\|
\widehat S_{r,j}(B)-s_{r,j,M}(B)
\big\|
=
O_p(M^{-1/2}),
\qquad
\big\|
\widehat T_{r,j}(B)-t_{r,j,M}(B)
\big\|
=
O_p(M^{-1/2})
\label{eq:proof_plugin_ST}
\end{equation}
jointly over the finite index set. The bound \eqref{eq:prop_plugin_lip} gives
\begin{equation}
\big\|
\widehat q_{r,j}(B)-q_{r,j,M}(B)
\big\|
=
O_p(M^{-1/2}),
\;
\big\|
Dg_{j,B}\big(\widehat S_{r,j}(B)\big)-Dg_{j,B}\big(s_{r,j,M}(B)\big)
\big\|_{\mathrm{op}}
=
O_p(M^{-1/2}).
\label{eq:proof_plugin_qA}
\end{equation}

Since the scores are bounded and the split fractions stay away from zero, there exists $C_\Gamma<\infty$ such that
\begin{equation}
\max_{i,r,j,B}
\Big(
\|\Gamma^{S}_{i,r,j,M}(B)\|,
\|\Gamma^{T}_{i,r,j,M}(B)\|,
\|\widehat\Gamma^{S}_{i,r,j,M}(B)\|,
\|\widehat\Gamma^{T}_{i,r,j,M}(B)\|
\Big)
\le C_\Gamma
\label{eq:proof_plugin_Gammabound}
\end{equation}
for all sufficiently large $M$. Moreover,
\begin{equation}
\widehat\Gamma^{S}_{i,r,j,k,M}(B)-\Gamma^{S}_{i,r,j,k,M}(B)
=
\frac{M}{|D_r^{\mathrm{score}}|}
\mathbf 1\{i\in D_r^{\mathrm{score}}\}
\big(s_{r,j,k,M}(B)-\widehat S_{r,j,k}(B)\big),
\label{eq:proof_plugin_hatGammaS}
\end{equation}
and similarly for $\widehat\Gamma^{T}_{i,r,j,k,M}(B)-\Gamma^{T}_{i,r,j,k,M}(B)$. Therefore,
\begin{equation}
\frac{1}{M}\sum_{i=1}^M
\big\|
\widehat\Gamma^{S}_{i,r,j,M}(B)-\Gamma^{S}_{i,r,j,M}(B)
\big\|^2
=
O_p(M^{-1}),
\label{eq:proof_plugin_hatGammaS_rate}
\end{equation}
\begin{equation}
\frac{1}{M}\sum_{i=1}^M
\big\|
\widehat\Gamma^{T}_{i,r,j,M}(B)-\Gamma^{T}_{i,r,j,M}(B)
\big\|^2
=
O_p(M^{-1}).
\label{eq:proof_plugin_hatGammaT_rate}
\end{equation}

For fixed $r,j,B$, define
\begin{equation}
U_{i,r,j,B,M}
=
q_{r,j,M}(B)^\top\Gamma^{T}_{i,r,j,M}(B),
\qquad
\widehat U_{i,r,j,B,M}
=
\widehat q_{r,j}(B)^\top\widehat\Gamma^{T}_{i,r,j,M}(B).
\label{eq:proof_plugin_U}
\end{equation}
Using \eqref{eq:proof_plugin_qA}, \eqref{eq:proof_plugin_Gammabound}, and \eqref{eq:proof_plugin_hatGammaT_rate},
\begin{equation}
\frac{1}{M}\sum_{i=1}^M
\big|
\widehat U_{i,r,j,B,M}-U_{i,r,j,B,M}
\big|^2
=
O_p(M^{-1}).
\label{eq:proof_plugin_Urate}
\end{equation}
Likewise define
\begin{equation}
V_{i,r,j,B,M}
=
t_{r,j,M}(B)^\top
Dg_{j,B}\big(s_{r,j,M}(B)\big)
\Gamma^{S}_{i,r,j,M}(B),
\label{eq:proof_plugin_V}
\end{equation}
\begin{equation}
\widehat V_{i,r,j,B,M}
=
\widehat T_{r,j}(B)^\top
Dg_{j,B}\big(\widehat S_{r,j}(B)\big)
\widehat\Gamma^{S}_{i,r,j,M}(B).
\label{eq:proof_plugin_Vhat}
\end{equation}
The same bounds, together with \eqref{eq:proof_plugin_hatGammaS_rate}, imply
\begin{equation}
\frac{1}{M}\sum_{i=1}^M
\big|
\widehat V_{i,r,j,B,M}-V_{i,r,j,B,M}
\big|^2
=
O_p(M^{-1}).
\label{eq:proof_plugin_Vrate}
\end{equation}
Finally,
\begin{equation}
\widehat\psi_{i,j,B,M}-\psi_{i,j,B,M}
=
\sum_{r=1}^R\omega_r
\Big[
\big(\widehat U_{i,r,j,B,M}-U_{i,r,j,B,M}\big)
+
\big(\widehat V_{i,r,j,B,M}-V_{i,r,j,B,M}\big)
\Big].
\label{eq:proof_plugin_psidiff}
\end{equation}
Since $R$ is fixed and the weights are nonnegative and sum to one,
\begin{align}
\frac{1}{M}\sum_{i=1}^M
\big|
\widehat\psi_{i,j,B,M}-\psi_{i,j,B,M}
\big|^2
&\le
2\sum_{r=1}^R\omega_r
\frac{1}{M}\sum_{i=1}^M
\big|
\widehat U_{i,r,j,B,M}-U_{i,r,j,B,M}
\big|^2
\notag\\
&\quad+
2\sum_{r=1}^R\omega_r
\frac{1}{M}\sum_{i=1}^M
\big|
\widehat V_{i,r,j,B,M}-V_{i,r,j,B,M}
\big|^2
=
O_p(M^{-1}).
\label{eq:proof_plugin_single}
\end{align}
Summing over the finite index set $(j,B)$ proves \eqref{eq:prop_plugin_rate}.
\end{proof}

\begin{remark}
The multiplier bootstrap uses one multiplier per benchmark item rather than one multiplier per split because the same item can appear in different roles across repeated splits and because the same evaluation pool is used across systems and budgets. The item-level vector $\widehat{\boldsymbol\Psi}_{i,M}$ collects all these roles before resampling. Using a common multiplier for the whole vector preserves the dependence structure needed for simultaneous bands and fixed contrasts.
\end{remark}
%
%

\section{Additional experiments}\label{sec:additional_experiments}
 
This appendix collects three sets of supplementary results that support the
main paper's claims across all four (subject, tuner) configurations:
(i)~Track~I theory-validation studies in controlled simulation
(Section~\ref{sec:track1});
(ii)~per-model SIREN-vs-baseline comparisons on real data
(Section~\ref{sec:appendix_siren_per_model}); and
(iii)~per-model and per-budget PromptEval-vs-SIREN comparisons
(Section~\ref{sec:appendix_pe}). 
 
\FloatBarrier

\subsection{Theory Validation}\label{sec:track1}
 
Track~I validates the three core theoretical results in controlled settings where ground truth is known by construction.
All three studies use a common Bernoulli item-response DGP: item~$i$ has difficulty $\delta_i \sim \mathrm{Uniform}(-2,2)$, artifact~$k$ has quality~$q_k$, and $Z_{ik} \sim \mathrm{Bernoulli}(\sigma(q_k - \delta_i))$.
 
\subsubsection{Study~A: Coverage Validation of the Multiplier Bootstrap}
\label{sec:studyA}
 
Study~A validates the three central operational claims of Theorems~\ref{thm:positive_smooth} and~\ref{thm:bootstrap_validity} in a controlled setting where the ground truth is tractable by Monte~Carlo: (i)~the multiplier bootstrap on the plug-in influence function $\widehat\psi$ achieves nominal coverage for the repeated-split estimator $\widetilde\theta$; (ii)~its confidence bands shrink at the theoretical $M^{-1/2}$ rate; (iii)~it reproduces a nonparametric item-bootstrap's coverage and width at a fraction of the computational cost.
 
\paragraph{Data generating process.}
Items have i.i.d.\ difficulties $\delta_i\sim\text{Uniform}(-2,2)$ and $K$ artifacts have qualities equally spaced over $[0, 0.3]$; scores are $Z_{ik}\sim\text{Bernoulli}(\sigma(q_k-\delta_i))$.
The selector is softmax with $\tau{=}0.1$, the split fraction is $\rho{=}0.5$, and the bootstrap uses $B_{\mathrm{boot}}{=}500$ Gaussian multipliers.
Ground truth $\theta^\star = \mathbb E[\widetilde\theta]$ is estimated by $N_{\mathrm{gt}}{=}3{,}000$ independent Monte~Carlo replications; empirical coverage is then measured over $N_{\mathrm{sim}}{=}2{,}000$ trials per configuration, giving a Monte~Carlo standard error of approximately $\mathrm{SE}\approx 0.5$\,pp at the nominal $95\%$ level.
 
\paragraph{Main result: calibrated coverage across $M$, $K$, and $R$.}
Figure~\ref{fig:studyA_main} shows the coverage sweep at $R{=}5$ across $M\in\{100,200,500,1000,2000\}$ and $K\in\{2,5,10\}$.
Coverage ranges from $93.9\%$ to $95.9\%$, with a maximum deviation of $1.1$ percentage points from nominal (Table~\ref{tab:studyA_main}).
The right panel confirms that CI width tracks the theoretical $M^{-1/2}$ scaling on a log--log plot: the empirical log--log slopes are $-0.510$, $-0.524$, and $-0.520$ for $K{=}2,5,10$, all within $5\%$ of the theoretical $-0.5$, and the three curves run parallel as predicted.
In the practitioner regime of $M\geq 500$ the coverage is always within $0.6$\,pp of nominal and the width is well below $0.1$, showing that the bootstrap is usable at sample sizes typical of LLM benchmarks.
 
\begin{figure}[h]
  \centering
  \includegraphics[width=\textwidth]{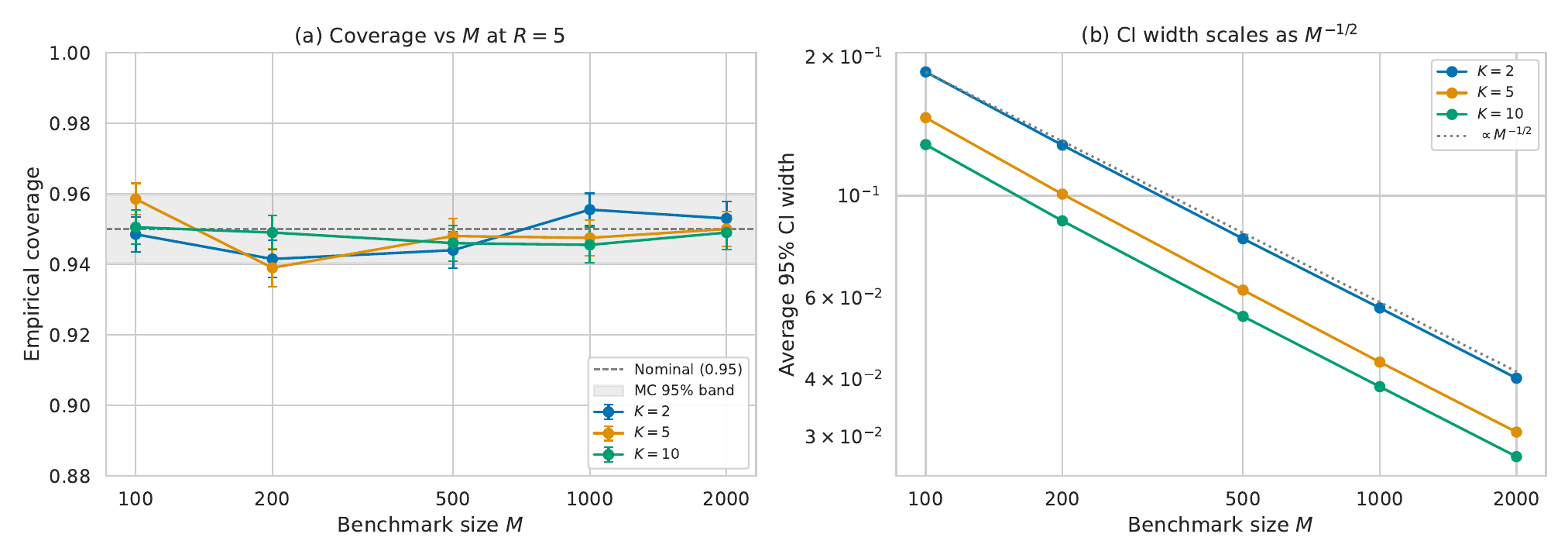}
  \caption{\textbf{Study~A: Coverage and width across $M$ at $R{=}5$.}
  (a)~Empirical coverage of the multiplier-bootstrap 95\% CI. Error bars are $\pm$SE over $N_{\mathrm{sim}}{=}2{,}000$ trials; the grey band is the Monte~Carlo $95\%$ tolerance region around the nominal level.
  (b)~Log--log plot of average CI width; dotted line shows the reference $M^{-1/2}$ slope.}
  \label{fig:studyA_main}
\end{figure}
 
\begin{table}[h]
\centering
\small
\caption{Study~A: Average CI width of the SIREN multiplier bootstrap at $R{=}5$.}
\label{tab:studyA_main}
\begin{tabular}{rccc}
\toprule
 & \multicolumn{3}{c}{Average CI width} \\
\cmidrule(lr){2-4}
$M$ & $K{=}2$ & $K{=}5$ & $K{=}10$ \\
\midrule
100   & 0.186 & 0.148 & 0.129 \\
200   & 0.129 & 0.101 & 0.088 \\
500   & 0.081 & 0.062 & 0.055 \\
1000  & 0.057 & 0.044 & 0.038 \\
2000  & 0.040 & 0.031 & 0.027 \\
\bottomrule
\end{tabular}
\end{table}
 
\paragraph{Effect of the number of repeated splits.}
Figure~\ref{fig:studyA_R} examines the effect of $R$ at fixed $M{=}500$ and $K{=}10$.
Coverage is stable across $R\in\{1,3,5,10\}$ (all within $1$\,pp of nominal), confirming that the influence-function bootstrap is already first-order valid at $R{=}1$.
What $R$ does buy is efficiency: increasing $R$ from $1$ to $5$ contracts the CI width by $22\%$ (from $0.0705$ to $0.0547$), but doubling $R$ further to $10$ contracts it by only another $4\%$ (to $0.0525$).
The curvature reflects the fact that the held-out sample at a given $R$ is not $R$-fold independent but merely $R$-fold reweighted, so the additional splits add vanishing information once $R$ is moderate.
This matches the practitioner recommendation $R{\in}[5,10]$ used throughout the remainder of the paper.
 
\begin{figure}[h]
  \centering
  \includegraphics[width=\textwidth]{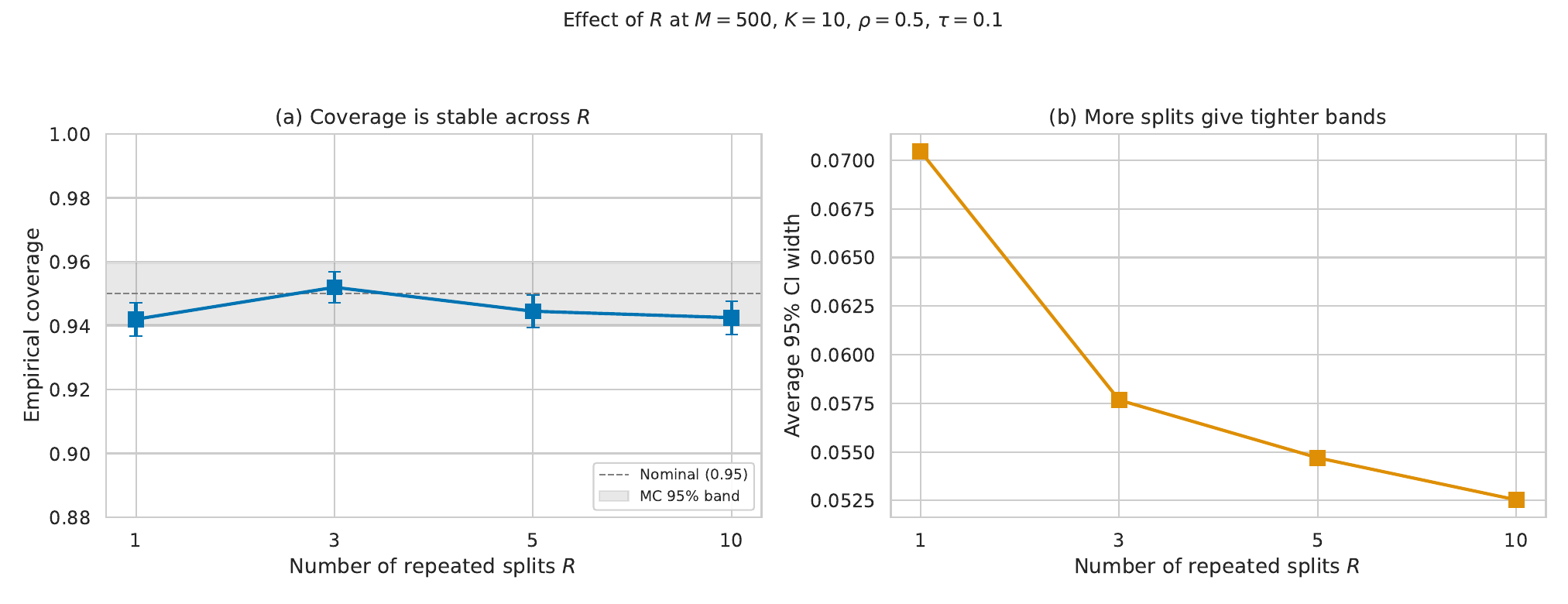}
  \caption{\textbf{Study~A: Effect of the number of repeated splits $R$} at fixed $M{=}500$, $K{=}10$.
  (a)~Coverage is stable across $R$ (all within the MC 95\% band).
  (b)~Width contracts quickly up to $R{=}5$, with diminishing returns beyond.}
  \label{fig:studyA_R}
\end{figure}
 
\paragraph{SIREN matches a nonparametric item-bootstrap at a fraction of the cost.}
A practitioner could, in principle, bypass the influence-function machinery and instead resample items with replacement and rerun the entire repeated-split pipeline on each resample---a nonparametric item bootstrap that respects split dependence by construction.
This comparison checks whether SIREN's influence-function approximation agrees with the item-bootstrap intervals up to Monte~Carlo noise.
Figure~\ref{fig:studyA_naive}~(a) confirms this: at $K{=}10,\,R{=}5$ and $M\in\{200,500,1000\}$, both methods cover within the MC band (SIREN: $94.6$--$95.0\%$; item-bootstrap: $95.0$--$96.0\%$), and their average widths agree to within $0.8\%$ in every configuration (Figure~\ref{fig:studyA_naive}~(b), ratios $0.994$--$1.008$).
The methods differ in computational cost.
On a representative configuration ($M{=}500$, $K{=}10$, $R{=}5$, $B_{\mathrm{boot}}{=}1000$), the SIREN multiplier bootstrap runs in $7.5$\,ms, while the nonparametric item bootstrap takes $281$\,ms --- a $37.6\times$ speedup from replacing $B_{\mathrm{boot}}$ rerun pipelines with a single length-$M$ dot product per resample.
The gap grows with $R$, $K$, and $B_{\mathrm{boot}}$: at the LLM-evaluation scale of the main MMLU-Pro experiments (Section~\ref{sec:exp_prompteval}) the same speedup reaches $35\times$ over a paired bootstrap and $\sim{}10^5\times$ over the IRT-based PromptEval.
Beyond wall-clock efficiency, the influence function $\widehat\psi$ in \eqref{eq:siren_hat_psi} decomposes additively into a held-out evaluation contribution and a derivative-weighted selection contribution, which lets the same bootstrap deliver a variance decomposition by source and paired CIs across budget pairs --- neither of which has any counterpart in the item-resample baseline.
 
\begin{figure}[h]
  \centering
  \includegraphics[width=\textwidth]{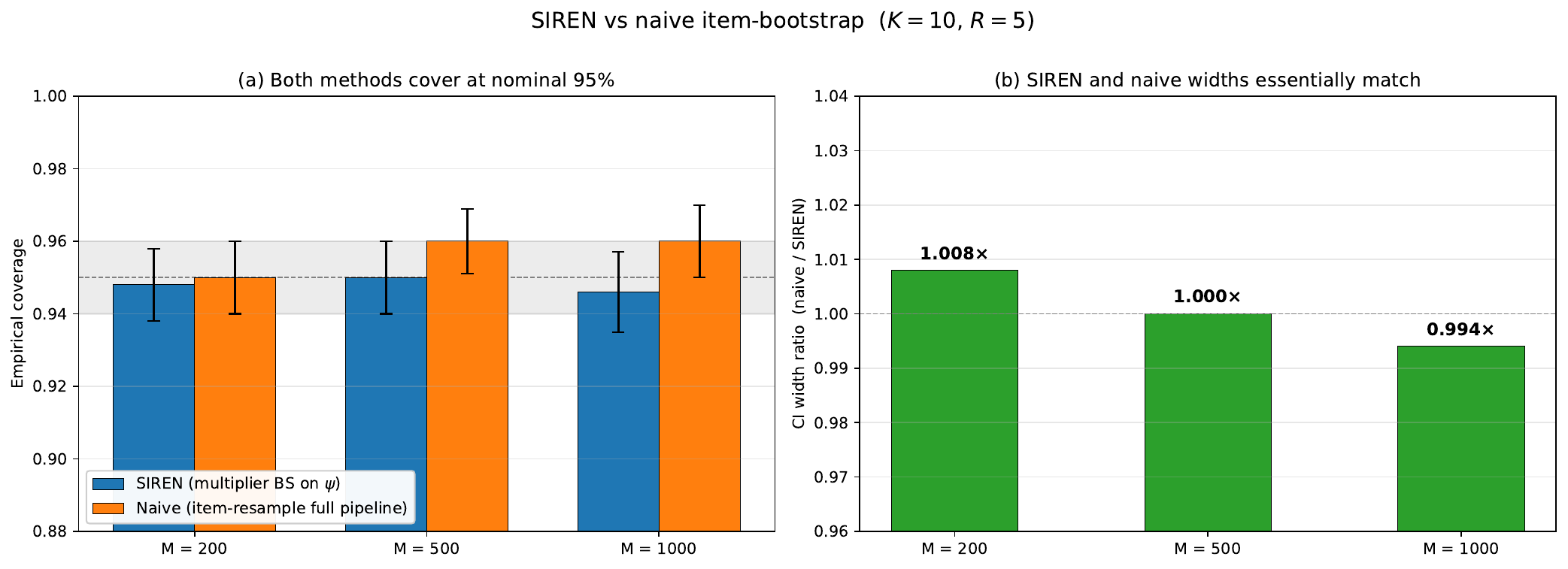}
  \caption{\textbf{Study~A: SIREN versus a nonparametric item-bootstrap baseline} at $K{=}10$, $R{=}5$.
  (a)~Both methods cover at the nominal level (error bars: $\pm 1.96\cdot$SE).
  (b)~CI width ratios are within $\pm 0.8\%$ of unity.
  The item-bootstrap is $\sim{}38\times$ slower because it reruns the full repeated-split pipeline on every resample, while SIREN runs it once and bootstraps a length-$M$ vector.}
  \label{fig:studyA_naive}
\end{figure}
 
\paragraph{Summary.}
Across $22$ trial-configurations totaling more than $40{,}000$ simulated trials, the SIREN multiplier bootstrap is empirically calibrated at the nominal level, scales with the theoretical $M^{-1/2}$ rate, and reproduces a nonparametric item-bootstrap's inference at a $\sim{}38\times$ speedup.
This establishes a clean empirical baseline before Studies~B and~C turn to the regimes (near-tie hard selection, same-data best-of reporting) where the theorem's regularity conditions are designed to fail.
 
\subsubsection{Study~B: Near-Tie Nonregularity of Hard Selection}
\label{sec:studyB}
 
Study~B validates three operational claims about hard-selection nonregularity in the same two-artifact setting where it predicts failure: (i)~hard argmax selection systematically undercovers when the top two artifacts are close in quality; (ii)~soft (softmax) selection, whose smooth Jacobian is exactly the regularity condition required by Theorem~\ref{thm:positive_smooth}, remains calibrated across all gaps; (iii)~the winner-instability diagnostic $\widehat\pi_{\mathrm{win}}$ serves as a practical trigger for switching between the two regimes.
 
\paragraph{Data generating process.}
Items have i.i.d.\ difficulties $\delta_i\sim\mathrm{Uniform}(-2,2)$.
The two artifacts have qualities $q_1 = 0.5 + \Delta$ and $q_2 = 0.5$, so that $\Delta$ directly controls the population margin.
Scores are $Z_{ik}\sim\mathrm{Bernoulli}(\sigma(q_k - \delta_i))$.
We fix $M{=}500$, $R{=}5$, $\rho{=}0.5$, $\tau{=}0.1$, and $B_{\mathrm{boot}}{=}500$, sweep $\Delta$ over $17$ values from $0.00$ to $0.80$, and use $N_{\mathrm{sim}}{=}2{,}000$ trials and $N_{\mathrm{gt}}{=}3{,}000$ Monte~Carlo replications per configuration (MC SE $\approx 0.5$\,pp at nominal coverage).
For each method, ground truth $\theta^\star$ is computed separately by taking the expectation of that method's $\widetilde\theta$ under the same DGP, so that coverage measures whether each method's CI covers \emph{its own} procedure-level target.
 
\paragraph{Main result: a U-shaped undercoverage curve for hard selection.}
Figure~\ref{fig:studyB_main}~(a) shows the coverage sweep.
Hard-selection coverage traces a clear U-shape: it is close to nominal at $\Delta{=}0$ (where the two artifacts are identical, so misselection is harmless), drops to a worst of \textbf{$89.8\%$} at $\Delta{=}0.20$ and $0.30$ ($5.2$ percentage points below nominal), and recovers to within MC noise once $\Delta\ge 0.35$.
Soft selection, in contrast, holds between $93.7\%$ and $96.0\%$ coverage across the entire sweep.
Table~\ref{tab:studyB} reports representative points.
The U-shape is the empirical signature of hard-selection nonregularity: at $\Delta{=}0$ the two artifacts are interchangeable; at large $\Delta$ the winner is stable; between these two extremes is a window where the winner flips frequently yet misselection carries a real cost, and that is exactly where hard bootstrap fails.
 
\begin{figure}[h]
  \centering
  \includegraphics[width=\textwidth]{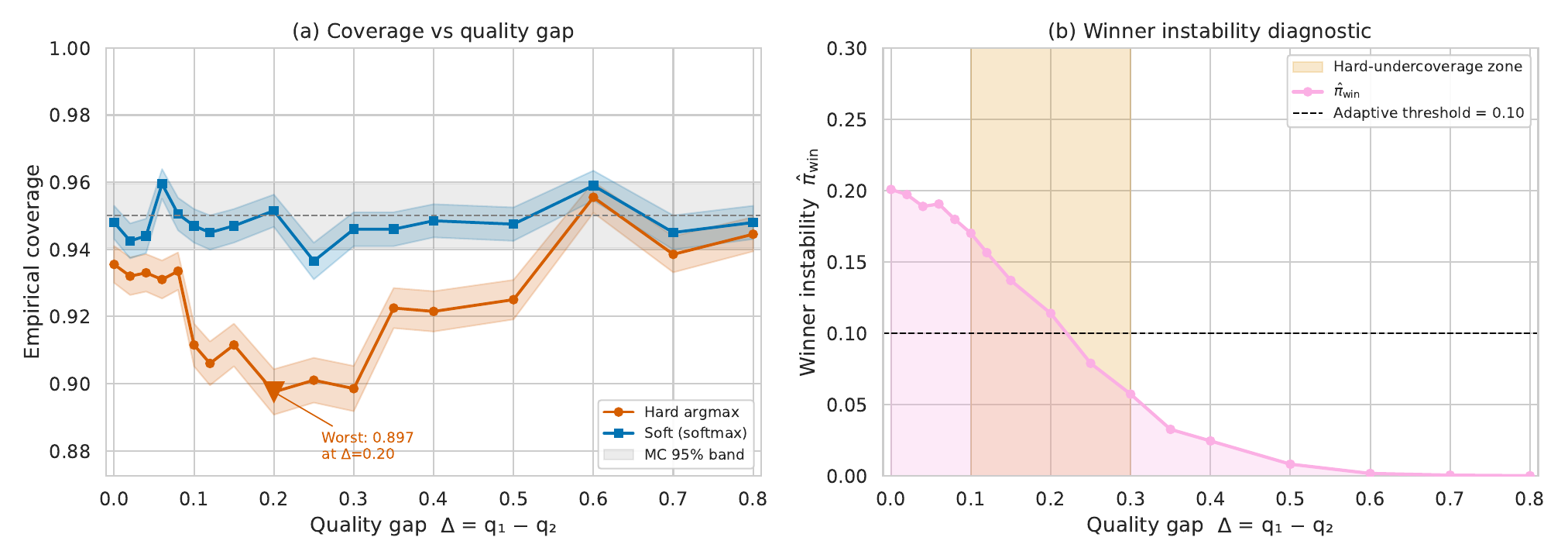}
  \caption{\textbf{Study~B: Hard-selection coverage fails in the near-tie regime.}
  (a)~Empirical coverage vs.\ quality gap $\Delta$; error bands are $\pm$SE over $N_{\mathrm{sim}}{=}2{,}000$ trials, grey band is the Monte~Carlo $95\%$ tolerance region around the nominal level.
  Hard (orange) drops to $89.8\%$ at $\Delta{=}0.20$; soft (blue) is stable across the sweep.
  (b)~Winner instability $\widehat\pi_{\mathrm{win}}$ and the adaptive threshold $0.10$. The shaded orange band marks the empirical hard-undercoverage zone $\Delta\in[0.10,0.30]$ and aligns with the regime where $\widehat\pi_{\mathrm{win}}$ crosses the threshold from above.}
  \label{fig:studyB_main}
\end{figure}
 
\begin{table}[h]
\centering
\small
\caption{Study~B: Empirical coverage (\%) and winner instability at selected gap levels ($M{=}500$, $K{=}2$, $R{=}5$, Monte~Carlo SE $\approx 0.5$\,pp at $N_{\mathrm{sim}}{=}2{,}000$). Hard-selection coverage drops up to $5.2$\,pp below nominal; soft is stable.}
\label{tab:studyB}
\begin{tabular}{r cc c c}
\toprule
Gap $\Delta$ & Hard (\%) & Soft (\%) & Adaptive (\%) & $\widehat\pi_{\mathrm{win}}$ \\
\midrule
0.00 & 93.6 & 94.8 & 93.8 & 0.201 \\
0.06 & 93.1 & 96.0 & 94.1 & 0.191 \\
0.10 & 91.1 & 94.7 & 92.8 & 0.170 \\
0.15 & 91.1 & 94.7 & 93.2 & 0.137 \\
0.20 & 89.8 & 95.2 & 93.5 & 0.114 \\
0.30 & 89.8 & 94.6 & 93.5 & 0.057 \\
\bottomrule
\end{tabular}
\end{table}
 
\paragraph{Mechanism: hard bootstrap estimates the wrong variance.}
Figure~\ref{fig:studyB_mechanism} visualizes exactly why hard selection fails.
For each gap we compare the true standard deviation of $\widetilde\theta$ across trials (solid) with the standard deviation implied by the average CI width (dashed; half-width divided by $1.96$, i.e., the Gaussian-CI interpretation of the bootstrap quantile).
Under soft selection the two curves track each other to within Monte~Carlo noise at every $\Delta$, confirming Theorem~\ref{thm:bootstrap_validity}.
Under hard selection the two curves diverge in the near-tie window: the true std climbs from $\approx 0.023$ at $\Delta{=}0$ to a peak of $\approx 0.027$ around $\Delta{=}0.25$, while the bootstrap-estimated std remains essentially flat near $0.022$.
The shaded region is the \emph{missed variance}: a real source of randomness---the discrete jump in $\widehat k^\star$ across splits---that the plug-in influence function $\widehat\psi^{\mathrm{hard}}$ is blind to because the argmax selector has zero derivative almost everywhere.
At $\Delta{=}0.20$ the bootstrap underestimates the std by approximately $14\%$, with the maximum miss of $\approx 15\%$ occurring at $\Delta{=}0.25$; both lie squarely inside the empirical undercoverage zone $\Delta\in[0.10,0.30]$.
 
\begin{figure}[h]
  \centering
  \includegraphics[width=\textwidth]{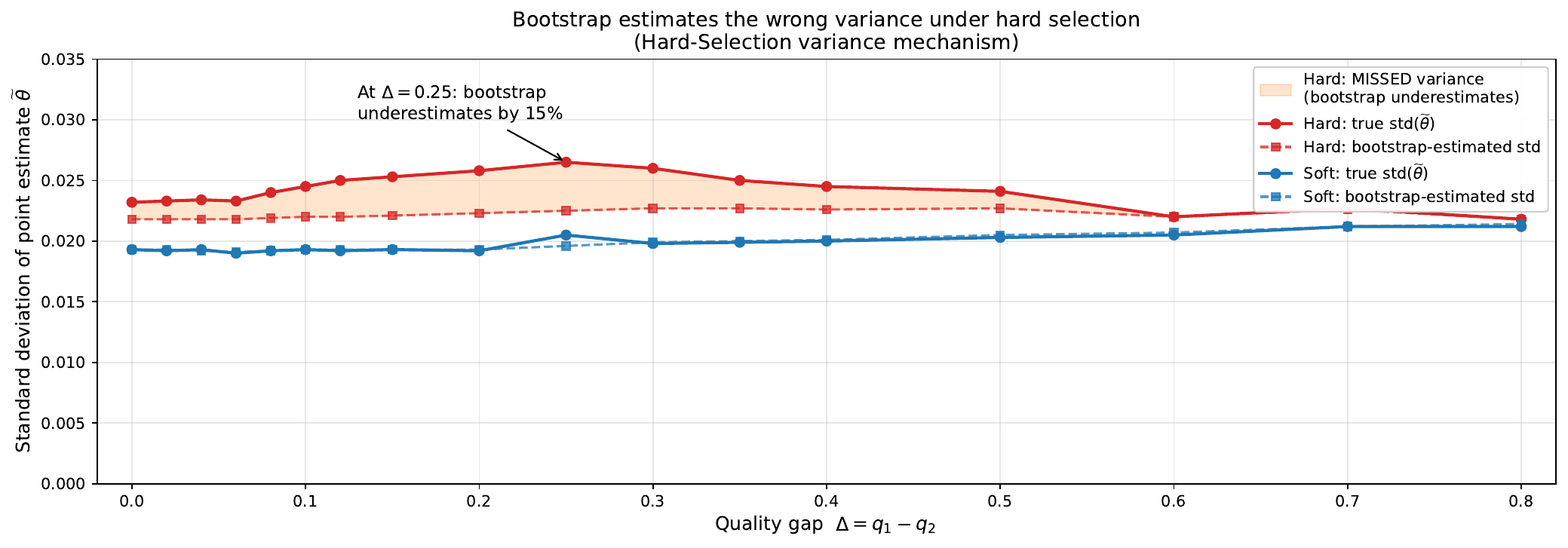}
  \caption{\textbf{Study~B: Why hard-selection bootstrap undercovers.}
  The orange shaded area is the variance that the hard influence function $\widehat\psi^{\mathrm{hard}}$ does not capture: the discontinuous jump of the argmax winner across splits.
  Soft selection has a non-zero Jacobian and therefore estimates its own variance correctly (blue lines track each other).}
  \label{fig:studyB_mechanism}
\end{figure}
 
\paragraph{Instability-triggered adaptive selection.}
Figure~\ref{fig:studyB_adaptive} evaluates an instability-triggered adaptive rule: on each trial, use hard selection when the winner-instability frequency $\widehat\pi_{\mathrm{win}}$ (the fraction of repeated splits whose argmax differs from the across-split majority winner) satisfies $\widehat\pi_{\mathrm{win}} \le 0.10$, and switch to soft otherwise.
Across the sweep, adaptive coverage improves on hard by $2$--$4$\,pp in the near-tie window, recovering to $93.5\%$ at the worst gap $\Delta{=}0.20$ (vs.\ $89.8\%$ for hard).
At $\Delta{=}0.40$, the width ratios relative to hard selection are $\bar w_{\mathrm{adapt}}/\bar w_{\mathrm{hard}}=0.99$ and $\bar w_{\mathrm{soft}}/\bar w_{\mathrm{hard}}=0.89$ (Figure~\ref{fig:studyB_adaptive}(b)).
However, the rule does \emph{not} fully close the gap at boundary points: at $\Delta{=}0.20$ the population $\widehat\pi_{\mathrm{win}}$ is $0.114$, just above the threshold, so individual trials oscillate across the decision boundary and the ensemble coverage sits between hard and soft.
Soft selection has the most stable coverage in this sweep. The adaptive rule does not inherit the smooth-selector guarantee because it switches discontinuously.
 
\begin{figure}[h]
  \centering
  \includegraphics[width=\textwidth]{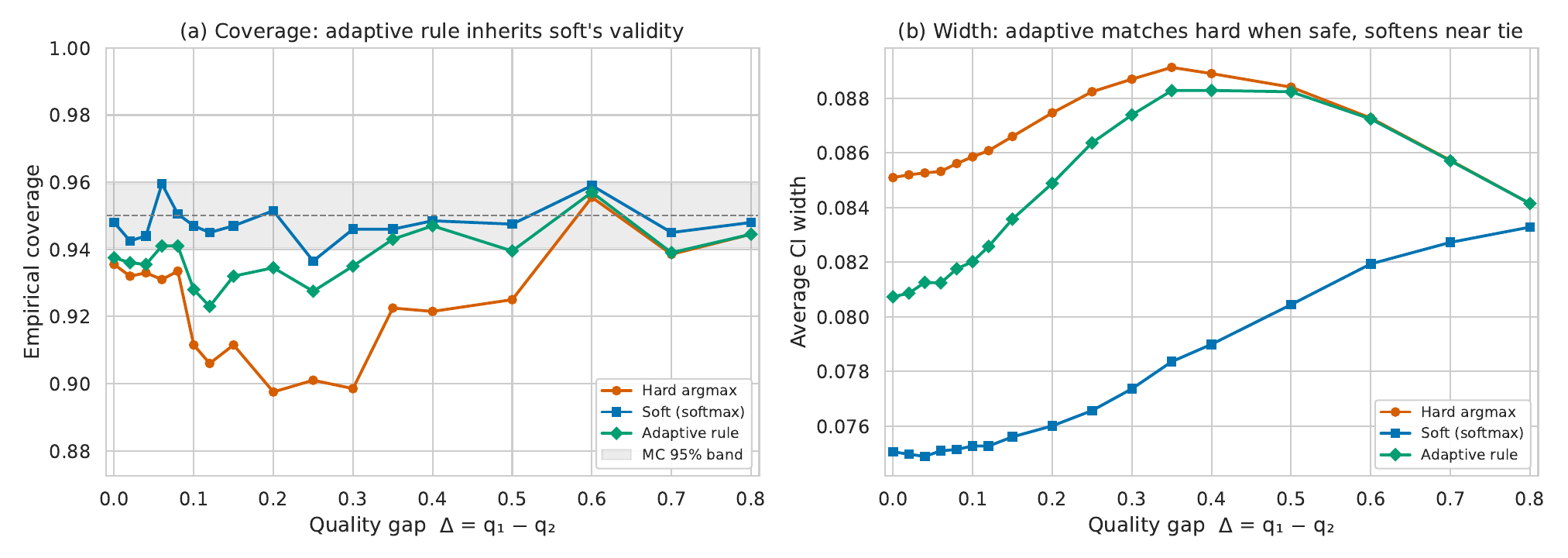}
  \caption{\textbf{Study~B: Instability-triggered adaptive rule.}
  (a)~Coverage of hard, soft, and adaptive across $\Delta$. Adaptive (green) improves on hard near ties but does not fully match soft at boundary points where $\widehat\pi_{\mathrm{win}}\approx 0.10$.
  (b)~Average CI width: adaptive coincides with hard in the stable regime (large $\Delta$) and pulls toward soft near ties, giving a dynamic trade-off between width and coverage.}
  \label{fig:studyB_adaptive}
\end{figure}
 
\paragraph{Summary.}
Across $17$ gap levels and $2{,}000$ trials each, hard argmax selection is empirically nonregular in a narrow but operationally relevant window around $\Delta\approx 0.10$--$0.30$, where its bootstrap CIs lose up to $5.2$\,pp of coverage.
The variance discrepancy arises because the plug-in influence function for argmax has zero derivative and therefore misses the selection-jump component of the true variance ($\approx 15\%$ underestimation of $\mathrm{std}(\widetilde\theta)$ at the peak).
Soft selection includes the selection contribution through its Jacobian and maintains coverage close to nominal across the sweep.
The winner-instability diagnostic identifies the sensitive region, but the adaptive rule retains some undercoverage near its switching threshold.
Together, Studies~A and~B show that Theorem~\ref{thm:positive_smooth}'s smoothness assumption is relevant to interval calibration in the near-tie regime.
 
\subsubsection{Study~C: Same-Data Optimism}
\label{sec:studyC}
 
Study~C empirically validates the same-data optimism phenomenon in a setting that is adversarial to any method relying on same-data best-of reporting: two systems with \emph{identical} true performance but unequal search budgets. The question is whether the method under evaluation avoids declaring a spurious winner, or whether it inflates the score of whichever system happened to try more artifacts.
 
\paragraph{Data generating process.}
Items have i.i.d.\ difficulties $\delta_i \sim \text{Uniform}(-2, 2)$, and every artifact in every system shares the identical quality $q_k \equiv 0.5$, so all systems have common population mean $\theta^\star \approx 0.596$.
We fix $M{=}500$, $R{=}5$, $\rho{=}0.5$, $\tau{=}0.1$, and $B_{\mathrm{boot}}{=}500$, with $N_{\mathrm{sim}}{=}2{,}000$ trials and $N_{\mathrm{gt}}{=}10{,}000$ Monte~Carlo replications per configuration.
The bias sweep (Table~\ref{tab:studyC} and Figure~\ref{fig:studyC}~(left)) varies $H_A\in\{3,5,10,20,50\}$ against a fixed comparator System~B with $H_B{=}3$ artifacts under shared item draws, so that bias and false-winner rate are measured on the same paired comparison.
The unequal-search bar chart (Figure~\ref{fig:studyC}~(right)) sweeps six $(H_A,H_B)$ pairs with independent item draws per system, the regime of practical interest where two teams report on disjoint evaluation runs.
 
\begin{table}[h]
\centering
\caption{Study~C: Same-data optimism ($M{=}500$, all artifacts identical, $\theta^\star \approx 0.596$, comparator $H_B{=}3$, paired item draws). The final columns report the frequency of declaring System~A the winner. Same-data reporting compares point estimates, whereas SIREN uses nonoverlapping $95\%$ CIs. The $50\%$ symmetry reference applies to the point-estimate comparison at equal search effort.}
\label{tab:studyC}
\small
\begin{tabular}{r cc cc}
\toprule
& \multicolumn{2}{c}{Optimism bias (pp)} & \multicolumn{2}{c}{FWR: A wins (\%)} \\
\cmidrule(lr){2-3} \cmidrule(lr){4-5}
$H_A$ & Same-data & SIREN & Same-data & SIREN \\
\midrule
3   & $+1.5$ & $-0.2$ & $48.3$ & $0.0$ \\
5   & $+2.3$ & $\,\,\,\,0.0$ & $62.1$ & $0.0$ \\
10  & $+3.1$ & $+0.1$ & $76.4$ & $0.0$ \\
20  & $+3.7$ & $+0.1$ & $86.3$ & $0.0$ \\
50  & $+4.0$ & $-0.3$ & $94.0$ & $0.0$ \\
\bottomrule
\end{tabular}
\end{table}
 
\paragraph{Main result: same-data bias tracks the $\sqrt{2\log H}$ theory.}
Table~\ref{tab:studyC} and the left panel of Figure~\ref{fig:studyC} report the optimism bias $\mathbb{E}[\widehat\theta - \theta^\star]$ of the two estimators on System~A.
Same-data best-of reporting inflates System~A's score from $+1.5$\,pp at $H_A{=}3$ to $+4.0$\,pp at $H_A{=}50$, and its growth tracks the theoretical rate $\sigma\sqrt{2\log H_A}/\sqrt{M}$ (dashed line) up to a constant factor; the empirical bias sits below the theoretical curve because the maximum-of-Gaussians bound is conservative in finite-$M$ Bernoulli scores.
The repeated-split estimator $\widetilde\theta$ shows bias bounded by $0.4$\,pp across all library sizes, because its held-out evaluation is independent of the development scores used to form the softmax weights.
 
\paragraph{Mechanism: the winner of a noisy ensemble is biased upward.}
The root cause is an elementary extreme-value fact.
For independent Gaussian estimation errors with standard error $\sigma/\sqrt{M}$, the expected maximum grows on the order of $\sigma\sqrt{2\log H}/\sqrt{M}$.
Same-data best-of reporting picks the arg-max \emph{and} reports its value on the same data that identified it as the winner, so the reported score inherits the entire extreme-value inflation.
The repeated-split estimator evaluates the softmax weights on held-out items, separating the scores used to choose the weights from the scores used to report performance.
 
\paragraph{Effect of unequal search effort on rankings.}
The practical consequence appears when the two systems have unequal $H$.
Same-data reporting inflates both systems' scores, but inflates the larger-library system more, creating a spurious gap between the two reports.
At $H_A{=}50$ vs.\ $H_B{=}3$ in the paired-item table, this gap is $+4.0 - (+1.5) \approx 2.5$\,pp of apparent advantage for System~A, despite \emph{identical} ground truth.
The FWR columns of Table~\ref{tab:studyC} translate this into false rankings: the probability that same-data reporting declares System~A the winner climbs from $48.3\%$ at $H_A{=}3$ (ties, near the $50\%$ fair baseline) to $\mathbf{94.0\%}$ at $H_A{=}50$.
The right panel of Figure~\ref{fig:studyC} shows the same effect for the more practical regime in which two teams evaluate on independent item draws: across six $(H_A,H_B)$ pairs, the win rate of the system with the larger search library ranges from $70.7\%$ at $(20,5)$ to $87.6\%$ at $(100,5)$, all well above the fair $50\%$ baseline.
SIREN's $95\%$ CIs overlap in every reported trial, so the interval rule does not declare a winner. The two rules use different evidence thresholds, and overlapping intervals do not establish equivalence.
 
\begin{figure}[h]
  \centering
  \includegraphics[width=\textwidth]{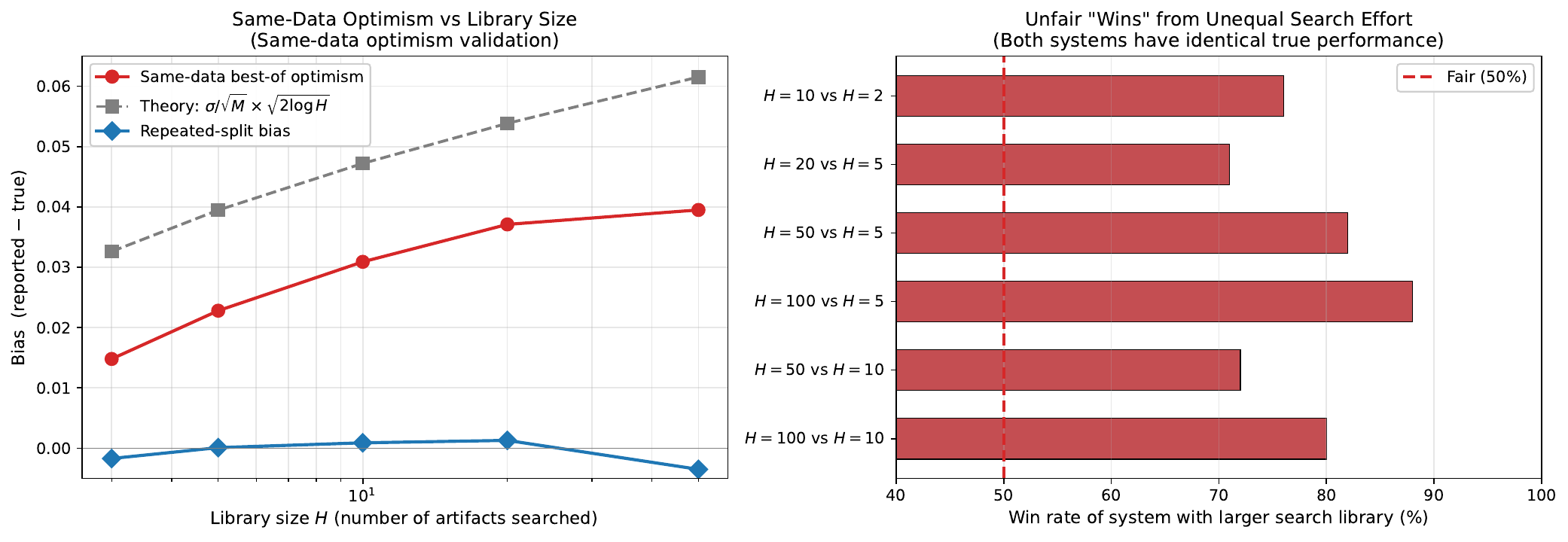}
  \caption{\textbf{Study~C: same-data optimism validation.}
  \emph{Left}: Optimism bias vs.\ library size $H_A$ at fixed $H_B{=}3$ (paired item draws).
  Same-data reporting (red) produces bias growing roughly as $\sigma\sqrt{2\log H}/\sqrt{M}$ (dashed); the repeated-split estimator (blue) has near-zero bias.
  \emph{Right}: When two systems with identical true quality search over unequal libraries (independent item draws), same-data reporting declares the system with the larger library the winner $71$--$88\%$ of the time across six pairings, despite ground-truth equality.}
  \label{fig:studyC}
\end{figure}
 
\paragraph{Summary.}
A benchmark report that compares two systems via their same-data best-of scores is not measuring procedure quality; it is partly measuring who searched harder.
Across $5$ paired library sizes and $6$ unpaired pairings totaling $22{,}000$ trials, the false-winner rate under same-data reporting grows from chance-level $48\%$ to near-certain $94\%$ as the search-effort imbalance widens, in qualitative agreement with the $\sqrt{2\log H}$ extreme-value theory.
SIREN reports no winner under its interval rule in these simulations, while the point-estimate comparisons show the ranking effect of unequal search effort.
 
\FloatBarrier

\subsection{Per-model SIREN-vs-baseline comparisons}
\label{sec:appendix_siren_per_model}

For each (subject, tuner) cell we report the per-model directional calls
and point estimates of SIREN against the four internal baselines: M1/M2
(naive max, Wald), M3 (single-split holdout), and M4 ($R$-split $t$-CI).
Tables~\ref{tab:siren_math_rs_all_budgets}--\ref{tab:siren_law_dspy_all_budgets}
give the per-model breakdown for all four cells. Across the four subject--tuner cells, SIREN has the highest aggregate directional accuracy under the common reporting reference in three cells and ties M1/M2 and M4 on Law with DSPy.
The aggregated directional accuracy across these per-model breakdowns
appears in Tables~\ref{tab:math_directional} and~\ref{tab:law_directional}
of the main text.

\FloatBarrier


\begin{table*}[!htbp]
\centering
\caption{
Unified per-model comparison across budgets on MMLU-Pro Math with the \textbf{random search} tuner.
The untuned-pipeline-default reference $\theta_B^\star$ is fixed across budgets and shown once at the top.
For each budget, $\theta_A^\star$ is the common soft-rule reference, and the method rows report scores $\hat\theta_A$ compared with it. The \cmark/\xmark marker indicates whether
$\mathrm{sign}(\hat\theta_A-\theta_B^\star)$ agrees with the true direction
$\mathrm{sign}(\theta_A^\star-\theta_B^\star)$. M1 and M2 are merged because
they are numerically identical in point estimate and directional call.
}
\label{tab:siren_math_rs_all_budgets}
\scriptsize
\setlength{\tabcolsep}{2.5pt}
\renewcommand{\arraystretch}{0.92}
\resizebox{\textwidth}{!}{%
\begin{tabular}{llccccccccccc}
\toprule
$B$ & Row
& Qwen3 & Phi-3.5 & Qwen2.5-7B & Llama3.1 & Yi-1.5
& InternLM & Qwen2-7B & Qwen2.5-3B & GLM-4 & Mistr-v0.3 & Mistr-v0.1 \\
\midrule

Ref. & $\theta_B^\star$
& 0.329 & 0.192 & 0.325 & 0.244 & 0.225 & 0.245 & 0.284 & 0.262 & 0.232 & 0.163 & 0.153 \\

\midrule
\multirow{6}{*}{$500K$}
& $\theta_A^\star$
& 0.294 & 0.175 & 0.213 & 0.145 & 0.147 & 0.201 & 0.243 & 0.240 & 0.189 & 0.124 & 0.130 \\
& True dir.
& A{<}B & A{<}B & A{<}B & A{<}B & A{<}B & A{<}B & A{<}B & A{<}B & A{<}B & A{<}B & A{<}B \\
& M1/M2
& 0.371\,\xmark & 0.234\,\xmark & 0.324\,\cmark & 0.203\,\cmark & 0.197\,\cmark & 0.259\,\xmark & 0.292\,\xmark & 0.254\,\cmark & 0.233\,\xmark & 0.158\,\cmark & 0.145\,\cmark \\
& M3
& 0.342\,\xmark & 0.233\,\xmark & 0.304\,\cmark & 0.188\,\cmark & 0.200\,\cmark & 0.246\,\xmark & 0.275\,\cmark & 0.236\,\cmark & 0.209\,\cmark & 0.158\,\cmark & \textbf{0.129}\,\cmark \\
& M4
& 0.365\,\xmark & 0.235\,\xmark & 0.323\,\cmark & 0.198\,\cmark & 0.183\,\cmark & 0.252\,\xmark & 0.284\,\xmark & 0.243\,\cmark & 0.224\,\cmark & 0.151\,\cmark & 0.134\,\cmark \\
& SIREN
& \textbf{0.291}\,\cmark & \textbf{0.174}\,\cmark & \textbf{0.213}\,\cmark & \textbf{0.144}\,\cmark & \textbf{0.147}\,\cmark & \textbf{0.202}\,\cmark & \textbf{0.241}\,\cmark & \textbf{0.238}\,\cmark & \textbf{0.187}\,\cmark & \textbf{0.125}\,\cmark & \textbf{0.129}\,\cmark \\

\midrule
\multirow{6}{*}{$1.5M$}
& $\theta_A^\star$
& 0.346 & 0.198 & 0.311 & 0.211 & 0.162 & 0.244 & 0.264 & 0.247 & 0.221 & 0.154 & 0.136 \\
& True dir.
& A{>}B & A{>}B & A{<}B & A{<}B & A{<}B & A{<}B & A{<}B & A{<}B & A{<}B & A{<}B & A{<}B \\
& M1/M2
& 0.373\,\cmark & 0.230\,\cmark & 0.331\,\xmark & 0.233\,\cmark & 0.197\,\cmark & 0.264\,\xmark & 0.292\,\xmark & 0.272\,\xmark & 0.237\,\xmark & 0.189\,\xmark & 0.149\,\cmark \\
& M3
& 0.342\,\cmark & 0.225\,\cmark & 0.327\,\xmark & 0.233\,\cmark & 0.190\,\cmark & \textbf{0.244}\,\cmark & 0.275\,\cmark & 0.259\,\cmark & 0.209\,\cmark & 0.188\,\xmark & 0.129\,\cmark \\
& M4
& 0.359\,\cmark & 0.222\,\cmark & 0.329\,\xmark & 0.227\,\cmark & 0.191\,\cmark & 0.256\,\xmark & 0.289\,\xmark & 0.265\,\xmark & 0.227\,\cmark & 0.191\,\xmark & \textbf{0.137}\,\cmark \\
& SIREN
& \textbf{0.343}\,\cmark & \textbf{0.199}\,\cmark & \textbf{0.312}\,\cmark & \textbf{0.213}\,\cmark & \textbf{0.163}\,\cmark & 0.245\,\xmark & \textbf{0.265}\,\cmark & \textbf{0.246}\,\cmark & \textbf{0.218}\,\cmark & \textbf{0.155}\,\cmark & \textbf{0.135}\,\cmark \\

\midrule
\multirow{6}{*}{$3M$}
& $\theta_A^\star$
& 0.345 & 0.196 & 0.313 & 0.212 & 0.164 & 0.244 & 0.263 & 0.239 & 0.218 & 0.155 & 0.137 \\
& True dir.
& A{>}B & A{>}B & A{<}B & A{<}B & A{<}B & A{<}B & A{<}B & A{<}B & A{<}B & A{<}B & A{<}B \\
& M1/M2
& 0.373\,\cmark & 0.237\,\cmark & 0.329\,\xmark & 0.226\,\cmark & 0.201\,\cmark & 0.264\,\xmark & 0.289\,\xmark & 0.267\,\xmark & 0.241\,\xmark & 0.171\,\xmark & 0.162\,\xmark \\
& M3
& \textbf{0.342}\,\cmark & 0.220\,\cmark & 0.323\,\cmark & \textbf{0.214}\,\cmark & 0.176\,\cmark & \textbf{0.244}\,\cmark & 0.270\,\cmark & 0.214\,\cmark & 0.227\,\cmark & 0.169\,\xmark & 0.158\,\xmark \\
& M4
& 0.359\,\cmark & 0.238\,\cmark & 0.317\,\cmark & 0.217\,\cmark & 0.189\,\cmark & 0.255\,\xmark & 0.273\,\cmark & 0.251\,\cmark & 0.230\,\cmark & 0.163\,\xmark & 0.154\,\xmark \\
& SIREN
& \textbf{0.342}\,\cmark & \textbf{0.196}\,\cmark & \textbf{0.313}\,\cmark & 0.215\,\cmark & \textbf{0.165}\,\cmark & 0.246\,\xmark & \textbf{0.265}\,\cmark & \textbf{0.238}\,\cmark & \textbf{0.216}\,\cmark & \textbf{0.156}\,\cmark & \textbf{0.136}\,\cmark \\

\midrule
\multirow{6}{*}{$6.5M$}
& $\theta_A^\star$
& 0.353 & 0.200 & 0.319 & 0.213 & 0.187 & 0.239 & 0.257 & 0.248 & 0.227 & 0.153 & 0.146 \\
& True dir.
& A{>}B & A{>}B & A{<}B & A{<}B & A{<}B & A{<}B & A{<}B & A{<}B & A{<}B & A{<}B & A{<}B \\
& M1/M2
& 0.373\,\cmark & 0.237\,\cmark & 0.322\,\cmark & 0.233\,\cmark & 0.219\,\cmark & 0.264\,\xmark & 0.280\,\cmark & 0.269\,\xmark & 0.241\,\xmark & 0.171\,\xmark & 0.163\,\xmark \\
& M3
& 0.343\,\cmark & 0.220\,\cmark & 0.297\,\cmark & 0.219\,\cmark & 0.206\,\cmark & 0.244\,\cmark & 0.251\,\cmark & \textbf{0.248}\,\cmark & 0.214\,\cmark & 0.171\,\xmark & 0.155\,\xmark \\
& M4
& 0.359\,\cmark & 0.235\,\cmark & 0.311\,\cmark & 0.220\,\cmark & 0.208\,\cmark & 0.261\,\xmark & 0.269\,\cmark & 0.260\,\cmark & 0.232\,\xmark & 0.161\,\cmark & 0.157\,\xmark \\
& SIREN
& \textbf{0.350}\,\cmark & \textbf{0.201}\,\cmark & \textbf{0.320}\,\cmark & \textbf{0.214}\,\cmark & \textbf{0.189}\,\cmark & \textbf{0.243}\,\cmark & \textbf{0.255}\,\cmark & 0.246\,\cmark & \textbf{0.225}\,\cmark & \textbf{0.153}\,\cmark & \textbf{0.146}\,\cmark \\

\bottomrule
\end{tabular}}
\end{table*}

\FloatBarrier


\begin{table*}[!htbp]
\centering
\caption{
Unified per-model comparison across budgets on MMLU-Pro Law with the \textbf{random search} tuner.
The untuned-pipeline-default reference $\theta_B^\star$ is fixed across budgets and shown once at the top.
For each budget, $\theta_A^\star$ is the common soft-rule reference, and the method rows report scores $\hat\theta_A$ compared with it. The \cmark/\xmark marker indicates whether
$\mathrm{sign}(\hat\theta_A-\theta_B^\star)$ agrees with the true direction
$\mathrm{sign}(\theta_A^\star-\theta_B^\star)$. M1 and M2 are merged because
they are numerically identical in point estimate and directional call.
}
\label{tab:siren_law_rs_all_budgets}
\scriptsize
\setlength{\tabcolsep}{2.5pt}
\renewcommand{\arraystretch}{0.92}
\resizebox{\textwidth}{!}{%
\begin{tabular}{llccccccccccc}
\toprule
$B$ & Row
& Qwen3 & Phi-3.5 & Qwen2.5-7B & Llama3.1 & Yi-1.5
& InternLM & Qwen2-7B & Qwen2.5-3B & GLM-4 & Mistr-v0.3 & Mistr-v0.1 \\
\midrule

Ref. & $\theta_B^\star$
& 0.359 & 0.318 & 0.345 & 0.315 & 0.227 & 0.269 & 0.311 & 0.234 & 0.282 & 0.225 & 0.185 \\

\midrule
\multirow{6}{*}{$500K$}
& $\theta_A^\star$
& 0.317 & 0.293 & 0.324 & 0.294 & 0.222 & 0.261 & 0.300 & 0.224 & 0.272 & 0.230 & 0.172 \\
& True dir.
& A{<}B & A{<}B & A{<}B & A{<}B & A{<}B & A{<}B & A{<}B & A{<}B & A{<}B & A{>}B & A{<}B \\
& M1/M2
& 0.364\,\xmark & 0.312\,\cmark & 0.330\,\cmark & 0.314\,\cmark & 0.227\,\xmark & 0.275\,\xmark & 0.315\,\xmark & 0.237\,\xmark & 0.285\,\xmark & 0.234\,\cmark & 0.185\,\xmark \\
& M3
& 0.377\,\xmark & 0.325\,\xmark & 0.335\,\cmark & 0.317\,\xmark & 0.238\,\xmark & 0.281\,\xmark & 0.309\,\cmark & 0.236\,\xmark & 0.283\,\xmark & 0.246\,\cmark & 0.196\,\xmark \\
& M4
& 0.359\,\xmark & 0.318\,\cmark & 0.330\,\cmark & 0.305\,\cmark & 0.220\,\cmark & 0.272\,\xmark & 0.303\,\cmark & \textbf{0.226}\,\cmark & \textbf{0.271}\,\cmark & \textbf{0.229}\,\cmark & 0.186\,\xmark \\
& SIREN
& \textbf{0.317}\,\cmark & \textbf{0.300}\,\cmark & \textbf{0.328}\,\cmark & \textbf{0.291}\,\cmark & \textbf{0.223}\,\cmark & \textbf{0.262}\,\cmark & \textbf{0.298}\,\cmark & 0.220\,\cmark & 0.267\,\cmark & 0.234\,\cmark & \textbf{0.176}\,\cmark \\

\midrule
\multirow{6}{*}{$1.5M$}
& $\theta_A^\star$
& 0.308 & 0.298 & 0.324 & 0.283 & 0.223 & 0.265 & 0.304 & 0.235 & 0.274 & 0.217 & 0.165 \\
& True dir.
& A{<}B & A{<}B & A{<}B & A{<}B & A{<}B & A{<}B & A{<}B & A{>}B & A{<}B & A{<}B & A{<}B \\
& M1/M2
& 0.364\,\xmark & 0.318\,\xmark & 0.338\,\cmark & 0.325\,\xmark & 0.251\,\xmark & 0.275\,\xmark & 0.320\,\xmark & 0.244\,\cmark & 0.285\,\xmark & 0.240\,\xmark & 0.185\,\xmark \\
& M3
& 0.377\,\xmark & 0.311\,\cmark & 0.329\,\cmark & 0.339\,\xmark & 0.259\,\xmark & 0.271\,\xmark & 0.309\,\cmark & \textbf{0.232}\,\xmark & 0.253\,\cmark & 0.248\,\xmark & 0.140\,\cmark \\
& M4
& 0.359\,\xmark & 0.318\,\xmark & 0.329\,\cmark & 0.323\,\xmark & 0.245\,\xmark & 0.263\,\cmark & \textbf{0.304}\,\cmark & 0.230\,\xmark & 0.264\,\cmark & 0.234\,\xmark & 0.176\,\cmark \\
& SIREN
& \textbf{0.309}\,\cmark & \textbf{0.305}\,\cmark & \textbf{0.326}\,\cmark & \textbf{0.283}\,\cmark & \textbf{0.226}\,\cmark & \textbf{0.265}\,\cmark & 0.303\,\cmark & \textbf{0.232}\,\xmark & \textbf{0.270}\,\cmark & \textbf{0.220}\,\cmark & \textbf{0.168}\,\cmark \\

\midrule
\multirow{6}{*}{$3M$}
& $\theta_A^\star$
& 0.324 & 0.294 & 0.325 & 0.281 & 0.224 & 0.265 & 0.308 & 0.241 & 0.273 & 0.230 & 0.173 \\
& True dir.
& A{<}B & A{<}B & A{<}B & A{<}B & A{<}B & A{<}B & A{<}B & A{>}B & A{<}B & A{>}B & A{<}B \\
& M1/M2
& 0.364\,\xmark & 0.318\,\xmark & 0.339\,\cmark & 0.327\,\xmark & 0.249\,\xmark & 0.277\,\xmark & 0.328\,\xmark & 0.261\,\cmark & 0.280\,\cmark & 0.246\,\cmark & 0.185\,\xmark \\
& M3
& 0.377\,\xmark & 0.327\,\xmark & 0.317\,\cmark & 0.301\,\cmark & 0.244\,\xmark & 0.257\,\cmark & 0.321\,\xmark & 0.251\,\cmark & 0.253\,\cmark & 0.257\,\cmark & 0.194\,\xmark \\
& M4
& 0.359\,\xmark & 0.321\,\xmark & 0.330\,\cmark & 0.321\,\xmark & 0.236\,\xmark & 0.262\,\cmark & 0.320\,\xmark & 0.246\,\cmark & 0.259\,\cmark & 0.235\,\cmark & 0.181\,\cmark \\
& SIREN
& \textbf{0.326}\,\cmark & \textbf{0.302}\,\cmark & \textbf{0.327}\,\cmark & \textbf{0.279}\,\cmark & \textbf{0.225}\,\cmark & \textbf{0.264}\,\cmark & \textbf{0.308}\,\cmark & \textbf{0.238}\,\cmark & \textbf{0.270}\,\cmark & \textbf{0.232}\,\cmark & \textbf{0.176}\,\cmark \\

\midrule
\multirow{6}{*}{$6.5M$}
& $\theta_A^\star$
& 0.318 & 0.294 & 0.325 & 0.288 & 0.240 & 0.261 & 0.309 & 0.242 & 0.279 & 0.224 & 0.176 \\
& True dir.
& A{<}B & A{<}B & A{<}B & A{<}B & A{>}B & A{<}B & A{<}B & A{>}B & A{<}B & A{<}B & A{<}B \\
& M1/M2
& 0.335\,\cmark & 0.303\,\cmark & 0.342\,\cmark & 0.330\,\xmark & 0.265\,\cmark & 0.284\,\xmark & 0.319\,\xmark & 0.258\,\cmark & 0.292\,\xmark & 0.240\,\xmark & 0.185\,\xmark \\
& M3
& 0.323\,\cmark & 0.287\,\cmark & 0.339\,\cmark & 0.355\,\xmark & 0.267\,\cmark & 0.297\,\xmark & 0.305\,\cmark & 0.251\,\cmark & 0.283\,\xmark & 0.246\,\xmark & 0.190\,\xmark \\
& M4
& 0.328\,\cmark & \textbf{0.298}\,\cmark & 0.338\,\cmark & 0.328\,\xmark & 0.261\,\cmark & 0.284\,\xmark & 0.305\,\cmark & \textbf{0.244}\,\cmark & \textbf{0.278}\,\cmark & 0.234\,\xmark & 0.180\,\cmark \\
& SIREN
& \textbf{0.320}\,\cmark & 0.300\,\cmark & \textbf{0.329}\,\cmark & \textbf{0.288}\,\cmark & \textbf{0.240}\,\cmark & \textbf{0.262}\,\cmark & \textbf{0.308}\,\cmark & 0.239\,\cmark & 0.276\,\cmark & \textbf{0.225}\,\xmark & \textbf{0.179}\,\cmark \\

\bottomrule
\end{tabular}}
\end{table*}

\FloatBarrier


\begin{table*}[!htbp]
\centering
\caption{
Unified per-model comparison across budgets on MMLU-Pro Law with the \textbf{DSPy} tuner.
The untuned-pipeline-default reference $\theta_B^\star$ is fixed across budgets and shown once at the top.
For each budget, $\theta_A^\star$ is the common soft-rule reference, and the method rows report scores $\hat\theta_A$ compared with it. The \cmark/\xmark marker indicates whether
$\mathrm{sign}(\hat\theta_A-\theta_B^\star)$ agrees with the true direction
$\mathrm{sign}(\theta_A^\star-\theta_B^\star)$. M1 and M2 are merged because
they are numerically identical in point estimate and directional call.
}
\label{tab:siren_law_dspy_all_budgets}
\scriptsize
\setlength{\tabcolsep}{2.5pt}
\renewcommand{\arraystretch}{0.92}
\resizebox{\textwidth}{!}{%
\begin{tabular}{llccccccccccc}
\toprule
$B$ & Row
& Qwen3 & Phi-3.5 & Qwen2.5-7B & Llama3.1 & Yi-1.5
& InternLM & Qwen2-7B & Qwen2.5-3B & GLM-4 & Mistr-v0.3 & Mistr-v0.1 \\
\midrule

Ref. & $\theta_B^\star$
& 0.340 & 0.312 & 0.334 & 0.328 & 0.219 & 0.256 & 0.302 & 0.251 & 0.257 & 0.201 & 0.169 \\

\midrule
\multirow{6}{*}{$500K$}
& $\theta_A^\star$
& 0.349 & 0.327 & 0.357 & 0.329 & 0.244 & 0.284 & 0.317 & 0.256 & 0.279 & 0.236 & 0.166 \\
& True dir.
& A{>}B & A{>}B & A{>}B & A{>}B & A{>}B & A{>}B & A{>}B & A{>}B & A{>}B & A{>}B & A{<}B \\
& M1/M2
& 0.360\,\cmark & 0.333\,\cmark & 0.361\,\cmark & 0.339\,\cmark & 0.250\,\cmark & 0.296\,\cmark & 0.323\,\cmark & \textbf{0.261}\,\cmark & 0.286\,\cmark & 0.246\,\cmark & 0.177\,\xmark \\
& M3
& 0.371\,\cmark & 0.335\,\cmark & 0.329\,\xmark & 0.351\,\cmark & \textbf{0.246}\,\cmark & 0.279\,\cmark & 0.321\,\cmark & 0.238\,\xmark & 0.269\,\cmark & 0.228\,\cmark & 0.176\,\xmark \\
& M4
& \textbf{0.350}\,\cmark & 0.332\,\cmark & 0.351\,\cmark & \textbf{0.325}\,\xmark & 0.248\,\cmark & 0.292\,\cmark & 0.312\,\cmark & 0.242\,\xmark & 0.271\,\cmark & \textbf{0.233}\,\cmark & 0.173\,\xmark \\
& SIREN
& 0.346\,\cmark & \textbf{0.331}\,\cmark & \textbf{0.356}\,\cmark & \textbf{0.325}\,\xmark & 0.247\,\cmark & \textbf{0.284}\,\cmark & \textbf{0.314}\,\cmark & 0.247\,\xmark & \textbf{0.273}\,\cmark & 0.231\,\cmark & \textbf{0.166}\,\cmark \\

\midrule
\multirow{6}{*}{$1.5M$}
& $\theta_A^\star$
& 0.353 & 0.331 & 0.359 & 0.323 & 0.243 & 0.284 & 0.323 & 0.258 & 0.285 & 0.237 & 0.172 \\
& True dir.
& A{>}B & A{>}B & A{>}B & A{<}B & A{>}B & A{>}B & A{>}B & A{>}B & A{>}B & A{>}B & A{>}B \\
& M1/M2
& 0.374\,\cmark & 0.341\,\cmark & 0.372\,\cmark & 0.330\,\xmark & 0.248\,\cmark & 0.290\,\cmark & 0.342\,\cmark & 0.269\,\cmark & 0.294\,\cmark & 0.253\,\cmark & 0.180\,\cmark \\
& M3
& 0.387\,\cmark & \textbf{0.335}\,\cmark & 0.363\,\cmark & 0.341\,\xmark & 0.250\,\cmark & 0.289\,\cmark & 0.349\,\cmark & 0.246\,\xmark & 0.277\,\cmark & \textbf{0.234}\,\cmark & 0.176\,\cmark \\
& M4
& 0.366\,\cmark & 0.337\,\cmark & 0.365\,\cmark & \textbf{0.319}\,\cmark & \textbf{0.240}\,\cmark & 0.281\,\cmark & 0.336\,\cmark & \textbf{0.253}\,\cmark & 0.280\,\cmark & 0.247\,\cmark & \textbf{0.171}\,\cmark \\
& SIREN
& \textbf{0.350}\,\cmark & \textbf{0.335}\,\cmark & \textbf{0.356}\,\cmark & 0.318\,\cmark & \textbf{0.246}\,\cmark & \textbf{0.284}\,\cmark & \textbf{0.320}\,\cmark & 0.250\,\xmark & \textbf{0.281}\,\cmark & \textbf{0.234}\,\cmark & 0.174\,\cmark \\

\midrule
\multirow{6}{*}{$3M$}
& $\theta_A^\star$
& 0.345 & 0.323 & 0.356 & 0.326 & 0.248 & 0.284 & 0.325 & 0.258 & 0.279 & 0.238 & 0.175 \\
& True dir.
& A{>}B & A{>}B & A{>}B & A{<}B & A{>}B & A{>}B & A{>}B & A{>}B & A{>}B & A{>}B & A{>}B \\
& M1/M2
& 0.358\,\cmark & 0.339\,\cmark & 0.362\,\cmark & 0.339\,\xmark & 0.258\,\cmark & 0.294\,\cmark & 0.342\,\cmark & 0.269\,\cmark & 0.290\,\cmark & 0.248\,\cmark & 0.187\,\cmark \\
& M3
& 0.379\,\cmark & 0.329\,\cmark & 0.343\,\cmark & 0.339\,\xmark & 0.267\,\cmark & 0.287\,\cmark & 0.349\,\cmark & 0.242\,\xmark & 0.265\,\cmark & 0.255\,\cmark & 0.186\,\cmark \\
& M4
& 0.349\,\cmark & 0.337\,\cmark & 0.351\,\cmark & \textbf{0.327}\,\cmark & \textbf{0.249}\,\cmark & 0.287\,\cmark & 0.334\,\cmark & \textbf{0.254}\,\cmark & \textbf{0.277}\,\cmark & \textbf{0.240}\,\cmark & \textbf{0.178}\,\cmark \\
& SIREN
& \textbf{0.344}\,\cmark & \textbf{0.326}\,\cmark & \textbf{0.352}\,\cmark & 0.321\,\cmark & 0.251\,\cmark & \textbf{0.283}\,\cmark & \textbf{0.322}\,\cmark & 0.251\,\xmark & 0.275\,\cmark & 0.235\,\cmark & \textbf{0.178}\,\cmark \\

\midrule
\multirow{6}{*}{$6.5M$}
& $\theta_A^\star$
& 0.346 & 0.323 & 0.357 & 0.324 & 0.244 & 0.282 & 0.322 & 0.259 & 0.283 & 0.240 & 0.176 \\
& True dir.
& A{>}B & A{>}B & A{>}B & A{<}B & A{>}B & A{>}B & A{>}B & A{>}B & A{>}B & A{>}B & A{>}B \\
& M1/M2
& 0.358\,\cmark & 0.339\,\cmark & 0.362\,\cmark & 0.333\,\xmark & 0.251\,\cmark & 0.294\,\cmark & 0.342\,\cmark & 0.269\,\cmark & 0.291\,\cmark & 0.248\,\cmark & 0.187\,\cmark \\
& M3
& 0.379\,\cmark & 0.329\,\cmark & 0.343\,\cmark & \textbf{0.325}\,\cmark & \textbf{0.246}\,\cmark & 0.287\,\cmark & 0.349\,\cmark & 0.242\,\xmark & 0.265\,\cmark & 0.255\,\cmark & 0.162\,\xmark \\
& M4
& 0.349\,\cmark & 0.337\,\cmark & 0.349\,\cmark & 0.321\,\cmark & 0.240\,\cmark & 0.286\,\cmark & 0.334\,\cmark & 0.250\,\xmark & 0.273\,\cmark & \textbf{0.241}\,\cmark & \textbf{0.176}\,\cmark \\
& SIREN
& \textbf{0.345}\,\cmark & \textbf{0.326}\,\cmark & \textbf{0.353}\,\cmark & 0.320\,\cmark & 0.248\,\cmark & \textbf{0.281}\,\cmark & \textbf{0.320}\,\cmark & \textbf{0.251}\,\cmark & \textbf{0.278}\,\cmark & 0.237\,\cmark & 0.177\,\cmark \\

\bottomrule
\end{tabular}}
\end{table*}

\FloatBarrier

\subsection{Per-model PromptEval-vs-SIREN comparisons}
\label{sec:appendix_pe}

We next report the same per-model breakdown for our external baseline,
PromptEval~\cite{polo2024efficient}, sweeping its cell-observation fraction
$f \in \{0.05, 0.10, 0.20, 0.40, 0.60, 0.80, 1.00\}$.
Tables~\ref{tab:prompteval_math_rs_all_budgets}--\ref{tab:prompteval_law_dspy_all_budgets}
give the per-model point estimates and directional calls for each $(f, B)$
combination across all four (subject, tuner) cells. A consistent picture
emerges: PromptEval's signed error transitions from strongly negative at
small $f$ (because the Rasch model shrinks strong prompts toward the mean
when the score matrix is sparse) through a positive plateau at large $f$
(where dense observation reintroduces the same selection bias as M1),
while SIREN remains within $\pm 0.21$\,pp of the Monte~Carlo target at
every budget across all four cells.

\FloatBarrier


\begin{table*}[!htbp]
\centering
\caption{
Unified PromptEval~\cite{polo2024efficient} vs SIREN comparison across budgets on MMLU-Pro Math with the \textbf{random search} tuner.
The untuned-pipeline-default reference $\theta_B^\star$ is fixed across budgets and shown once at the top.
For each budget, $\theta_A^\star$ is the common soft-rule reference, and the method rows report scores $\hat\theta_A$ compared with it. The \cmark/\xmark marker indicates whether
$\mathrm{sign}(\hat\theta_A-\theta_B^\star)$ agrees with the true direction
$\mathrm{sign}(\theta_A^\star-\theta_B^\star)$.
}
\label{tab:prompteval_math_rs_all_budgets}
\scriptsize
\setlength{\tabcolsep}{2.5pt}
\renewcommand{\arraystretch}{0.92}
\resizebox{\textwidth}{!}{%
\begin{tabular}{llccccccccccc}
\toprule
$B$ & Row
& Qwen3 & Phi-3.5 & Qwen2.5-7B & Llama3.1 & Yi-1.5
& InternLM & Qwen2-7B & Qwen2.5-3B & GLM-4 & Mistr-v0.3 & Mistr-v0.1 \\
\midrule

Ref. & $\theta_B^\star$
& 0.329 & 0.192 & 0.325 & 0.244 & 0.225 & 0.245 & 0.284 & 0.262 & 0.232 & 0.163 & 0.153 \\

\midrule
\multirow{10}{*}{$500K$}
& $\theta_A^\star$
& 0.294 & 0.175 & 0.213 & 0.145 & 0.147 & 0.201 & 0.243 & 0.240 & 0.189 & 0.124 & 0.130 \\
& True dir.
& A{<}B & A{<}B & A{<}B & A{<}B & A{<}B & A{<}B & A{<}B & A{<}B & A{<}B & A{<}B & A{<}B \\
& PE $f{=}0.05$
& 0.260\,\cmark & 0.147\,\cmark & 0.283\,\cmark & 0.140\,\cmark & 0.120\,\cmark & 0.173\,\cmark & 0.194\,\cmark & 0.201\,\cmark & 0.176\,\cmark & 0.097\,\cmark & 0.125\,\cmark \\
& PE $f{=}0.10$
& 0.309\,\cmark & 0.191\,\cmark & 0.261\,\cmark & 0.160\,\cmark & 0.170\,\cmark & 0.219\,\cmark & 0.252\,\cmark & 0.264\,\xmark & 0.212\,\cmark & 0.119\,\cmark & 0.149\,\cmark \\
& PE $f{=}0.20$
& 0.348\,\xmark & 0.241\,\xmark & 0.357\,\xmark & 0.219\,\cmark & 0.221\,\cmark & 0.269\,\xmark & 0.286\,\xmark & 0.298\,\xmark & 0.240\,\xmark & 0.149\,\cmark & 0.170\,\xmark \\
& PE $f{=}0.40$
& 0.360\,\xmark & 0.228\,\xmark & 0.349\,\xmark & 0.218\,\cmark & 0.211\,\cmark & 0.261\,\xmark & 0.283\,\cmark & 0.267\,\xmark & 0.242\,\xmark & 0.154\,\cmark & 0.157\,\xmark \\
& PE $f{=}0.60$
& 0.370\,\xmark & 0.233\,\xmark & 0.332\,\xmark & 0.215\,\cmark & 0.199\,\cmark & 0.258\,\xmark & 0.293\,\xmark & 0.258\,\cmark & 0.234\,\xmark & 0.155\,\cmark & 0.152\,\cmark \\
& PE $f{=}0.80$
& 0.372\,\xmark & 0.232\,\xmark & 0.327\,\xmark & 0.205\,\cmark & 0.198\,\cmark & 0.260\,\xmark & 0.299\,\xmark & 0.253\,\cmark & 0.233\,\xmark & 0.155\,\cmark & 0.149\,\cmark \\
& PE $f{=}1.00$ (M1)
& 0.371\,\xmark & 0.234\,\xmark & 0.324\,\cmark & 0.203\,\cmark & 0.197\,\cmark & 0.259\,\xmark & 0.292\,\xmark & 0.254\,\cmark & 0.233\,\xmark & 0.158\,\cmark & 0.145\,\cmark \\
& SIREN
& \textbf{0.291}\,\cmark & \textbf{0.174}\,\cmark & \textbf{0.213}\,\cmark & \textbf{0.144}\,\cmark & \textbf{0.147}\,\cmark & \textbf{0.202}\,\cmark & \textbf{0.241}\,\cmark & \textbf{0.238}\,\cmark & \textbf{0.187}\,\cmark & \textbf{0.125}\,\cmark & \textbf{0.129}\,\cmark \\

\midrule
\multirow{10}{*}{$1.5M$}
& $\theta_A^\star$
& 0.346 & 0.198 & 0.311 & 0.211 & 0.162 & 0.244 & 0.264 & 0.247 & 0.221 & 0.154 & 0.136 \\
& True dir.
& A{>}B & A{>}B & A{<}B & A{<}B & A{<}B & A{<}B & A{<}B & A{<}B & A{<}B & A{<}B & A{<}B \\
& PE $f{=}0.05$
& 0.296\,\xmark & 0.180\,\xmark & 0.274\,\cmark & 0.163\,\cmark & 0.147\,\cmark & 0.200\,\cmark & 0.204\,\cmark & 0.206\,\cmark & 0.199\,\cmark & 0.122\,\cmark & 0.099\,\cmark \\
& PE $f{=}0.10$
& \textbf{0.349}\,\cmark & 0.216\,\cmark & 0.318\,\cmark & 0.227\,\cmark & 0.184\,\cmark & 0.256\,\xmark & 0.276\,\cmark & 0.264\,\xmark & 0.237\,\xmark & 0.166\,\xmark & 0.149\,\cmark \\
& PE $f{=}0.20$
& 0.396\,\cmark & 0.241\,\cmark & 0.352\,\xmark & 0.238\,\cmark & 0.240\,\xmark & 0.283\,\xmark & 0.290\,\xmark & 0.292\,\xmark & 0.259\,\xmark & 0.199\,\xmark & 0.158\,\xmark \\
& PE $f{=}0.40$
& 0.383\,\cmark & 0.228\,\cmark & 0.328\,\xmark & 0.233\,\cmark & 0.209\,\cmark & 0.286\,\xmark & 0.307\,\xmark & 0.285\,\xmark & 0.244\,\xmark & 0.189\,\xmark & 0.151\,\cmark \\
& PE $f{=}0.60$
& 0.374\,\cmark & 0.229\,\cmark & 0.325\,\xmark & 0.232\,\cmark & 0.204\,\cmark & 0.270\,\xmark & 0.302\,\xmark & 0.274\,\xmark & 0.236\,\xmark & 0.189\,\xmark & 0.148\,\cmark \\
& PE $f{=}0.80$
& 0.373\,\cmark & 0.224\,\cmark & 0.330\,\xmark & 0.234\,\cmark & 0.198\,\cmark & 0.260\,\xmark & 0.292\,\xmark & 0.271\,\xmark & 0.237\,\xmark & 0.184\,\xmark & 0.146\,\cmark \\
& PE $f{=}1.00$ (M1)
& 0.373\,\cmark & 0.230\,\cmark & 0.331\,\xmark & 0.233\,\cmark & 0.197\,\cmark & 0.264\,\xmark & 0.292\,\xmark & 0.272\,\xmark & 0.237\,\xmark & 0.189\,\xmark & 0.149\,\cmark \\
& SIREN
& \textbf{0.343}\,\cmark & \textbf{0.199}\,\cmark & \textbf{0.312}\,\cmark & \textbf{0.213}\,\cmark & \textbf{0.163}\,\cmark & \textbf{0.245}\,\xmark & \textbf{0.265}\,\cmark & \textbf{0.246}\,\cmark & \textbf{0.218}\,\cmark & \textbf{0.155}\,\cmark & \textbf{0.135}\,\cmark \\

\midrule
\multirow{10}{*}{$3M$}
& $\theta_A^\star$
& 0.345 & 0.196 & 0.313 & 0.212 & 0.164 & 0.244 & 0.263 & 0.239 & 0.218 & 0.155 & 0.137 \\
& True dir.
& A{>}B & A{>}B & A{<}B & A{<}B & A{<}B & A{<}B & A{<}B & A{<}B & A{<}B & A{<}B & A{<}B \\
& PE $f{=}0.05$
& 0.366\,\cmark & 0.144\,\xmark & 0.275\,\cmark & 0.159\,\cmark & 0.104\,\cmark & 0.213\,\cmark & 0.208\,\cmark & 0.202\,\cmark & 0.187\,\cmark & 0.117\,\cmark & 0.116\,\cmark \\
& PE $f{=}0.10$
& 0.369\,\cmark & 0.206\,\cmark & 0.323\,\cmark & 0.218\,\cmark & 0.172\,\cmark & 0.275\,\xmark & 0.282\,\cmark & 0.263\,\xmark & 0.236\,\xmark & 0.169\,\xmark & 0.150\,\cmark \\
& PE $f{=}0.20$
& 0.382\,\cmark & 0.250\,\cmark & 0.340\,\xmark & 0.251\,\xmark & 0.212\,\cmark & 0.280\,\xmark & 0.302\,\xmark & 0.282\,\xmark & 0.246\,\xmark & 0.194\,\xmark & 0.179\,\xmark \\
& PE $f{=}0.40$
& 0.397\,\cmark & 0.247\,\cmark & 0.326\,\xmark & 0.230\,\cmark & 0.204\,\cmark & 0.288\,\xmark & 0.292\,\xmark & 0.277\,\xmark & 0.253\,\xmark & 0.188\,\xmark & 0.175\,\xmark \\
& PE $f{=}0.60$
& 0.391\,\cmark & 0.243\,\cmark & 0.330\,\xmark & 0.231\,\cmark & 0.214\,\cmark & 0.265\,\xmark & 0.296\,\xmark & 0.264\,\xmark & 0.249\,\xmark & 0.176\,\xmark & 0.162\,\xmark \\
& PE $f{=}0.80$
& 0.375\,\cmark & 0.234\,\cmark & 0.330\,\xmark & 0.226\,\cmark & 0.204\,\cmark & 0.268\,\xmark & 0.291\,\xmark & 0.263\,\xmark & 0.242\,\xmark & 0.174\,\xmark & 0.164\,\xmark \\
& PE $f{=}1.00$ (M1)
& 0.373\,\cmark & 0.237\,\cmark & 0.329\,\xmark & 0.226\,\cmark & 0.201\,\cmark & 0.264\,\xmark & 0.289\,\xmark & 0.267\,\xmark & 0.241\,\xmark & 0.171\,\xmark & 0.162\,\xmark \\
& SIREN
& \textbf{0.342}\,\cmark & \textbf{0.196}\,\cmark & \textbf{0.313}\,\cmark & \textbf{0.215}\,\cmark & \textbf{0.165}\,\cmark & \textbf{0.246}\,\xmark & \textbf{0.265}\,\cmark & \textbf{0.238}\,\cmark & \textbf{0.216}\,\cmark & \textbf{0.156}\,\cmark & \textbf{0.136}\,\cmark \\

\midrule
\multirow{10}{*}{$6.5M$}
& $\theta_A^\star$
& 0.353 & 0.200 & 0.319 & 0.213 & 0.187 & 0.239 & 0.257 & 0.248 & 0.227 & 0.153 & 0.146 \\
& True dir.
& A{>}B & A{>}B & A{<}B & A{<}B & A{<}B & A{<}B & A{<}B & A{<}B & A{<}B & A{<}B & A{<}B \\
& PE $f{=}0.05$
& 0.287\,\xmark & 0.179\,\xmark & 0.274\,\cmark & 0.172\,\cmark & 0.138\,\cmark & 0.199\,\cmark & 0.200\,\cmark & 0.193\,\cmark & 0.200\,\cmark & 0.119\,\cmark & 0.117\,\cmark \\
& PE $f{=}0.10$
& 0.365\,\cmark & 0.221\,\cmark & 0.325\,\xmark & 0.225\,\cmark & 0.206\,\cmark & 0.262\,\xmark & 0.277\,\cmark & 0.252\,\cmark & 0.244\,\xmark & 0.165\,\xmark & 0.166\,\xmark \\
& PE $f{=}0.20$
& 0.400\,\cmark & 0.245\,\cmark & 0.329\,\xmark & 0.235\,\cmark & 0.250\,\xmark & 0.277\,\xmark & 0.295\,\xmark & 0.277\,\xmark & 0.249\,\xmark & 0.183\,\xmark & 0.168\,\xmark \\
& PE $f{=}0.40$
& 0.377\,\cmark & 0.251\,\cmark & 0.321\,\cmark & 0.236\,\cmark & 0.230\,\xmark & 0.283\,\xmark & 0.288\,\xmark & 0.276\,\xmark & 0.245\,\xmark & 0.177\,\xmark & 0.169\,\xmark \\
& PE $f{=}0.60$
& 0.372\,\cmark & 0.246\,\cmark & 0.327\,\xmark & 0.231\,\cmark & 0.237\,\xmark & 0.266\,\xmark & 0.276\,\cmark & 0.264\,\xmark & 0.243\,\xmark & 0.171\,\xmark & 0.164\,\xmark \\
& PE $f{=}0.80$
& 0.373\,\cmark & 0.239\,\cmark & 0.328\,\xmark & 0.236\,\cmark & 0.228\,\xmark & 0.273\,\xmark & 0.283\,\cmark & 0.270\,\xmark & 0.243\,\xmark & 0.174\,\xmark & 0.160\,\xmark \\
& PE $f{=}1.00$ (M1)
& 0.373\,\cmark & 0.237\,\cmark & 0.322\,\cmark & 0.233\,\cmark & 0.219\,\cmark & 0.264\,\xmark & 0.280\,\cmark & 0.269\,\xmark & 0.241\,\xmark & 0.171\,\xmark & 0.163\,\xmark \\
& SIREN
& \textbf{0.350}\,\cmark & \textbf{0.201}\,\cmark & \textbf{0.320}\,\cmark & \textbf{0.214}\,\cmark & \textbf{0.189}\,\cmark & \textbf{0.243}\,\cmark & \textbf{0.255}\,\cmark & \textbf{0.246}\,\cmark & \textbf{0.225}\,\cmark & \textbf{0.153}\,\cmark & \textbf{0.146}\,\cmark \\

\bottomrule
\end{tabular}}
\end{table*}

\FloatBarrier

\begin{table*}[!htbp]
\centering
\caption{
Unified PromptEval~\cite{polo2024efficient} vs SIREN comparison across budgets on MMLU-Pro Law with the \textbf{random search} tuner.
The untuned-pipeline-default reference $\theta_B^\star$ is fixed across budgets and shown once at the top.
For each budget, $\theta_A^\star$ is the common soft-rule reference, and the method rows report scores $\hat\theta_A$ compared with it. The \cmark/\xmark marker indicates whether
$\mathrm{sign}(\hat\theta_A-\theta_B^\star)$ agrees with the true direction
$\mathrm{sign}(\theta_A^\star-\theta_B^\star)$.
}
\label{tab:prompteval_law_rs_all_budgets}
\scriptsize
\setlength{\tabcolsep}{2.5pt}
\renewcommand{\arraystretch}{0.92}
\resizebox{\textwidth}{!}{%
\begin{tabular}{llccccccccccc}
\toprule
$B$ & Row
& Qwen3 & Phi-3.5 & Qwen2.5-7B & Llama3.1 & Yi-1.5
& InternLM & Qwen2-7B & Qwen2.5-3B & GLM-4 & Mistr-v0.3 & Mistr-v0.1 \\
\midrule

Ref. & $\theta_B^\star$
& 0.359 & 0.318 & 0.345 & 0.315 & 0.227 & 0.269 & 0.311 & 0.234 & 0.282 & 0.225 & 0.185 \\

\midrule
\multirow{10}{*}{$500K$}
& $\theta_A^\star$
& 0.317 & 0.293 & 0.324 & 0.294 & 0.222 & 0.261 & 0.300 & 0.224 & 0.272 & 0.230 & 0.172 \\
& True dir.
& A{<}B & A{<}B & A{<}B & A{<}B & A{<}B & A{<}B & A{<}B & A{<}B & A{<}B & A{>}B & A{<}B \\
& PE $f{=}0.05$
& 0.248\,\cmark & 0.162\,\cmark & 0.157\,\cmark & 0.141\,\cmark & 0.119\,\cmark & 0.172\,\cmark & 0.137\,\cmark & 0.087\,\cmark & 0.163\,\cmark & 0.121\,\xmark & 0.128\,\cmark \\
& PE $f{=}0.10$
& 0.254\,\cmark & 0.169\,\cmark & 0.190\,\cmark & 0.192\,\cmark & 0.146\,\cmark & 0.209\,\cmark & 0.178\,\cmark & 0.112\,\cmark & 0.225\,\cmark & 0.122\,\xmark & 0.123\,\cmark \\
& PE $f{=}0.20$
& 0.311\,\cmark & 0.227\,\cmark & 0.244\,\cmark & 0.234\,\cmark & 0.177\,\cmark & 0.283\,\xmark & 0.204\,\cmark & 0.164\,\cmark & 0.280\,\cmark & 0.158\,\xmark & 0.154\,\cmark \\
& PE $f{=}0.40$
& 0.362\,\xmark & \textbf{0.299}\,\cmark & 0.339\,\cmark & 0.318\,\xmark & 0.243\,\xmark & 0.275\,\xmark & 0.302\,\cmark & 0.235\,\xmark & 0.291\,\xmark & 0.235\,\cmark & 0.196\,\xmark \\
& PE $f{=}0.60$
& 0.360\,\xmark & 0.301\,\cmark & 0.335\,\cmark & 0.323\,\xmark & 0.229\,\xmark & 0.285\,\xmark & \textbf{0.301}\,\cmark & 0.230\,\cmark & 0.282\,\xmark & 0.237\,\cmark & 0.197\,\xmark \\
& PE $f{=}0.80$
& 0.357\,\cmark & 0.311\,\cmark & 0.335\,\cmark & 0.321\,\xmark & 0.228\,\xmark & 0.284\,\xmark & 0.313\,\xmark & 0.236\,\xmark & 0.278\,\cmark & 0.236\,\cmark & 0.192\,\xmark \\
& PE $f{=}1.00$ (M1)
& 0.364\,\xmark & 0.312\,\cmark & 0.330\,\cmark & 0.314\,\cmark & 0.227\,\xmark & 0.275\,\xmark & 0.315\,\xmark & 0.237\,\xmark & 0.285\,\xmark & \textbf{0.234}\,\cmark & 0.185\,\xmark \\
& SIREN
& \textbf{0.317}\,\cmark & 0.300\,\cmark & \textbf{0.328}\,\cmark & \textbf{0.291}\,\cmark & \textbf{0.223}\,\cmark & \textbf{0.262}\,\cmark & 0.298\,\cmark & \textbf{0.220}\,\cmark & \textbf{0.267}\,\cmark & \textbf{0.234}\,\cmark & \textbf{0.176}\,\cmark \\

\midrule
\multirow{10}{*}{$1.5M$}
& $\theta_A^\star$
& 0.308 & 0.298 & 0.324 & 0.283 & 0.223 & 0.265 & 0.304 & 0.235 & 0.274 & 0.217 & 0.165 \\
& True dir.
& A{<}B & A{<}B & A{<}B & A{<}B & A{<}B & A{<}B & A{<}B & A{>}B & A{<}B & A{<}B & A{<}B \\
& PE $f{=}0.05$
& 0.236\,\cmark & 0.257\,\cmark & 0.243\,\cmark & 0.303\,\cmark & 0.187\,\cmark & 0.225\,\cmark & 0.265\,\cmark & 0.206\,\xmark & 0.265\,\cmark & 0.191\,\cmark & 0.128\,\cmark \\
& PE $f{=}0.10$
& 0.325\,\cmark & \textbf{0.300}\,\cmark & 0.320\,\cmark & 0.305\,\cmark & \textbf{0.225}\,\cmark & 0.277\,\xmark & 0.334\,\xmark & 0.256\,\cmark & 0.299\,\xmark & 0.241\,\xmark & 0.174\,\cmark \\
& PE $f{=}0.20$
& 0.345\,\cmark & 0.322\,\xmark & 0.352\,\xmark & 0.345\,\xmark & 0.246\,\xmark & 0.287\,\xmark & 0.354\,\xmark & 0.274\,\cmark & 0.305\,\xmark & 0.270\,\xmark & 0.207\,\xmark \\
& PE $f{=}0.40$
& 0.353\,\cmark & 0.318\,\cmark & 0.351\,\xmark & 0.345\,\xmark & 0.247\,\xmark & 0.279\,\xmark & 0.338\,\xmark & 0.253\,\cmark & 0.293\,\xmark & 0.253\,\xmark & 0.197\,\xmark \\
& PE $f{=}0.60$
& 0.365\,\xmark & 0.317\,\cmark & 0.341\,\cmark & 0.330\,\xmark & 0.238\,\xmark & 0.278\,\xmark & 0.325\,\xmark & 0.248\,\cmark & 0.289\,\xmark & 0.248\,\xmark & 0.191\,\xmark \\
& PE $f{=}0.80$
& 0.362\,\xmark & 0.322\,\xmark & 0.340\,\cmark & 0.323\,\xmark & 0.252\,\xmark & 0.279\,\xmark & 0.321\,\xmark & 0.243\,\cmark & 0.284\,\xmark & 0.243\,\xmark & 0.190\,\xmark \\
& PE $f{=}1.00$ (M1)
& 0.364\,\xmark & 0.318\,\xmark & 0.338\,\cmark & 0.325\,\xmark & 0.251\,\xmark & 0.275\,\xmark & 0.320\,\xmark & 0.244\,\cmark & 0.285\,\xmark & 0.240\,\xmark & 0.185\,\xmark \\
& SIREN
& \textbf{0.309}\,\cmark & 0.305\,\cmark & \textbf{0.326}\,\cmark & \textbf{0.283}\,\cmark & 0.226\,\cmark & \textbf{0.265}\,\cmark & \textbf{0.303}\,\cmark & \textbf{0.232}\,\xmark & \textbf{0.270}\,\cmark & \textbf{0.220}\,\cmark & \textbf{0.168}\,\cmark \\

\midrule
\multirow{10}{*}{$3M$}
& $\theta_A^\star$
& 0.324 & 0.294 & 0.325 & 0.281 & 0.224 & 0.265 & 0.308 & 0.241 & 0.273 & 0.230 & 0.173 \\
& True dir.
& A{<}B & A{<}B & A{<}B & A{<}B & A{<}B & A{<}B & A{<}B & A{>}B & A{<}B & A{>}B & A{<}B \\
& PE $f{=}0.05$
& 0.314\,\cmark & 0.284\,\cmark & 0.344\,\cmark & 0.296\,\cmark & 0.218\,\cmark & 0.247\,\cmark & 0.277\,\cmark & 0.274\,\cmark & 0.254\,\cmark & 0.204\,\xmark & 0.145\,\cmark \\
& PE $f{=}0.10$
& 0.350\,\cmark & 0.308\,\cmark & 0.351\,\xmark & 0.329\,\xmark & 0.261\,\xmark & 0.282\,\xmark & 0.327\,\xmark & 0.274\,\cmark & 0.296\,\xmark & 0.242\,\cmark & 0.189\,\xmark \\
& PE $f{=}0.20$
& 0.372\,\xmark & 0.333\,\xmark & 0.360\,\xmark & 0.326\,\xmark & 0.267\,\xmark & 0.297\,\xmark & 0.329\,\xmark & 0.279\,\cmark & 0.308\,\xmark & 0.258\,\cmark & 0.199\,\xmark \\
& PE $f{=}0.40$
& 0.362\,\xmark & 0.315\,\cmark & 0.363\,\xmark & 0.338\,\xmark & 0.251\,\xmark & 0.276\,\xmark & 0.333\,\xmark & 0.275\,\cmark & 0.306\,\xmark & 0.258\,\cmark & 0.187\,\xmark \\
& PE $f{=}0.60$
& 0.367\,\xmark & 0.312\,\cmark & 0.350\,\xmark & 0.342\,\xmark & 0.248\,\xmark & 0.283\,\xmark & 0.327\,\xmark & 0.273\,\cmark & 0.281\,\cmark & 0.252\,\cmark & 0.186\,\xmark \\
& PE $f{=}0.80$
& 0.358\,\cmark & 0.317\,\cmark & 0.346\,\xmark & 0.331\,\xmark & 0.250\,\xmark & 0.280\,\xmark & 0.330\,\xmark & 0.269\,\cmark & 0.283\,\xmark & 0.248\,\cmark & 0.185\,\xmark \\
& PE $f{=}1.00$ (M1)
& 0.364\,\xmark & 0.318\,\xmark & 0.339\,\cmark & 0.327\,\xmark & 0.249\,\xmark & 0.277\,\xmark & 0.328\,\xmark & 0.261\,\cmark & 0.280\,\cmark & 0.246\,\cmark & 0.185\,\xmark \\
& SIREN
& \textbf{0.326}\,\cmark & \textbf{0.302}\,\cmark & \textbf{0.327}\,\cmark & \textbf{0.279}\,\cmark & \textbf{0.225}\,\cmark & \textbf{0.264}\,\cmark & \textbf{0.308}\,\cmark & \textbf{0.238}\,\cmark & \textbf{0.270}\,\cmark & \textbf{0.232}\,\cmark & \textbf{0.176}\,\cmark \\

\midrule
\multirow{10}{*}{$6.5M$}
& $\theta_A^\star$
& 0.318 & 0.294 & 0.325 & 0.288 & 0.240 & 0.261 & 0.309 & 0.242 & 0.279 & 0.224 & 0.176 \\
& True dir.
& A{<}B & A{<}B & A{<}B & A{<}B & A{>}B & A{<}B & A{<}B & A{>}B & A{<}B & A{<}B & A{<}B \\
& PE $f{=}0.05$
& 0.250\,\cmark & 0.248\,\cmark & 0.287\,\cmark & 0.271\,\cmark & 0.155\,\xmark & 0.226\,\cmark & 0.288\,\cmark & 0.231\,\xmark & 0.234\,\cmark & 0.209\,\cmark & 0.170\,\cmark \\
& PE $f{=}0.10$
& 0.341\,\cmark & 0.306\,\cmark & 0.344\,\cmark & 0.318\,\xmark & 0.223\,\xmark & 0.277\,\xmark & 0.331\,\xmark & 0.269\,\cmark & 0.287\,\xmark & 0.245\,\xmark & 0.213\,\xmark \\
& PE $f{=}0.20$
& 0.355\,\cmark & 0.330\,\xmark & 0.347\,\xmark & 0.353\,\xmark & 0.269\,\cmark & 0.298\,\xmark & 0.358\,\xmark & 0.285\,\cmark & 0.309\,\xmark & 0.255\,\xmark & 0.215\,\xmark \\
& PE $f{=}0.40$
& 0.354\,\cmark & 0.311\,\cmark & 0.346\,\xmark & 0.338\,\xmark & 0.269\,\cmark & 0.279\,\xmark & 0.330\,\xmark & 0.272\,\cmark & 0.317\,\xmark & 0.249\,\xmark & 0.204\,\xmark \\
& PE $f{=}0.60$
& 0.344\,\cmark & 0.311\,\cmark & 0.344\,\cmark & 0.331\,\xmark & 0.269\,\cmark & 0.281\,\xmark & 0.326\,\xmark & 0.262\,\cmark & 0.299\,\xmark & 0.243\,\xmark & 0.194\,\xmark \\
& PE $f{=}0.80$
& 0.337\,\cmark & 0.306\,\cmark & 0.344\,\cmark & 0.324\,\xmark & 0.260\,\cmark & 0.284\,\xmark & 0.322\,\xmark & 0.264\,\cmark & 0.292\,\xmark & 0.243\,\xmark & 0.187\,\xmark \\
& PE $f{=}1.00$ (M1)
& 0.335\,\cmark & 0.303\,\cmark & 0.342\,\cmark & 0.330\,\xmark & 0.265\,\cmark & 0.284\,\xmark & 0.319\,\xmark & 0.258\,\cmark & 0.292\,\xmark & 0.240\,\xmark & 0.185\,\xmark \\
& SIREN
& \textbf{0.320}\,\cmark & \textbf{0.300}\,\cmark & \textbf{0.329}\,\cmark & \textbf{0.288}\,\cmark & \textbf{0.240}\,\cmark & \textbf{0.262}\,\cmark & \textbf{0.308}\,\cmark & \textbf{0.239}\,\cmark & \textbf{0.276}\,\cmark & \textbf{0.225}\,\xmark & \textbf{0.179}\,\cmark \\

\bottomrule
\end{tabular}}
\end{table*}

\FloatBarrier

\begin{table*}[!htbp]
\centering
\caption{
Unified PromptEval~\cite{polo2024efficient} vs SIREN comparison across budgets on MMLU-Pro Law with the \textbf{DSPy} tuner.
The untuned-pipeline-default reference $\theta_B^\star$ is fixed across budgets and shown once at the top.
For each budget, $\theta_A^\star$ is the common soft-rule reference, and the method rows report scores $\hat\theta_A$ compared with it. The \cmark/\xmark marker indicates whether
$\mathrm{sign}(\hat\theta_A-\theta_B^\star)$ agrees with the true direction
$\mathrm{sign}(\theta_A^\star-\theta_B^\star)$.
}
\label{tab:prompteval_law_dspy_all_budgets}
\scriptsize
\setlength{\tabcolsep}{2.5pt}
\renewcommand{\arraystretch}{0.92}
\resizebox{\textwidth}{!}{%
\begin{tabular}{llccccccccccc}
\toprule
$B$ & Row
& Qwen3 & Phi-3.5 & Qwen2.5-7B & Llama3.1 & Yi-1.5
& InternLM & Qwen2-7B & Qwen2.5-3B & GLM-4 & Mistr-v0.3 & Mistr-v0.1 \\
\midrule

Ref. & $\theta_B^\star$
& 0.340 & 0.312 & 0.334 & 0.328 & 0.219 & 0.256 & 0.302 & 0.251 & 0.257 & 0.201 & 0.169 \\

\midrule
\multirow{10}{*}{$500K$}
& $\theta_A^\star$
& 0.349 & 0.327 & 0.357 & 0.329 & 0.244 & 0.284 & 0.317 & 0.256 & 0.279 & 0.236 & 0.166 \\
& True dir.
& A{>}B & A{>}B & A{>}B & A{>}B & A{>}B & A{>}B & A{>}B & A{>}B & A{>}B & A{>}B & A{<}B \\
& PE $f{=}0.05$
& 0.237\,\xmark & 0.179\,\xmark & 0.207\,\xmark & 0.168\,\xmark & 0.106\,\xmark & 0.139\,\xmark & 0.153\,\xmark & 0.153\,\xmark & 0.174\,\xmark & 0.122\,\xmark & 0.108\,\cmark \\
& PE $f{=}0.10$
& 0.223\,\xmark & 0.197\,\xmark & 0.269\,\xmark & 0.218\,\xmark & 0.145\,\xmark & 0.202\,\xmark & 0.211\,\xmark & 0.195\,\xmark & 0.213\,\xmark & 0.152\,\xmark & 0.114\,\cmark \\
& PE $f{=}0.20$
& 0.314\,\xmark & 0.259\,\xmark & 0.365\,\cmark & 0.317\,\xmark & 0.182\,\xmark & 0.261\,\cmark & 0.285\,\xmark & 0.240\,\xmark & 0.295\,\cmark & 0.215\,\cmark & 0.156\,\cmark \\
& PE $f{=}0.40$
& 0.376\,\cmark & 0.337\,\cmark & 0.372\,\cmark & 0.347\,\cmark & 0.252\,\cmark & 0.296\,\cmark & 0.331\,\cmark & \textbf{0.259}\,\cmark & 0.293\,\cmark & 0.249\,\cmark & 0.197\,\xmark \\
& PE $f{=}0.60$
& 0.365\,\cmark & \textbf{0.330}\,\cmark & 0.361\,\cmark & 0.342\,\cmark & 0.252\,\cmark & 0.287\,\cmark & 0.323\,\cmark & 0.268\,\cmark & 0.289\,\cmark & 0.244\,\cmark & 0.189\,\xmark \\
& PE $f{=}0.80$
& 0.364\,\cmark & 0.332\,\cmark & 0.362\,\cmark & 0.338\,\cmark & 0.249\,\cmark & 0.289\,\cmark & 0.321\,\cmark & \textbf{0.259}\,\cmark & 0.289\,\cmark & 0.248\,\cmark & 0.176\,\xmark \\
& PE $f{=}1.00$ (M1)
& 0.360\,\cmark & 0.333\,\cmark & 0.361\,\cmark & 0.339\,\cmark & 0.250\,\cmark & 0.296\,\cmark & 0.323\,\cmark & 0.261\,\cmark & 0.286\,\cmark & 0.246\,\cmark & 0.177\,\xmark \\
& SIREN
& \textbf{0.346}\,\cmark & 0.331\,\cmark & \textbf{0.356}\,\cmark & \textbf{0.325}\,\xmark & \textbf{0.247}\,\cmark & \textbf{0.284}\,\cmark & \textbf{0.314}\,\cmark & 0.247\,\xmark & \textbf{0.273}\,\cmark & \textbf{0.231}\,\cmark & \textbf{0.166}\,\cmark \\

\midrule
\multirow{10}{*}{$1.5M$}
& $\theta_A^\star$
& 0.353 & 0.331 & 0.359 & 0.323 & 0.243 & 0.284 & 0.323 & 0.258 & 0.285 & 0.237 & 0.172 \\
& True dir.
& A{>}B & A{>}B & A{>}B & A{<}B & A{>}B & A{>}B & A{>}B & A{>}B & A{>}B & A{>}B & A{>}B \\
& PE $f{=}0.05$
& 0.261\,\xmark & 0.196\,\xmark & 0.287\,\xmark & 0.300\,\cmark & 0.163\,\xmark & 0.163\,\xmark & 0.210\,\xmark & 0.162\,\xmark & 0.197\,\xmark & 0.181\,\xmark & 0.135\,\xmark \\
& PE $f{=}0.10$
& 0.277\,\xmark & 0.246\,\xmark & 0.314\,\xmark & 0.284\,\cmark & 0.200\,\xmark & 0.219\,\xmark & 0.254\,\xmark & 0.189\,\xmark & 0.243\,\xmark & 0.221\,\cmark & 0.135\,\xmark \\
& PE $f{=}0.20$
& 0.365\,\cmark & 0.340\,\cmark & 0.374\,\cmark & 0.343\,\xmark & 0.276\,\cmark & 0.295\,\cmark & 0.335\,\cmark & 0.279\,\cmark & 0.300\,\cmark & 0.254\,\cmark & 0.194\,\cmark \\
& PE $f{=}0.40$
& 0.378\,\cmark & 0.336\,\cmark & 0.369\,\cmark & 0.340\,\xmark & 0.250\,\cmark & 0.289\,\cmark & 0.330\,\cmark & 0.282\,\cmark & 0.304\,\cmark & 0.259\,\cmark & 0.191\,\cmark \\
& PE $f{=}0.60$
& 0.368\,\cmark & 0.339\,\cmark & 0.371\,\cmark & 0.335\,\xmark & 0.255\,\cmark & 0.294\,\cmark & 0.335\,\cmark & 0.276\,\cmark & 0.297\,\cmark & 0.246\,\cmark & 0.188\,\cmark \\
& PE $f{=}0.80$
& 0.376\,\cmark & 0.336\,\cmark & 0.370\,\cmark & 0.333\,\xmark & 0.250\,\cmark & 0.291\,\cmark & 0.342\,\cmark & 0.274\,\cmark & 0.295\,\cmark & 0.251\,\cmark & 0.188\,\cmark \\
& PE $f{=}1.00$ (M1)
& 0.374\,\cmark & 0.341\,\cmark & 0.372\,\cmark & 0.330\,\xmark & 0.248\,\cmark & 0.290\,\cmark & 0.342\,\cmark & 0.269\,\cmark & 0.294\,\cmark & 0.253\,\cmark & 0.180\,\cmark \\
& SIREN
& \textbf{0.350}\,\cmark & \textbf{0.335}\,\cmark & \textbf{0.356}\,\cmark & \textbf{0.318}\,\cmark & \textbf{0.246}\,\cmark & \textbf{0.284}\,\cmark & \textbf{0.320}\,\cmark & \textbf{0.250}\,\xmark & \textbf{0.281}\,\cmark & \textbf{0.234}\,\cmark & \textbf{0.174}\,\cmark \\

\midrule
\multirow{10}{*}{$3M$}
& $\theta_A^\star$
& 0.345 & 0.323 & 0.356 & 0.326 & 0.248 & 0.284 & 0.325 & 0.258 & 0.279 & 0.238 & 0.175 \\
& True dir.
& A{>}B & A{>}B & A{>}B & A{<}B & A{>}B & A{>}B & A{>}B & A{>}B & A{>}B & A{>}B & A{>}B \\
& PE $f{=}0.05$
& 0.257\,\xmark & 0.192\,\xmark & 0.286\,\xmark & 0.256\,\cmark & 0.152\,\xmark & 0.167\,\xmark & 0.216\,\xmark & 0.196\,\xmark & 0.182\,\xmark & 0.159\,\xmark & 0.149\,\xmark \\
& PE $f{=}0.10$
& 0.284\,\xmark & 0.239\,\xmark & 0.299\,\xmark & 0.272\,\cmark & 0.208\,\xmark & 0.226\,\xmark & 0.252\,\xmark & 0.198\,\xmark & 0.219\,\xmark & 0.192\,\xmark & 0.156\,\xmark \\
& PE $f{=}0.20$
& 0.356\,\cmark & \textbf{0.326}\,\cmark & 0.376\,\cmark & 0.343\,\xmark & 0.281\,\cmark & 0.304\,\cmark & 0.335\,\cmark & 0.270\,\cmark & 0.288\,\cmark & 0.248\,\cmark & 0.205\,\cmark \\
& PE $f{=}0.40$
& 0.363\,\cmark & 0.340\,\cmark & 0.377\,\cmark & 0.337\,\xmark & 0.272\,\cmark & 0.298\,\cmark & 0.337\,\cmark & 0.276\,\cmark & 0.290\,\cmark & 0.247\,\cmark & 0.198\,\cmark \\
& PE $f{=}0.60$
& 0.356\,\cmark & 0.334\,\cmark & 0.374\,\cmark & 0.333\,\xmark & 0.260\,\cmark & 0.295\,\cmark & 0.334\,\cmark & \textbf{0.261}\,\cmark & 0.287\,\cmark & 0.256\,\cmark & 0.188\,\cmark \\
& PE $f{=}0.80$
& 0.356\,\cmark & 0.337\,\cmark & 0.366\,\cmark & 0.339\,\xmark & 0.264\,\cmark & 0.294\,\cmark & 0.339\,\cmark & 0.264\,\cmark & 0.286\,\cmark & 0.251\,\cmark & 0.185\,\cmark \\
& PE $f{=}1.00$ (M1)
& 0.358\,\cmark & 0.339\,\cmark & 0.362\,\cmark & 0.339\,\xmark & 0.258\,\cmark & 0.294\,\cmark & 0.342\,\cmark & 0.269\,\cmark & 0.290\,\cmark & 0.248\,\cmark & 0.187\,\cmark \\
& SIREN
& \textbf{0.344}\,\cmark & \textbf{0.326}\,\cmark & \textbf{0.352}\,\cmark & \textbf{0.321}\,\cmark & \textbf{0.251}\,\cmark & \textbf{0.283}\,\cmark & \textbf{0.322}\,\cmark & 0.251\,\xmark & \textbf{0.275}\,\cmark & \textbf{0.235}\,\cmark & \textbf{0.178}\,\cmark \\

\midrule
\multirow{10}{*}{$6.5M$}
& $\theta_A^\star$
& 0.346 & 0.323 & 0.357 & 0.324 & 0.244 & 0.282 & 0.322 & 0.259 & 0.283 & 0.240 & 0.176 \\
& True dir.
& A{>}B & A{>}B & A{>}B & A{<}B & A{>}B & A{>}B & A{>}B & A{>}B & A{>}B & A{>}B & A{>}B \\
& PE $f{=}0.05$
& 0.220\,\xmark & 0.206\,\xmark & 0.270\,\xmark & 0.260\,\cmark & 0.148\,\xmark & 0.166\,\xmark & 0.233\,\xmark & 0.194\,\xmark & 0.184\,\xmark & 0.184\,\xmark & 0.129\,\xmark \\
& PE $f{=}0.10$
& 0.268\,\xmark & 0.263\,\xmark & 0.305\,\xmark & 0.274\,\cmark & 0.185\,\xmark & 0.215\,\xmark & 0.281\,\xmark & 0.190\,\xmark & 0.223\,\xmark & 0.229\,\cmark & 0.148\,\xmark \\
& PE $f{=}0.20$
& 0.361\,\cmark & 0.354\,\cmark & 0.373\,\cmark & 0.342\,\xmark & 0.258\,\cmark & 0.295\,\cmark & 0.341\,\cmark & 0.277\,\cmark & 0.308\,\cmark & 0.268\,\cmark & 0.207\,\cmark \\
& PE $f{=}0.40$
& 0.363\,\cmark & 0.344\,\cmark & 0.380\,\cmark & 0.358\,\xmark & 0.254\,\cmark & 0.291\,\cmark & 0.351\,\cmark & 0.277\,\cmark & 0.297\,\cmark & 0.252\,\cmark & 0.202\,\cmark \\
& PE $f{=}0.60$
& 0.359\,\cmark & 0.337\,\cmark & 0.373\,\cmark & 0.344\,\xmark & 0.259\,\cmark & 0.291\,\cmark & 0.342\,\cmark & 0.278\,\cmark & 0.298\,\cmark & 0.252\,\cmark & 0.190\,\cmark \\
& PE $f{=}0.80$
& 0.360\,\cmark & 0.338\,\cmark & 0.364\,\cmark & 0.334\,\xmark & 0.260\,\cmark & 0.294\,\cmark & 0.341\,\cmark & 0.280\,\cmark & 0.292\,\cmark & 0.249\,\cmark & 0.184\,\cmark \\
& PE $f{=}1.00$ (M1)
& 0.358\,\cmark & 0.339\,\cmark & 0.362\,\cmark & 0.333\,\xmark & 0.251\,\cmark & 0.294\,\cmark & 0.342\,\cmark & 0.269\,\cmark & 0.291\,\cmark & 0.248\,\cmark & 0.187\,\cmark \\
& SIREN
& \textbf{0.345}\,\cmark & \textbf{0.326}\,\cmark & \textbf{0.353}\,\cmark & \textbf{0.320}\,\cmark & \textbf{0.248}\,\cmark & \textbf{0.281}\,\cmark & \textbf{0.320}\,\cmark & \textbf{0.251}\,\cmark & \textbf{0.278}\,\cmark & \textbf{0.237}\,\cmark & \textbf{0.177}\,\cmark \\

\bottomrule
\end{tabular}}
\end{table*}

\FloatBarrier

\subsubsection{Per-budget directional accuracy summaries}
\label{sec:appendix_pe_per_budget}

The per-budget summary tables below collapse each $(f, B)$ cell to a single
$(\mathrm{dir}, \mathrm{bias})$ pair averaged over the 11 models, making the
$f$-dependent bias trade-off easier to read at a glance.
Table~\ref{tab:prompteval_per_budget_math_dspy} for Math with the DSPy
tuner appears in the main text alongside the discussion of the
PromptEval comparison; the corresponding tables for the other three
(subject, tuner) cells are presented here. The U-shape pattern across
$f$ is consistent across all four cells: SIREN's mean signed error
stays bounded by $0.21$\,pp at every budget, while every PromptEval
fraction $f$ reports a bias of at least $0.74$\,pp in absolute value.

\FloatBarrier


\begin{table}[!htbp]
\centering
\caption{PromptEval~\cite{polo2024efficient} per-budget directional accuracy on \textbf{MMLU-Pro Math} for Random Search.}
\label{tab:prompteval_per_budget_math_rs}
\small
\setlength{\tabcolsep}{4pt}
\renewcommand{\arraystretch}{1.0}
\begin{tabular}{@{}c|cc|cc|cc|cc@{}}
\toprule
& \multicolumn{2}{c|}{500K} & \multicolumn{2}{c|}{1.5M} & \multicolumn{2}{c|}{3M} & \multicolumn{2}{c}{6.5M} \\
$f$ & dir & bias & dir & bias & dir & bias & dir & bias \\
\midrule
$0.05$ & $\mathbf{11/11}$ & $-1.67$ & $9/11$ & $-3.66$ & $\mathbf{10/11}$ & $-3.59$ & $9/11$ & $-4.20$ \\
$0.10$ & $10/11$ & $+1.87$ & $7/11$ & $+1.35$ & $7/11$ & $+1.61$ & $6/11$ & $+1.52$ \\
$0.20$ & $3/11$ & $+6.36$ & $3/11$ & $+4.14$ & $3/11$ & $+3.94$ & $3/11$ & $+3.33$ \\
$0.40$ & $4/11$ & $+5.72$ & $5/11$ & $+3.20$ & $4/11$ & $+3.54$ & $4/11$ & $+2.84$ \\
$0.60$ & $5/11$ & $+5.43$ & $5/11$ & $+2.65$ & $4/11$ & $+3.06$ & $4/11$ & $+2.34$ \\
$0.80$ & $5/11$ & $+5.29$ & $5/11$ & $+2.33$ & $4/11$ & $+2.58$ & $4/11$ & $+2.43$ \\
$1.00$ & $6/11$ & $+5.19$ & $5/11$ & $+2.50$ & $4/11$ & $+2.49$ & $6/11$ & $+2.09$ \\
\midrule
SIREN & $\mathbf{11/11}$ & $\mathbf{-0.09}$ & $\mathbf{10/11}$ & $\mathbf{+0.01}$ & $\mathbf{10/11}$ & $\mathbf{+0.01}$ & $\mathbf{11/11}$ & $\mathbf{+0.02}$ \\
\bottomrule
\end{tabular}
\end{table}

\FloatBarrier

\begin{table}[!htbp]
\centering
\caption{PromptEval~\cite{polo2024efficient} per-budget directional accuracy on \textbf{MMLU-Pro Law} for Random Search.}
\label{tab:prompteval_per_budget_law_rs}
\small
\setlength{\tabcolsep}{4pt}
\renewcommand{\arraystretch}{1.0}
\begin{tabular}{@{}c|cc|cc|cc|cc@{}}
\toprule
& \multicolumn{2}{c|}{500K} & \multicolumn{2}{c|}{1.5M} & \multicolumn{2}{c|}{3M} & \multicolumn{2}{c}{6.5M} \\
$f$ & dir & bias & dir & bias & dir & bias & dir & bias \\
\midrule
$0.05$ & $10/11$ & $-11.58$ & $\mathbf{10/11}$ & $-3.52$ & $10/11$ & $-0.74$ & $9/11$ & $-3.51$ \\
$0.10$ & $10/11$ & $-9.01$ & $7/11$ & $+1.46$ & $4/11$ & $+2.45$ & $4/11$ & $+1.80$ \\
$0.20$ & $9/11$ & $-4.30$ & $2/11$ & $+3.74$ & $2/11$ & $+3.54$ & $3/11$ & $+3.81$ \\
$0.40$ & $4/11$ & $+1.68$ & $3/11$ & $+3.02$ & $3/11$ & $+2.96$ & $4/11$ & $+2.85$ \\
$0.60$ & $5/11$ & $+1.56$ & $3/11$ & $+2.51$ & $4/11$ & $+2.55$ & $5/11$ & $+2.25$ \\
$0.80$ & $5/11$ & $+1.65$ & $2/11$ & $+2.40$ & $4/11$ & $+2.34$ & $5/11$ & $+1.89$ \\
$1.00$ & $4/11$ & $+1.50$ & $2/11$ & $+2.24$ & $4/11$ & $+2.10$ & $5/11$ & $+1.77$ \\
\midrule
SIREN & $\mathbf{11/11}$ & $\mathbf{+0.05}$ & $\mathbf{10/11}$ & $\mathbf{+0.11}$ & $\mathbf{11/11}$ & $\mathbf{+0.07}$ & $\mathbf{10/11}$ & $\mathbf{+0.10}$ \\
\bottomrule
\end{tabular}
\end{table}

\FloatBarrier

\begin{table}[!htbp]
\centering
\caption{PromptEval~\cite{polo2024efficient} per-budget directional accuracy on \textbf{MMLU-Pro Law} for DSPy.}
\label{tab:prompteval_per_budget_law_dspy}
\small
\setlength{\tabcolsep}{4pt}
\renewcommand{\arraystretch}{1.0}
\begin{tabular}{@{}c|cc|cc|cc|cc@{}}
\toprule
& \multicolumn{2}{c|}{500K} & \multicolumn{2}{c|}{1.5M} & \multicolumn{2}{c|}{3M} & \multicolumn{2}{c}{6.5M} \\
$f$ & dir & bias & dir & bias & dir & bias & dir & bias \\
\midrule
$0.05$ & $1/11$ & $-12.72$ & $1/11$ & $-8.30$ & $1/11$ & $-8.60$ & $1/11$ & $-8.76$ \\
$0.10$ & $1/11$ & $-9.14$ & $2/11$ & $-5.34$ & $1/11$ & $-5.56$ & $2/11$ & $-5.22$ \\
$0.20$ & $5/11$ & $-2.31$ & $\mathbf{10/11}$ & $+1.70$ & $\mathbf{10/11}$ & $+1.57$ & $10/11$ & $+2.08$ \\
$0.40$ & $\mathbf{10/11}$ & $+1.49$ & $\mathbf{10/11}$ & $+1.46$ & $\mathbf{10/11}$ & $+1.59$ & $10/11$ & $+1.93$ \\
$0.60$ & $\mathbf{10/11}$ & $+0.96$ & $\mathbf{10/11}$ & $+1.24$ & $\mathbf{10/11}$ & $+1.08$ & $10/11$ & $+1.50$ \\
$0.80$ & $\mathbf{10/11}$ & $+0.74$ & $\mathbf{10/11}$ & $+1.26$ & $\mathbf{10/11}$ & $+1.12$ & $10/11$ & $+1.27$ \\
$1.00$ & $\mathbf{10/11}$ & $+0.77$ & $\mathbf{10/11}$ & $+1.10$ & $\mathbf{10/11}$ & $+1.13$ & $10/11$ & $+1.04$ \\
\midrule
SIREN & $9/11$ & $\mathbf{-0.21}$ & $\mathbf{10/11}$ & $\mathbf{-0.17}$ & $\mathbf{10/11}$ & $\mathbf{-0.20}$ & $\mathbf{11/11}$ & $\mathbf{-0.21}$ \\
\bottomrule
\end{tabular}
\end{table}

\FloatBarrier

\clearpage
\section{Additional empirical checks}
\label{app:rebuttal_checks}

\subsection{Interval centering and the selection contribution}
\label{app:wald_checks}

We first isolate the interval center by assigning all artifacts equal
population quality, so hard and soft selection have the same target.
We use $M=500$, $R=5$, scoring fraction $0.5$, selector temperature
$\tau=0.1$, and $2{,}000$ trials per configuration.
As $K$ increases from $1$ to $50$, M2's signed bias increases from
$-0.10$ to $+4.29$\,pp and its $95\%$ coverage decreases from
$95.7\%$ to $49.3\%$.
At $K=50$, $50.7\%$ of intervals lie entirely above the truth,
with no downward misses; the reported standard error is $1.58$ times
the empirical standard deviation.
This points to the shifted selected-winner center, rather than
insufficient width, as the main problem in this setting.
In the same $K=50$ setting, SIREN-centered Wald and multiplier
intervals using the full item-level contribution achieve
$95.1\%$ and $94.4\%$ coverage, respectively, with nearly identical widths.

We next center every interval at the same SIREN estimate and vary
only the variance calculation.
Across a pre-specified quality-gap sweep, the evaluation-only Wald
ablation omits the scoring-subset effect through the selector.
Its coverage ranges from $91.7\%$ to $95.6\%$, and its reported
standard error reaches $0.88$ of the empirical standard deviation.
The full-contribution Wald and multiplier intervals remain close
in coverage and width (Table~\ref{tab:wald_checks}).
Thus the selection contribution matters for calibration even after
correcting the center.
For one pointwise coordinate, the full-contribution Wald and
Gaussian multiplier intervals behave similarly; the multiplier
construction also supports simultaneous bands and correlated
system-budget comparisons.

\begin{table}[htbp]
\centering
\small
\caption{Coverage over the quality-gap sweep with the same SIREN
point estimate. All intervals have nominal $95\%$ coverage.}
\label{tab:wald_checks}
\begin{tabular}{lc}
\toprule
Variance calculation & Coverage range (\%) \\
\midrule
Evaluation-only Wald & $91.7$--$95.6$ \\
Full-contribution Wald & $94.7$--$95.8$ \\
Full-contribution multiplier & $94.3$--$95.3$ \\
\bottomrule
\end{tabular}
\end{table}

\FloatBarrier

\FloatBarrier

\subsection{Estimation under each method's reporting target}
\label{app:matched_targets}

We reanalyze all $176$ Math/Law cells using a fresh-data target for each method's hard or soft reporting rule. We use the empirical item-row distribution and $10,000$ independent redraws with replacement, retaining fixed scoring-position masks. Selected artifacts are valued by their full-table mean scores. The common-reference tables instead use permutations and complementary splits without replacement, so the two references need not coincide.

All Math/Law tables use the audited point estimates from this reanalysis. For Math/random search, M1/M2 uses the full-matrix maximum, correcting the earlier first-column calculation; the four Qwen3-8B configurations use the same split implementation as the other cells. We compute directions from unrounded scores and count exact default ties as incorrect. These corrections reconcile M1/M2 with full-observation PromptEval.

 Table~\ref{tab:matched_targets} summarizes estimator accuracy under those references. M1/M2 and SIREN have similar mean absolute errors, $0.22$ and $0.24$\,pp, and the same directional accuracy, while SIREN has the mean signed error closest to zero. The results do not show that M1/M2 estimate SIREN's target. They show that proximity to a common soft-rule reference and accuracy for a method's own target are different comparisons. Under the common reference, M1/M2, M3, M4, and SIREN have MAEs of $1.85$, $1.67$, $1.30$, and $0.24$\,pp. SIREN has $167/176$ correct directions under that reference and $165/176$ under its fresh-data target. The two changed directions are GLM-4 on Math with DSPy at $1.5\mathrm{M}$ and $6.5\mathrm{M}$, where the fresh-data reference lies just above the default and the permutation reference lies just below it.

\begin{table}[htbp]
\centering
\small
\caption{Fresh-data reanalysis of $176$ cells under each method's own reporting target. Errors are in percentage points. Bold identifies the mean signed error closest to zero, smallest mean absolute error (MAE), and best directional accuracy, including ties.}
\label{tab:matched_targets}
\begin{tabular}{lrrc}
\toprule
Method & Mean signed error & MAE & Correct direction \\
\midrule
M1/M2 & $+0.22$ & $\mathbf{0.22}$ & $\mathbf{165/176}$ \\
M3 & $-0.66$ & $1.28$ & $126/176$ \\
M4 & $-0.44$ & $0.56$ & $148/176$ \\
SIREN & $\mathbf{-0.05}$ & $0.24$ & $\mathbf{165/176}$ \\
\bottomrule
\end{tabular}
\end{table}

\FloatBarrier

\subsection{Comparison with nested cross-validation}
\label{app:nested_cv}

We compare SIREN with five-fold nested cross-validation on MMLU-Pro Law using Qwen3-8B, fixed-library random search, and $B=500\mathrm{K}$. Each fold holds out $220$ or $221$ outer-test items. Within the outer-training portion, we pre-specify $100$ development items and $200$ disjoint SIREN reporting items. We rerun the complete search-and-selection procedure under four tuner seeds per fold and freeze the returned hard artifact, shortlist, and soft-rule weights before reading the outer-test scores. All $1{,}101$ items serve as outer test once, giving $20$ complete tuner reruns with the original development size.

\begin{table}[htbp]
\centering
\small
\caption{Nested cross-validation and SIREN with matched development size. Scores are percentages, and differences are percentage points.}
\label{tab:nested_cv}
\begin{tabular}{lr}
\toprule
Quantity & Result \\
\midrule
Nested-CV hard-artifact outer-test mean & $32.38$ \\
SIREN reporting mean & $32.69$ \\
Matched outer-test soft-rule mean & $32.03$ \\
SIREN minus matched soft-rule mean & $+0.66$\,pp \\
Outer-test soft minus hard mean & $-0.35$\,pp \\
Complete tuner reruns & $20$ \\
\bottomrule
\end{tabular}
\end{table}

The matched comparison for SIREN is its reporting estimate against the outer-test value of the same soft rule. The $+0.66$\,pp difference is a reporting gap against a finite outer-test sample, rather than exact population error. The $-0.35$\,pp soft-minus-hard difference concerns the deployment rules. The point estimates are close in this configuration, but do not establish equivalence. Full nested cross-validation evaluates the complete hard tune-then-deploy procedure by rerunning tuning inside each outer fold. SIREN provides conditional inference after one completed search. M3 and M4 reuse a frozen shortlist and therefore are holdout reporting baselines, rather than full nested cross-validation.

\FloatBarrier

\subsection{Fresh development samples and tuner reruns}
\label{app:tuner_reruns}

We examine the relation between the conditional reporting target and the full tuning procedure on MMLU-Pro Law using Qwen3-8B, an $18$-configuration fixed-library random-search tuner, and $B=500\mathrm{K}$. The primary design uses five pre-specified development samples and two tuner seeds per sample ($5D\times2U$). Each run uses $100$ development items and freezes its returned hard artifact and shortlist before SIREN uses a separate $200$-item reporting pool. A further disjoint $200$-item audit pool evaluates the soft rule and the returned hard artifact.

The fixed run is the pre-specified \texttt{D00\_U00} run of this experiment. It was fixed before the repeated-tuning results were examined, is included in the repeated-tuning averages, and differs from the original open-ended-tuner run. As a post-result sensitivity check, we add two further pre-generated tuner seeds to each of the same five development samples ($5D\times4U$), without adding development, reporting, or audit items. Both designs therefore have five independent development samples. The joint bootstrap resamples development samples, then tuner seeds within each sampled development set, and uses common item resamples across runs that share reporting and audit pools.

\begin{table}[htbp]
\centering
\small
\caption{Fixed-run reporting and repeated tuning. Scores are percentages. Differences and confidence intervals are in percentage points and are computed before rounding. The fixed-soft value is measured on the audit pool.}
\label{tab:tuner_reruns}
\begin{tabular}{lcc}
\toprule
Quantity & Primary: $5D\times2U$ & Sensitivity: $5D\times4U$ \\
\midrule
Fixed-run SIREN estimate & $34.96$ & $34.96$ \\
Fixed-run soft-rule audit value & $34.73$ & $34.73$ \\
Repeated-tuning soft-rule mean & $34.40$ & $34.31$ \\
Repeated-tuning hard-artifact mean & $35.20$ & $35.08$ \\
Fixed-soft minus repeated-hard gap & $-0.47$ & $-0.34$ \\
$95\%$ joint bootstrap CI & $[-2.98,1.89]$ & $[-2.59,1.85]$ \\
\bottomrule
\end{tabular}
\end{table}

The primary fixed-run reporting gap is $34.96-34.73=+0.23$\,pp. The fixed-soft minus repeated-soft gap is $+0.33$\,pp, and the repeated soft-minus-hard gap is $-0.80$\,pp, giving a fixed-soft minus repeated-hard gap of $-0.47$\,pp. Adding the reporting gap gives a SIREN-estimate minus repeated-hard difference of $-0.24$\,pp. With four seeds, the two audit-based components are $+0.42$ and $-0.76$\,pp, giving $-0.34$\,pp. The confidence intervals are too wide to establish equivalence within the pre-specified $\pm1$\,pp margin. The experiment separates fixed-run reporting error, variation across tuner runs, and the difference between deployment rules; it does not identify the conditional soft-rule target with the original hard tune-then-deploy target.

\FloatBarrier

\subsection{Stochastic decoding and fractional scores}
\label{app:stochastic_decoding}

We freeze a pre-specified five-artifact Qwen3-8B Law shortlist and generate ten outputs for each artifact on each of $1{,}001$ items at decoding temperature $1.0$, giving $50{,}050$ generations. Among item-artifact pairs, $24.4\%$ produce more than one valid answer and $11.6\%$ change between correct and incorrect across the ten decodes. We evaluate one sampled decode per item and the average of two sampled decodes, which gives scores in $\{0,0.5,1\}$. Each setting uses a reporting target computed from $10{,}000$ independent redraws and $2{,}000$ further trials for interval coverage, with $2{,}000$ multiplier draws per trial.

\begin{table}[htbp]
\centering
\small
\caption{Bootstrap calibration under real stochastic execution. Coverage is measured over repeated trials for the reporting target induced by the frozen output table.}
\label{tab:stochastic_decoding}
\begin{tabular}{lrrr}
\toprule
Setting & Bias (pp) & $95\%$ coverage (\%) & Reported SE / empirical SD \\
\midrule
One stochastic decode & $-0.04$ & $95.05$ & $1.004$ \\
Mean of two stochastic decodes & $-0.04$ & $95.05$ & $1.007$ \\
\bottomrule
\end{tabular}
\end{table}

The multiplier intervals remain calibrated for both scoring rules. All executions belonging to one benchmark item remain together in its item-level contribution. The experiment checks the reporting target induced by the frozen real-output table; it does not rerun the tuner or establish calibration for every stochastic task.

\FloatBarrier

\subsection{Additional MMLU-Pro domains}
\label{app:additional_domains}

We apply the original eleven-model random-search framework to MMLU-Pro Philosophy and Economics at $B\in\{500\mathrm{K},1.5\mathrm{M}\}$, with $K=10$, $R=10$, scoring fraction $0.5$, and $\tau=1$. Each model-budget cell has a $100{,}000$-redraw Monte Carlo reporting reference. Table~\ref{tab:additional_domains} uses the original common-reference convention. For M1/M2, M3, and M4, differences from this soft-rule reference can include both estimation error and differences between reporting rules. MAE averages absolute cellwise differences, rather than taking the absolute value of mean signed bias.

\begin{table}[htbp]
\centering
\small
\caption{Additional domains under the common reporting reference. Signed bias and MAE are in percentage points. Bold identifies the best direction, signed bias closest to zero, and smallest MAE within each subject or aggregate.}
\label{tab:additional_domains}
\begin{tabular}{llcrr}
\toprule
Subject & Method & Correct direction & Mean signed bias & MAE \\
\midrule
Philosophy & M1/M2 & $6/22$ & $+2.93$ & $2.93$ \\
 & M3 & $13/22$ & $+1.44$ & $2.86$ \\
 & M4 & $15/22$ & $+1.31$ & $1.71$ \\
 & SIREN & $\mathbf{21/22}$ & $\mathbf{+0.09}$ & $\mathbf{0.37}$ \\
\midrule
Economics & M1/M2 & $13/22$ & $+2.89$ & $2.89$ \\
 & M3 & $13/22$ & $+2.59$ & $2.91$ \\
 & M4 & $16/22$ & $+2.20$ & $2.32$ \\
 & SIREN & $\mathbf{22/22}$ & $\mathbf{+0.15}$ & $\mathbf{0.28}$ \\
\midrule
All & M1/M2 & $19/44$ & $+2.91$ & $2.91$ \\
 & M3 & $26/44$ & $+2.02$ & $2.88$ \\
 & M4 & $31/44$ & $+1.76$ & $2.02$ \\
 & SIREN & $\mathbf{43/44}$ & $\mathbf{+0.12}$ & $\mathbf{0.32}$ \\
\bottomrule
\end{tabular}
\end{table}

The reference direction is negative in $40/44$ cells. SIREN preserves all $4/4$ positive tuned gains and identifies $39/40$ regressions, so its directional accuracy is not explained by always predicting the majority direction. All $44$ SIREN intervals contain their corresponding Monte Carlo references. We report this as interval containment across heterogeneous cells, rather than a repeated-sampling coverage estimate.

\FloatBarrier

\subsection{Code generation scored by execution}
\label{app:code_generation}

We evaluate HumanEval+ Mini through EvalPlus~\cite{liu2023evalplus}, built on the $164$-task HumanEval benchmark~\cite{chen2021evaluating}. Using Qwen3-8B, we treat $18$ decoding temperatures from $0.1$ to $1.8$ as deployable artifacts and generate two programs per task-artifact pair, giving $5{,}904$ programs scored by executable unit tests. Each task-seeded realization uses $100$ development tasks, $64$ disjoint reporting tasks, development-only top-$K=10$ retention, and a $10{,}000$-redraw reporting reference. The $164$ realizations per budget overlap in tasks and are descriptive realizations, rather than independent models or datasets.

\begin{table}[htbp]
\centering
\small
\caption{Execution-scored code generation under the common reporting reference. Signed bias and MAE are in percentage points. Bold identifies the best direction, signed bias closest to zero, and smallest MAE at each budget.}
\label{tab:code_generation}
\begin{tabular}{llcrr}
\toprule
Budget cap & Method & Correct direction & Mean signed bias & MAE \\
\midrule
$500\mathrm{K}$ & M1/M2 & $64/164$ & $+2.60$ & $2.61$ \\
 & M3 & $73/164$ & $+0.12$ & $4.25$ \\
 & M4 & $104/164$ & $-0.49$ & $1.91$ \\
 & SIREN & $\mathbf{115/164}$ & $\mathbf{-0.08}$ & $\mathbf{1.39}$ \\
\midrule
$1.5\mathrm{M}$ & M1/M2 & $71/164$ & $+2.69$ & $2.69$ \\
 & M3 & $80/164$ & $-0.20$ & $4.18$ \\
 & M4 & $108/164$ & $-0.43$ & $1.94$ \\
 & SIREN & $\mathbf{113/164}$ & $\mathbf{-0.10}$ & $\mathbf{1.37}$ \\
\bottomrule
\end{tabular}
\end{table}

SIREN is closest to the specified reporting reference under both signed bias and MAE. The distinction matters for M3 at $500\mathrm{K}$, whose signed bias is $+0.12$\,pp but whose MAE is $4.25$\,pp. As in the additional-domain comparison, the baseline differences from the common reference are not pure errors for each method's own target. The $1.5\mathrm{M}$ row uses a token cap and the full library, which is exhausted after approximately $0.89$ to $0.98$ million actual tokens. The study extends the evidence to open-ended program outputs scored by execution, while retaining a fixed artifact library.

\FloatBarrier

\end{document}

%% file: checklist.tex
\section*{NeurIPS Paper Checklist}
 
\begin{enumerate}
 
\item {\bf Claims}
    \item[] Question: Do the main claims made in the abstract and introduction accurately reflect the paper's contributions and scope?
    \item[] Answer: \answerYes{}
    \item[] Justification: The main claims made in the abstract and introduction accurately reflect the paper's contributions and scope.
 
\item {\bf Limitations}
    \item[] Question: Does the paper discuss the limitations of the work performed by the authors?
    \item[] Answer: \answerYes{}
    \item[] Justification: The assumptions for the formal theoretical results are explicitly stated in both the main text and appendix. The practical limitations are discussed in the conclusion when mentioning future work.
    
\item {\bf Theory Assumptions and Proofs}
    \item[] Question: For each theoretical result, does the paper include the full set of assumptions and a complete proof?
    \item[] Answer: \answerYes{}
    \item[] Justification: The assumptions are informally stated in the main text along with the informal version of the theorem statements, and are rigorously characterized in the appendix.
 
\item {\bf Experimental Result Reproducibility}
    \item[] Question: Does the paper provide sufficient information for reproducing the main experimental results?
    \item[] Answer: \answerYes{}
    \item[] Justification: The paper's experiment section covers sufficient details for reproducibility.
 
\item {\bf Open access to data and code}
    \item[] Question: Does the paper include the code, data, and instructions needed to reproduce the main experimental results?
    \item[] Answer: \answerYes{}
    \item[] Justification: Yes, we have uploaded it into the supplementary materials. It can also be accessed in \url{https://anonymous.4open.science/r/siren-supplementary-4B81/README.md}.
    
\item {\bf Experimental Setting/Details}
    \item[] Question: Does the paper specify all the training and evaluation details?
    \item[] Answer: \answerYes{}
    \item[] Justification: The experimental setup in Section~4 specifies: data splits (100 train / 1,251 or 1,101 eval items), model identifiers and serving parameters (vLLM, temperature~0, single RTX~4090), tuner configurations (candidate generation, budget grid $\mathcal{B}{=}\{500\text{K}, 1.5\text{M}, 3\text{M}, 6.5\text{M}\}$), retained shortlist sizes ($K{=}10$ for random search, $K{=}6$ for DSPy), SIREN hyperparameters ($R{=}10$, $\rho_{\mathrm{score}}{=}0.5$, $\tau{=}1$, uniform $\omega_r$, $2{,}000$ bootstrap draws), and PromptEval settings (Rasch IRT, cell-observation fractions $f \in \{0.05, \ldots, 1.00\}$). Full model identifiers are listed in the appendix.
 
\item {\bf Experiment Statistical Significance}
    \item[] Question: Does the paper report error bars or other measures of statistical significance?
    \item[] Answer: \answerYes{}
    \item[] Justification: SIREN produces 95\% confidence intervals by design (multiplier bootstrap), reported in all main tables. Simulation studies report Monte Carlo standard errors (SE $\approx 0.5$\,pp at $N_{\mathrm{sim}}{=}2{,}000$) and show error bands in all figures (Studies~A--C). Ground truth is approximated by $10{,}000$ Monte Carlo redraws with reported MC error. The PromptEval comparison reports bootstrap CIs where applicable.
 
\item {\bf Experiments Compute Resources}
    \item[] Question: Does the paper provide sufficient information on compute resources?
    \item[] Answer: \answerYes{}
    \item[] Justification: All 11 models are served via vLLM on a single NVIDIA RTX~4090 GPU (24\,GB VRAM) with sequential model loading. 
 
\item {\bf Code of Ethics}
    \item[] Question: Have you read the NeurIPS Code of Ethics and ensured that your research conforms to it?
    \item[] Answer: \answerYes{}
    \item[] Justification: We have read the NeurIPS Code of Ethics. Our work is a statistical methodology paper; it does not involve human subjects, does not collect personal data, and does not develop models with dual-use concerns. All benchmark data used (MMLU-Pro) is publicly available.
 
\item {\bf Broader Impacts}
    \item[] Question: Does the paper discuss potential broader societal impacts?
    \item[] Answer: \answerNA{}
    \item[] Justification: This paper is purely on statistical LLM evaluations and does not have broader societal impact.
 
\item {\bf Safeguards}
    \item[] Question: Does the paper describe safeguards for responsible release of data/models with high risk for misuse?
    \item[] Answer: \answerNA{}
    \item[] Justification: This paper proposes a statistical evaluation protocol. It does not release pretrained language models, generate synthetic content, or produce assets with dual-use risk. No safeguards beyond standard academic release are required.
 
\item {\bf Licenses}
    \item[] Question: Does the paper cite the creators of used assets and respect their licenses?
    \item[] Answer: \answerYes{}
    \item[] Justification: MMLU-Pro~\cite{wang2024mmlupro} is cited and used under its public release terms. All 11 open-weight models are cited by their original publications or model cards and are used under their respective open-source licenses (Apache 2.0, Llama Community License, etc.). DSPy~\cite{khattab2023dspy} and vLLM~\cite{kwon2023vllm} are cited. PromptEval~\cite{polo2024efficient} is cited and our reimplementation follows the published algorithm.
 
\item {\bf New Assets}
    \item[] Question: Does the paper release new assets with proper documentation?
    \item[] Answer: \answerNA{}
    \item[] Justification: We do not release new datasets or pretrained models. 

\item {\bf Crowdsourcing and Human Subjects}
    \item[] Question: Does the paper involve crowdsourcing or research with human subjects?
    \item[] Answer: \answerNA{}
    \item[] Justification: This work does not involve crowdsourcing, human annotation, or human subjects research. All experiments use automated LLM evaluation on publicly available benchmark data.
 
\item {\bf IRB Approvals}
    \item[] Question: Does the paper require IRB approval?
    \item[] Answer: \answerNA{}
    \item[] Justification: No human subjects research is conducted. IRB approval is not applicable.
 
\item {\bf Declaration of LLM usage}
    \item[] Question: Does the paper describe the usage of LLMs if they are an important component of the core methods?
    \item[] Answer: \answerYes{}
    \item[] Justification: LLMs are the objects being evaluated, not tools used for method development. The 11 open-weight LLMs used as evaluation targets are fully specified in Section~4 and the appendix. DSPy uses the target LLM itself to generate candidate instructions as part of the tuning pipeline (Section~4, ``System-A: DSPy MIPROv2''); this is a standard use of DSPy and is documented. The authors also use LLMs for polishing the draft.
 
\end{enumerate}